\documentclass[twoside,11pt]{article}

\usepackage{blindtext}
\usepackage{subcaption}
\usepackage{graphicx}
\usepackage{caption}
\usepackage{hyperref}
\usepackage[nohyperref]{jmlr2e}

\usepackage{amsmath, booktabs, comment, amssymb, comment, mathrsfs, mathtools, bbm, url, xcolor}

\usepackage{graphicx}

\newtheorem{assumption}[theorem]{Assumption}

\newcommand{\N}{\mathbb{N}}

\newcommand{\dd}{{\rm d}}

\newcommand{\veps}{\varepsilon}

\newcommand{\eps}{\veps}
\newcommand{\M}{\mathcal{M}}
\newcommand{\E}{\mathbb{E}}

\newcommand{\X}{\mathcal{X}}
\newcommand{\T}{T}

\newcommand{\XX}{\bar{\mathcal{X}}}

\newcommand{\R}{\mathbb{R}}
\newcommand{\nc}{\normalcolor}

\numberwithin{equation}{section}

\newtheorem{setting}{Setting}

\numberwithin{equation}{section}

\usepackage{lastpage}
\jmlrheading{27}{2026}{1-\pageref{LastPage}}{6/25}{8/26}{25-1466}{Chenghui Li and A. Martina Neuman}

\ShortHeadings{Consistency of augmentation graph and network approximability}{C. Li and A.M. Neuman}
\firstpageno{1}

\begin{document}

\title{Consistency of augmentation graph and network approximability in contrastive learning}

\author{\name Chenghui Li \email 
       cli539@wisc.edu\\
       \addr Department of Statistics, University of Wisconsin-Madison\\
       Wisconsin, USA
       \AND
       \name A. Martina Neuman \email 
       neumana53@univie.ac.at\\
       \addr Faculty of Mathematics, University of Vienna\\
       Vienna, Austria}

\editor{Brian Kulis}

\maketitle


\begin{abstract}
Contrastive learning leverages data augmentation to develop feature representation without relying on large labeled datasets. However, despite its empirical success, the theoretical foundations of contrastive learning remain incomplete, with many essential guarantees left unaddressed, particularly the realizability assumption concerning neural approximability of an optimal spectral contrastive loss solution.
In this work, we overcome these limitations by analyzing pointwise and spectral consistency of the augmentation graph Laplacian. 
We establish that, under specific conditions for data generation and graph connectivity, as the augmented dataset size increases, the augmentation graph Laplacian converges to a weighted Laplace-Beltrami operator on the natural data manifold. 
These consistency results ensure that the graph Laplacian spectrum effectively captures the manifold geometry. 
Consequently, they give way to a robust framework for establishing neural approximability, directly resolving the realizability assumption in a current paradigm.
\end{abstract}


\section{Introduction}

A major challenge in machine learning is the high cost of collecting large, labeled image datasets, which are essential for the success of many algorithms. To address this, data augmentation techniques have been widely utilized, as noted in \citep{mumuni2022data,shorten2021text}. In computer vision, for example, these techniques include transformations such as rotations, flips, and alterations in lighting or color. 
As a result, augmented data not only diversify the training set but also significantly improve the representational power of neural networks, enabling robust generalization from otherwise limited datasets. 

In self-supervised learning, augmentation plays a pivotal role by generating multiple views of the same image, thus providing the essential self-supervision signals that drive effective learning.
Recent advances have significantly benefited from contrastive learning, a paradigm that harnesses data augmentation to drive feature representation development without the need for extensive labeled datasets. 
In computer vision, contrastive methods aim to maximize the similarity between positive pairs, such as augmented versions of the same image, while minimizing the similarity between negative pairs, such as randomly sampled images.
In text-image representation learning, contrastive methods exploit semantically aligned text-image pairs, using an objective function to capture the underlying semantic relationships.
While empirically successful, a robust theoretical foundation for contrastive learning remains incomplete. Existing theoretical frameworks, including the original work \citep{saunshi2019theoretical}, and \citep{tosh2020contrastive, tosh2021contrastive}, often assume conditional independence between positive pairs given the class label. However, this assumption is frequently violated in practice, particularly when highly correlated positive pairs, such as image augmentations, are employed. As observed by \citep{haochen2021provable}, the two views of the same natural image exhibit strong correlations that cannot be fully mitigated through conditioning on the label. As a result, these frameworks tend to fall short of explaining the practical success of contrastive learning.

\citep{haochen2021provable} proposed a framework for {(spectral)} contrastive learning that does not rely on the conditional independence of positive pairs. They introduced the notion of an \emph{augmentation graph}, in which edges between {augmented data points} are weighted according to the {likelihood} that {those points} originate from the same underlying natural data point (e.g., through Gaussian noise augmentation). {For an illustration, see Figure~\ref{fig:illustrationAugmenetation} below. The central insight of their framework is that the leading eigenvectors of this graph encode its latent geometric and semantic structure. Consequently, the goal is to learn neural representations that approximate a spectral embedding of the augmentation graph. To achieve this, \citep{haochen2021provable} introduced a spectral contrastive loss whose population minimizers correspond to a spectral embedding of the augmentation graph (see, e.g. \eqref{eq:optimization}). In this way, contrastive learning is reformulated as (1) performing spectral embedding on the population-level augmentation graph, and the resulting embedding task can then be expressed directly as (2) an optimization problem over neural network representations. Both are equivalent interpretations of the same learning objective.}

\begin{figure}[ht!]
    
    \centering
    \includegraphics[width=\linewidth]{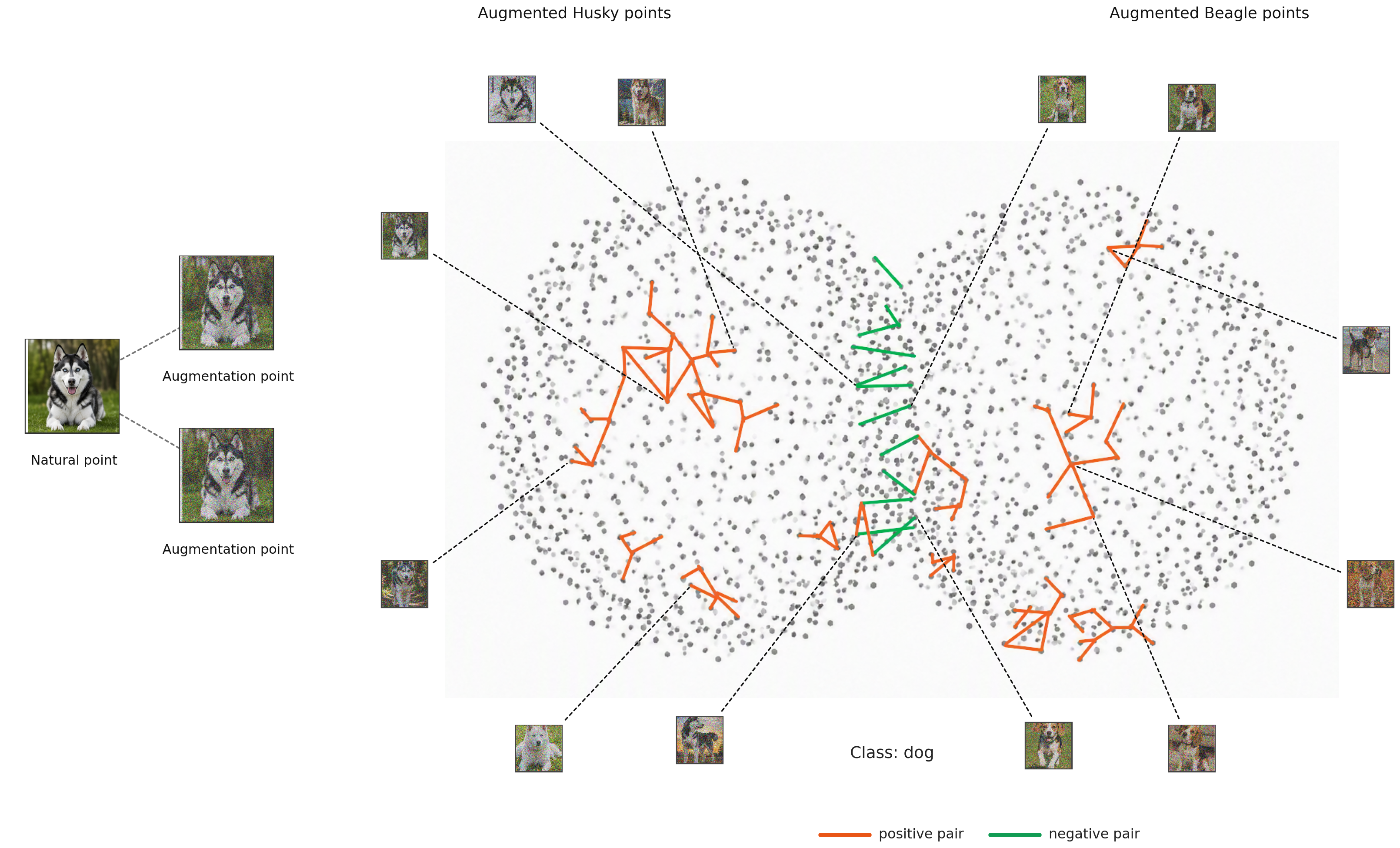}
    \caption{Illustration of the augmentation graph constructed from augmented data points generated by applying Gaussian perturbations to natural data points, shown here using a conceptual dog-image manifold containing two dog breeds. \textbf{The left cloud} contains augmentation points corresponding to Huskies. 
    \textbf{The right cloud} contains augmentation points corresponding to Beagles.
    {\color{orange} Orange edges} represent strongly connected positive pairs arising from nearby augmentations of similar dog breeds (between two Huskies or two Beagles). {\color{green} Green edges} represent weakly connected negative pairs between augmentation points corresponding to different dog breeds (between a Husky and a Beagle). The figure is intended only as a conceptual illustration of the augmentation graph and its connectivity structure. 
    }
\label{fig:illustrationAugmenetation}
\end{figure} 

Despite its utility in various applications, the theoretical framework introduced by \citep{haochen2021provable} relies on several assumptions {that may be restrictive.}
First, it implicitly assumes (e.g., \citep[Example 3.8]{haochen2021provable}) that {different classes or semantic categories are supported on disconnected geometric structures. For example, images of dogs are assumed to lie on one manifold, while images of cats lie on another. We refer to this as the \textit{separability assumption}. Such an assumption} underlies the derivation of downstream generalization error bounds within their spectral contrastive framework. However, this perspective can be limiting in practice. {For example, visually similar dog breeds such as Huskies and Alaskan Malamutes may lie on the same overarching dog manifold. In such situations, the distinction between these subclasses arises from finer geometric features, such as class boundaries, rather than from manifold-level separation.}

{Second, the framework requires a \textit{recoverability assumption} \citep[Assumption 3.6]{haochen2021provable}. Informally, this means that the spectral embedding of the augmentation graph reliably reveals the underlying clustering structure of the data, so that distinct clusters can be identified from the leading eigenvectors of the graph. Yet distinct clusters may naturally share a common manifold, which highlights the tension between the recoverability assumption and the (strict) separability assumption discussed above.}

{Finally, the framework relies on a} \textit{realizability assumption} {\citep[Assumption~3.7]{haochen2021provable}}, which can be paraphrased as
\begin{center}
\textit{``There exists a deep neural network capable of achieving a global minimum of the spectral contrastive loss on population data.''}
\end{center}
{In other words, the solution to the resulting spectral embedding is assumed to be representable by the chosen neural network architecture. However, establishing such a result implicitly requires knowledge of the regularity of the target solution function. Since this regularity has not been rigorously established, it remains difficult to determine the network capacity necessary to realize the target embedding,} thus limiting the scope and strength of the theoretical guarantees.

To address these limitations, we propose a detailed study of pointwise and spectral consistency of the augmentation graph Laplacian.
We show that the augmentation graph Laplacian can be leveraged to unveil the geometry of the underlying natural data manifold. 
Specifically, within a range of data generation and graph connectivity parameters (see e.g. \eqref{def:condprob}, \eqref{def:augmentedLap} below respectively), the graph Laplacian converges to a continuum (weighted) Laplace-Beltrami operator on the manifold---i.e. the natural generalization of the Euclidean Laplacian, see \eqref{Laplace-Beltrami}---thereby revealing its underlying structure. 
We give a meaning to this ``passage to the continuum'' next. 
We assume that the original natural dataset consists of i.i.d. $N$ samples drawn from an $m$-dimensional smooth, connected, compact manifold embedded in $\mathbb{R}^d$. 
From these natural data points, an augmented dataset of size $n\geq N$ is generated. Each pair of augmented data points is connected by an edge whose weight is determined by their average shared similarity, up to a graph connectivity parameter $\eps$. 
As $n$ tends to infinity and $\eps$ approaches zero sufficiently slowly (compared to $n$), the augmentation graph Laplacian, when applied to a smooth test function within a sufficiently small neighborhood of the manifold, converges with high probability to the continuum operator acting on the test function evaluated over the projections of the augmented data. 
Then, we demonstrate spectral consistency by showing that the graph Laplacian eigenvectors converge to the corresponding weighted Laplace-Beltrami eigenfunctions, measured in an $L^2$-type distance.

The eigenfunctions of a weighted Laplace-Beltrami operator are known to be regular (see, e.g., \citep[Section~6.3]{gilbarg1977elliptic}). Consequently, an immediate implication of our consistency results is the effective relaxation of the recoverability assumption. This is because weak (strong) energetic connections between distinct (identical) augmented subgroups can be mapped to observable variations in the eigenfunctions. 
More significantly, our results also facilitate a relaxation of the realizability assumption. The regularity of the eigenfunctions ensures their approximation by a neural network with bounded complexity, and their proximity to the graph Laplacian eigenvectors further enables the same network to approximate these eigenvectors with high accuracy. 
This offers a concrete estimate of the neural network capacity necessary to replicate an optimal spectral contrastive loss solution and thus concludes the validity of the realizability assumption. 
Finally, our analysis characterizes a precise regime of data generation and graph connectivity parameters in which the spectral proximity necessary for neural approximability holds with high probability.
By relating the graph Laplacian to the Laplace-Beltrami operator, this specification also pinpoints the conditions that enable the graph structure to faithfully capture the geometry of the underlying data manifold.
Such detailed guarantees are absent in \citep{haochen2021provable}, as we illustrate empirically in Section \ref{sec:sim}.

We now outline our novel theoretical contributions, which are twofold. The primary contribution lies in the field of manifold learning, while the secondary, an application of the first, pertains to approximation theory.

\paragraph{Manifold learning.} In one aspect, our work falls within the domain of manifold learning, where the underlying geometry of the graphical dataset, and its subsequent impact on the performance of machine learning algorithms, are explored through the consistency analysis of the corresponding graph Laplacians in the non-asymptotic, large-sample regime.
However, a central assumption in the classical setting is that graphical data are i.i.d. samples drawn from a distribution $\mu$ supported on a low-dimensional manifold (see a wealth of works such as \citep{Dmitri2014Graph, garcia2020error, calder2022improved, Dunson2021Spectral, Wormell2021Spectral, trillos2023large, calder2022lipschitz}). 
{In \citep{trillos2025minimax}, the authors established minimax rates for approximating eigenfunctions of the weighted Laplace-Beltrami operator under this data-generation model, and demonstrated that graph-Laplacian-based methods achieve nearly minimax-optimal convergence rates. 
In \citep{chaubet2025minimax}, under the additional assumption that the manifold is known, the authors considered a weaker metric for measuring the approximation of PDE eigenfunctions and obtained corresponding minimax rates.} 
In contrast, augmented graphical data points may deviate from the natural data manifold (e.g. through a Gaussian perturbation), preventing the direct application of traditional manifold learning techniques to analyze the augmentation graph Laplacian.
Moreover, the augmented data points may exhibit dependencies. 
Therefore, ensuring a ratio of augmented to natural data points ($N\leq n\leq CN$)  that guarantees sufficient independence for the desired concentration effects is crucial to the validity of our analysis.
Additionally, existing research on graph Laplacian consistency has predominantly focused on specific graph constructions, often relying on \textit{primary similarity measures} (e.g., the $\eps$-proximity graph).
Our augmentation graph, however, is based on the theoretical principles in \citep{haochen2021provable}, which employs a concept akin to a \textit{secondary similarity measure} (see, e.g., the notion of shared nearest neighbors graph in \citep{ayad2003refined}).
Namely, given a natural data point $x$, we assume that the augmented data points $\bar{x}$ follow a distribution $\mathbb{P}(\bar{x}|x)$.
Two augmented data points are considered similar if both are likely, on average, to be generated from the same natural data point. Accordingly, the weight between two augmented data points is determined by $\mathbb{E}_{x\sim \mu} \big[\mathbb{P}(\bar{x}|x)\mathbb{P}(\bar{x}'|x)\big]$, which reflects a second-order (secondary) similarity derived from the first-order similarities $\mathbb{P}(\bar{x}|x)$ and $\mathbb{P}(\bar{x}'|x)$. 
Thus, in this regard, our work makes a significant theoretical contribution to manifold learning by introducing novel methods to address data augmentation graphs. 
These methods, to the best of our knowledge, have not been previously presented.
Our consistency results achieve the most optimal convergence rates available in the current literature.
Explicitly, if the graph connectivity parameter $\eps$ scales like\footnote{The constant $C$ in \eqref{lengthscale} may be chosen to depend only on the geometry of the natural data manifold and the natural data sampling density; see Theorems~\ref{thm:pointwise},~\ref{thm:spectralconvergence}.}
\begin{equation} \label{lengthscale}
    C\Big(\frac{\log n}{n}\Big)^{\frac{1}{m+4}} \leq \eps\ll 1, 
\end{equation}
then with high probability, both the pointwise and spectral convergence rates scale linearly in the graph connectivity parameter $\eps$.
Moreover, if $\eps = C\big(\frac{\log n}{n}\big)^{\frac{1}{m+4}}$, then we obtain the convergence rate of $\mathcal{O}(n^{-\frac{1}{m+4}})$, up to a log factor. Since $N \leq n \leq CN$, all these estimates can be expressed in terms of the natural dataset size $N$ in place of the augmented dataset size $n$.

\paragraph{Approximation theory.} 
Another, more application-oriented contribution of our work is a network approximation result for the target classifier function. 
This connects to the emerging approximation theory of spectral neural networks (see \citep{li2026spectral}), which seeks to apply manifold learning to graph analysis.
Building on the framework in \citep{haochen2021provable}, contrastive learning can be interpreted as a parametric variant of the spectral embedding algorithm applied to augmented graphical datasets.
Consequently, the eigenvectors of the augmentation graph Laplacian encapsulate the spectral embedding information of a spectral contrastive loss solution, which corresponds to the target classifier function. 
This perspective enables us to deduce neural approximability of the classifier by leveraging the spectral convergence of the augmentation graph Laplacian to the Laplace-Beltrami operator on the natural data manifold and the regularity of the Laplace-Beltrami eigenfunctions.
Specifically, the regularity of the eigenfunctions allows them to be approximated by a neural network with bounded complexity, and their proximity to the corresponding graph Laplacian eigenvectors further guarantees that the neural network can similarly approximate the latter with high fidelity.
We demonstrate that, in the scale \eqref{lengthscale}, with high probability (\textit{only} due to the randomness of data) the required network architecture has a bounded complexity outlined in \eqref{networkcomplex}, with a depth of $\mathcal{O}(\log (\eps^{-1}) + \log d)$ and a width of $\mathcal{O}(\eps^{-m}+d)$.
Thus, if $\eps = C\big(\frac{\log n}{n}\big)^{\frac{1}{m+4}}$, then the network depth scales at most logarithmically with the data size $n$ (or $N$), while the network width scales sublinearly.
This result is different from standard approximation theory (e.g., \citep{chen2022nonparametric}) and memorization frameworks (e.g., \citep{yun2019small}), as it requires a network to replicate the eigenvectors while also generalizing to the true eigenfunctions of the (weighted) Laplace-Beltrami operator.


\subsection{Simulation} \label{sec:sim}

\begin{figure}
    \centering
		\begin{subfigure}[t]{0.45\textwidth}
		\centering
        \includegraphics[width=0.86\linewidth]{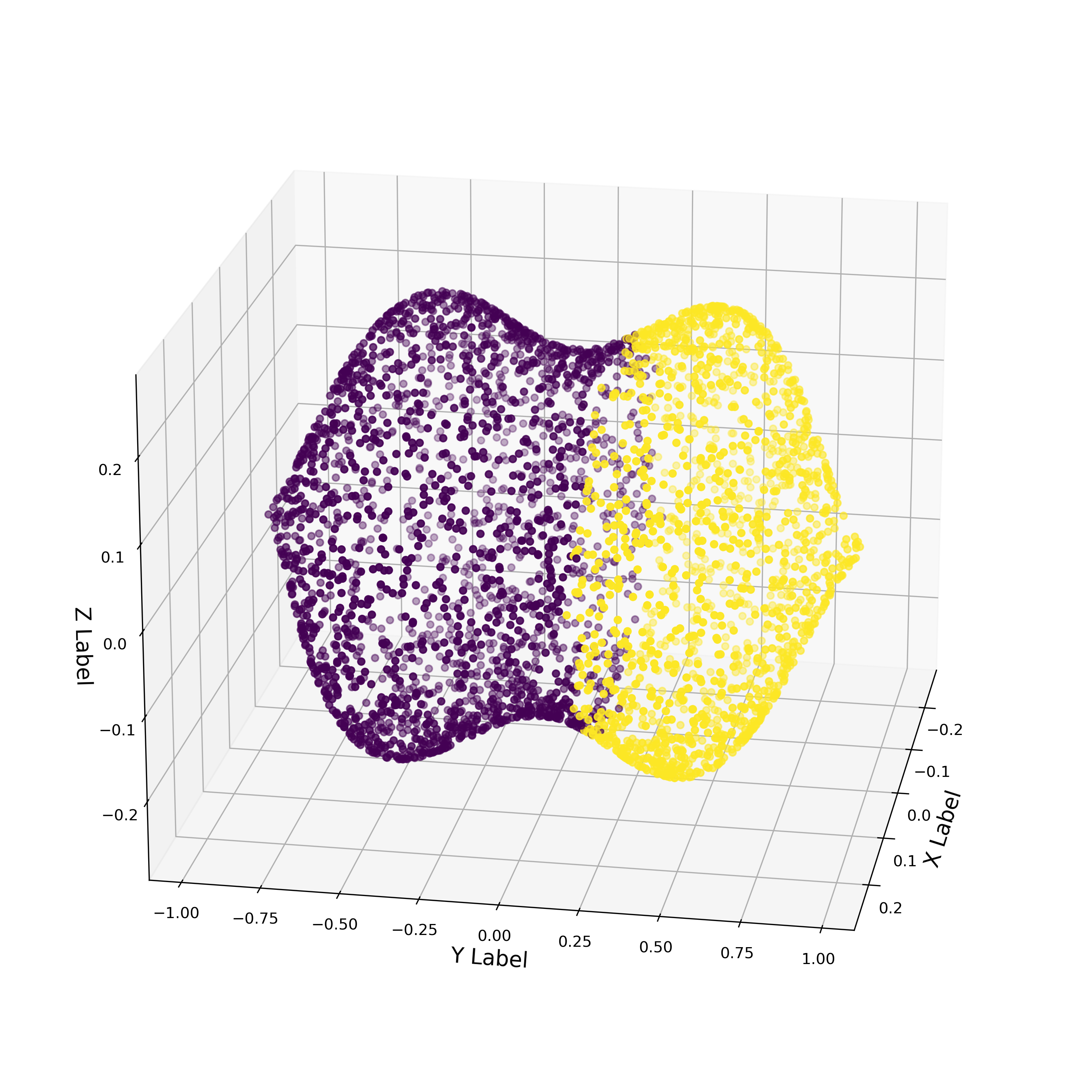}
		\caption{When choosing $\eps_{\rm p},\eps_{\rm w},\eps_{\rm n}$ according to Table \ref{tab:symbols} with $\tau=2$, $\eta=1$.}
		\label{fig:eps_p3}
		\end{subfigure}
		\begin{subfigure}[t]{0.45\textwidth}
		\centering
		\includegraphics[width=0.9\linewidth]{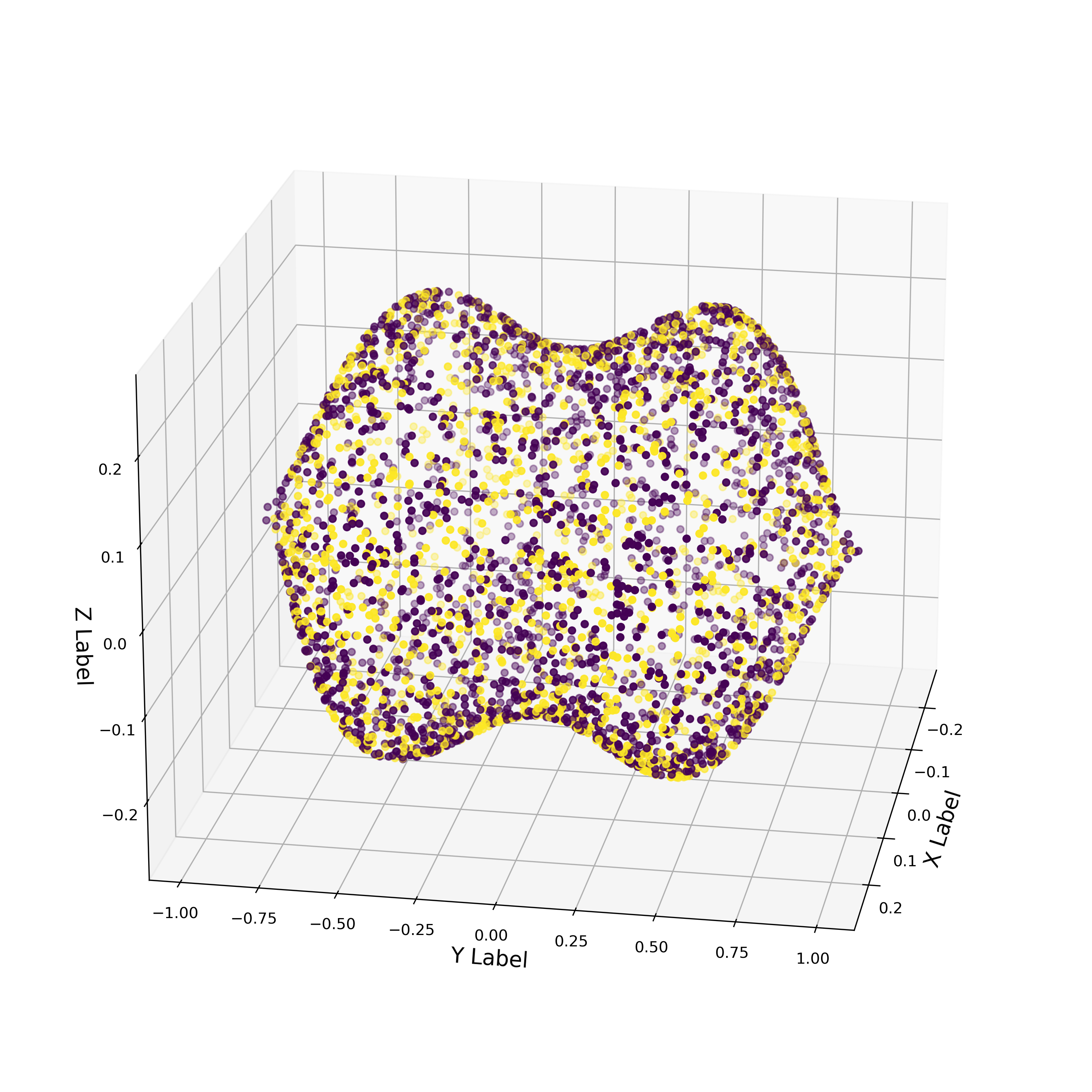}
		\caption{When choosing $\eps_{\rm p}=\eps_{\rm w}=\eps$.}
	\label{fig:eps_p1}
		\end{subfigure}
\caption{(a) When choosing $\eps_{\rm p},\eps_{\rm w},\eps_{\rm n}$ according to Table \ref{tab:symbols} with $\tau=2$, $\eta=1$, the first non-trivial eigenvectors of the augmentation graph Laplacian detects the bottleneck. (b) When choosing $\eps_{\rm p}=\eps_{\rm w}=\eps$, the first non-trivial eigenvectors of the augmentation graph Laplacian cannot detect the bottleneck.}
\end{figure}

\noindent In this simulation, the natural data $x_i$ are uniformly sampled from the dumbbell.
We expect the first nontrivial eigenvectors of the augmentation graph Laplacian to reveal the bottleneck, enabling geometric recovery via $k$-means clustering when projected onto the dumbbell. This expectation follows from the fact that the first nontrivial eigenfunctions of the limit Laplace-Beltrami operator of a dumbbell capture its shape (see \eqref{Laplace-Beltrami} and Remark~\ref{rem:modification} afterward).
Specifically, we first generate $4000$ samples i.i.d. uniformly from a $2$-dimensional dumbbell.
Our generation scheme for augmentation graph requires three parameters: a graph connectivity parameter $\eps= n^{-\frac{1}{2+4}}$, and different choices of, an edge weight parameter  $\eps_{\rm w}$ and a data generation parameter $\eps_{\rm p}$. A detailed introduction of the latter parameters is given in Section~\ref{sec:prelim}; see also as summary in Table~\ref{tab:symbols}.
We note that the connectivity parameter $\eps=n^{-\frac{1}{2+4}}$ corresponds to the optimal rate for the manifold learning problem discovered in \citep{calder2022improved}{---}which is $\eps=n^{-\frac{1}{m+4}}$ with $m=2${---}up to a logarithmic factor. 
When $\eps_{\rm p}=\eps_{\rm w}=\eps$, the setting reduces to that in \citep{haochen2021provable}, causing the manifold's geometry to become ambiguous; see Figure \ref{fig:eps_p1}. In contrast, selecting the parameters as suggested in Table~\ref{tab:symbols}, the manifold geometry is recovered; see Figure \ref{fig:eps_p3}.

\subsection{Outline}

Our paper is organized as follows.
In Section~\ref{sec:prelim}, we provide an overview of data augmentation (Section~\ref{sec:aug}) as well as essential preliminaries, conventions, and assumptions (Section~\ref{sec:assumption}) that will be consistently employed throughout the paper. Section~\ref{sec:setup} is dedicated to describing the framework for each of the main results. Section~\ref{sec:pointwiseconsistency} focuses on establishing prerequisites for pointwise convergence, and Section~\ref{sec:spectralconsistency} on the framework for spectral convergence. In Section~\ref{sec:contrastivelearn}, we provide the setup for spectral contrastive learning. The main results are formally stated in Section~\ref{sec:results}. 
Section~\ref{sec:numerics} presents numerical experiments to validate the choice of parameters that work for contrastive learning.
Section~\ref{sec:pointwisewhole} is devoted to the proof of Theorem~\ref{thm:pointwise}, with the main proof given in Section~\ref{sec:pointwiseproof}. Section~\ref{sec:spectralwhole} is devoted to the proof of Theorem~\ref{thm:spectralconvergence}, with the main proof given in Section~\ref{sec:spectralproof}. Finally, Section~\ref{sec:contrastiveproof} presents the proof of Theorem~\ref{thm:contrastive}, a consequence of Theorem~\ref{thm:spectralconvergence}.
The Appendix contains all supplementary results, including relevant concentration results and helpful geometry, as well as detailed calculations to support the main text.

\section{Preliminary} \label{sec:prelim}

\subsection{Augmented data and augmentation graph Laplacian} \label{sec:aug}

Let $\mathcal{M}\subset\mathbb{R}^{d}$ be an $m$-dimensional smooth, closed (compact), connected, orientable manifold without boundary. Let $\mu$ be a distribution on $\mathcal{M}$, and let $x_1,\dots,x_{N}\sim\mu$ be i.i.d. samples, referred to as natural data points. 
Let $\X:=\{x_1,\dots,x_{N}\}$. The specific ordering $x_{i}$ is immaterial; thus we refer to a sample as $x\in\X$. 
From $\X$, we generate an augmented dataset $\XX$, where each element $\bar{x}\in\XX$ is an augmented datum. 
For each $x\in\M$, let $\mathbb{P}(\cdot|x)$ denote the conditional distribution that $\bar{x}\in\XX$ follows. We specify that $\mathbb{P}(\cdot|x)$ takes the form
\begin{equation} \label{def:condprob}
    \mathbb{P}(\bar{x}|x) = \pi_{\eps_{\rm p}}(\bar{x}|x):=\frac{1}{(\eps_{\rm p}\sqrt{2\pi})^{d}}\exp\Big(-\frac{\|x-\bar{x}\|^2}{2\eps_{\rm p}^2}\Big),
\end{equation}
for some $\eps_{\rm p}>0$ to be chosen later. 
Here, $\|\cdot\|$ denotes the usual Euclidean vector norm in $\mathbb{R}^{d}$.
Intuitively, $\mathbb{P}(\bar{x}|x)$ gives the likelihood of $\bar{x}$ being generated from $x\in\X$. 
Note that, 
\begin{equation*}
    \int_{\mathbb{R}^{d}} \mathbb{P}(y|x)\dd y = \int_{\mathbb{R}^{d}} \pi_{\eps_{\rm p}}(y|x) \dd y = 1.
\end{equation*}
Therefore, we can write, $\XX := \{\bar{x}: \bar{x} \text{ is sampled from the density } \pi_{\eps_p}(\cdot|x) \text{ for some } x\in\X\}$.
The law \eqref{def:condprob} indicates that the majority of the generated data is not far from $\M$. 
Precisely, there exists $0<r<\infty$ such that the event
\begin{equation} \label{def:dist}
    {\rm dist}(\bar{x},\mathcal{M}) < r
\end{equation}
holds with a high probability, where ${\rm dist}(\bar{x},\mathcal{M}):=\inf\{\|\bar{x}-z\|: z\in\mathcal{M}\}$.
This essential fact is the content of Lemma~\ref{lem:probcompact} below.

Let $\#\XX = n\geq N$. 
Note that $\XX$ itself is not an independent set of random variables but only partially so. 
To ensure sufficient independence from the augmented data cloud for our analysis, we impose the following assumption, which is reasonably aligned with practical settings.

\begin{assumption} \label{assum:partindependent}
Let $C_0\in\mathbb{N}$ be predetermined. We assume either of the following holds:
\begin{enumerate}
    \item for each augmented data point $\bar{x}$, there exist at most $C_0$ other points $\bar{x}'$ dependent on $\bar{x}$;
    \item for any natural data point $x$, at most $C_0$ augmented data points $\bar{x}$ are generated from $x$.
\end{enumerate}
\end{assumption}

Note that point (2) above, which is easier to implement in practice, implies point (1).
Under Assumption~\ref{assum:partindependent}, $N\leq n\leq C_0 N$\footnote{The condition $n \geq N$ reflects practical considerations, while $n \leq C_0 N$ ensures sufficient independence in the augmented data cloud.}. We construct an augmentation graph $G_{\XX}$ from $\XX$ as follows. 
For $\bar{x}\not=\bar{x}'\in\XX$, let
\begin{equation} \label{def:edgeweight}
    \omega(\bar{x},\bar{x}') := (\eps_{\rm w}\sqrt{2\pi})^{2d} \mathbb{E}_{x\sim \mu} \Big[\pi_{\eps_{\rm w}}(\bar{x}|x)\pi_{\eps_{\rm w}}(\bar{x}'|x)\Big],
\end{equation}
for some $\eps_{\rm w}$ related to $\eps_{\rm p}$ in \eqref{def:condprob}, to be specified shortly. 
Let $\eps>0$.
With $\omega(\bar{x}, \bar{x}') \in (0,1)$ representing the connection strength between $\bar{x}$ and $\bar{x}'$, we define the augmentation graph $G_{\XX}$ to consist of a vertex set in $\XX$ and the edge set
\begin{equation*}
    \{(\bar{x},\bar{x}')\in\XX\times\XX: 0<\|\bar{x}-\bar{x}'\|\leq\eps\}.
\end{equation*}
To the graph $G_{\XX}$, we associate an \textit{unnormalized} graph Laplacian operator $\mathcal{L}_{\rm aug}$ as follows\footnote{The subscript ``aug" stands for ``augmentation."}. 
For $f: \XX\to\R$, we let
\begin{equation} \label{def:augmentedLap}
    \mathcal{L}_{\rm aug} f(\bar{x}) := \frac{1}{n\eps_{\rm w}^m\eps^{m+2}}\sum_{\bar{x}'\in\XX: 0<\|\bar{x}'-\bar{x}\|\leq\eps} \omega(\bar{x},\bar{x}')(f(\bar{x})-f({\bar{x}'})).
\end{equation}
{For the derivation of the associated graph-theoretic formulation of the graph Laplacian, see Section~\ref{sec:contrastivelearn}.} 
It is evident from \eqref{def:augmentedLap} that $\eps$ acts as a graph connectivity parameter. 
We now dictate the relationship among the parameters $\eps_{\rm p}, \eps_{\rm w}$ and $\eps$.
Let $\tau>0$, $\eta>0$. Set
\begin{equation*} 
    \eps_{\rm w} := \eps^{\tau} \quad\text{ and }\quad
    \eps_{\rm p} := \eps^{\tau+1}\eta^{1/d}.
\end{equation*}
The choice of $\eps_{\rm p}\ll \eps_{\rm w}$ is made to ensure consistency for $\mathcal{L}_{\rm aug}$, as we shall see. 
In addition, it directly supports the validity of the claim \eqref{def:dist}, whose proof {occupies the remainder of this section.}

\begin{lemma} \label{lem:probcompact}
Let $\eps>0$ and $0<r<\infty$. 
Then with probability at least $1-2dn\exp(-r^2/(2d\eps_{\rm p}^2))$, ${\rm dist}(\bar{x},\mathcal{M})< r$, for all $\bar{x}\in\XX$. In particular, letting $\eps_{\rm n} := \eps^{\tau+1}$, 
the event
\begin{equation} \label{90dist}
    {\rm dist}(\bar{x},\mathcal{M}) < \eps_{\rm n},
    \quad \forall\,\bar{x}\in\XX,
\end{equation}
with probability at least $1-2dn\exp(-\eta^{-2/d}/(2d))$.
\end{lemma}
    
\begin{proof}[Proof of Lemma~\ref{lem:probcompact}]
Conditioning on $x\in\M$, we let $\bar{x}\sim \pi_{\eps_{\rm p}}(\cdot|x)$. 
Then it follows from \eqref{def:condprob} that $(\bar{x})_{l}\sim\mathcal{N}((x)_{l},\eps_{\rm p}^2)$ i.i.d, for every $l=1,\dots,d$. 
By Hoeffding's inequality \citep{hoeffding1994probability},
\begin{equation} \label{hoeffding}
    \mathbb{P} \big(|(\bar{x})_{l} - (x)_{l}|^2\geq t^2\big)
    =\mathbb{P} \big(|(\bar{x})_{l} - (x)_{l}|\geq t \big)
    \leq 2\exp\Big(-\frac{t^2}{2\eps_{\rm p}^2}\Big), 
\end{equation}
for every $t>0$. If $r^2\leq \|\bar{x}-x\|^2=\sum_{l=1}^{d} |(\bar{x})_{l} - (x)_{l}|^2$, then there must exist one $l$ such that \eqref{hoeffding} holds with $t=r/\sqrt{d}$. Consequently,
\begin{equation*}
    \mathbb{P} \big(\|\bar{x}-x\|^2\geq r^2 \big)
    = \mathbb{P} \big(\|\bar{x}-x\|\geq r \big)
    \leq 2d\exp\Big(-\frac{r^2}{2d\eps_{\rm p}^2}\Big).
\end{equation*}
Given that $\|\bar{x}-x\|\geq {\rm dist}(\bar{x},\M)$, and the continuity of the distance functions, we conclude
\begin{equation*}
    \mathbb{P} \big({\rm dist}(\bar{x},\mathcal{M})\geq r \big)\leq 2d\exp\Big(-\frac{r^2}{2d\eps_{\rm p}^2}\Big).
\end{equation*}
By performing a union bound over $n$ points of $\bar{x}\in\XX$, we obtain uniformly for all such $\bar{x}$, ${\rm dist}(\bar{x},\mathcal{M}) < r$ with probability at least $1-2dn\exp(-r^2/(2d\eps_{\rm p}^2))$. 
Then \eqref{90dist} automatically follows from the selection $r = \eps_{\rm n} = \eps^{\tau+1}$. 
\end{proof}

{\begin{remark}
The edge weight $\omega(\bar{x},\bar{x}') = (\varepsilon_{\rm w}\sqrt{2\pi})^{2d} \mathbb{E}_{x\sim \mu} \big[\pi_{\varepsilon_{\rm w}}(\bar{x}|x)\pi_{\varepsilon_{\rm w}}(\bar{x}'|x)\big]$, 
is an idealized population-level quantity modeling the average likelihood that the augmented data points $\bar{x},\bar{x}'$ arise from the same underlying natural data point. 
In practice, since the sampling distribution $\mu$ is unknown, the expectation may be approximated empirically using the observed i.i.d. natural data points $X_1,\dots,X_N\sim\mu$, namely through the Monte Carlo approximation
\begin{equation*}
    \frac{1}{N} \sum_{i=1}^N \pi_{\varepsilon_{\rm w}}(\bar{x}|X_i)\pi_{\varepsilon_{\rm w}}(\bar{x}'|X_i).
\end{equation*}
Because $\pi_{\varepsilon_{\rm w}}$ is a Gaussian density with variance $\varepsilon_{\rm w}^2$, the dominant contributions to such averages arise from natural data points lying simultaneously near \emph{both} $\bar{x}$, $\bar{x}'$, within distances on the order of $\varepsilon_{\rm w}$.
Thus, the construction may be viewed heuristically as measuring the degree of shared local natural data between the two augmented data points, in a manner reminiscent of shared nearest neighbor constructions, see e.g. \citep{neuman2023graph}.
From a statistical perspective, there are therefore two levels of approximation. First, standard concentration results for empirical averages of i.i.d. random variables (e.g. Hoeffding-type inequalities \citep[Theorem~2.8]{BLM2013}, \citep[Theorem~2.6.2]{Vershynin2018}) justify replacing the population expectation by its empirical Monte Carlo approximation when $N$ is sufficiently large. Second, known quantitative Gaussian tail estimates yield sharp exponential decay of contributions from sufficiently distant natural data points and may be used to control the additional approximation error from Gaussian localization.
A detailed quantitative analysis of such estimation errors, however, lies beyond the scope of the present theoretical work.
\end{remark}}

{\begin{remark}
    While in this paper we specifically focus on the Gaussian augmentation graph, the framework is intended to model local perturbative augmentations more broadly. 
    In particular, sufficiently small transformations, including translations $x\mapsto x+v$ with $\|v\|\ll 1$ and rotations $x \mapsto Rx$ with $R \approx I$, may be viewed as local perturbations of the underlying data cloud of small magnitude in the ambient space.
    From the viewpoint of our analysis, isotropic Gaussian perturbations may therefore be interpreted as a mathematically tractable generic local perturbation model probing the local geometry of the manifold in \textit{all directions} simultaneously. 
    Thus, isotropic noise augmentation does not primarily enforce invariance to a particular global transformation, but rather promotes local geometric stability around the data manifold. In this sense, the role of isotropic augmentation is less about imposing a specific symmetry and more about enriching the data cloud without substantially distorting the intrinsic geometry relevant to downstream tasks. 
\end{remark}}

\subsection{Basic assumptions and convention} \label{sec:assumption}

\paragraph{Basic notation.} We write $\mathbb{R}_{\geq 0}$ and $\mathbb{R}_+$ to denote the sets of nonnegative and positive real numbers, respectively.
We write $0$ to denote either the scalar zero, the zero vector, or the zero function, with the distinction made clear from context.
For a vector $x$, we write $x^i$ to refer to its $i$th component. 
We use $\|\cdot\|$ to represent either the Euclidean vector norm (with the dimension also inferred from context) or the Frobenius matrix norm, and $|\cdot|$ to represent the absolute value of a scalar.
Additionally, we use $\|\cdot\|$ with a subscript to specify a functional norm, such as $\|\cdot\|_{L^{\infty}(\M)}$. 
In the case of a vector or a matrix, we also use $\|\cdot\|_0$ to specify the \textit{nuclear norm}, counting the number of nonzero entries, and $\|\cdot\|_{\infty}$ to specify the $\ell^{\infty}$-norm.
These last two notations are used exclusively in Section~\ref{sec:contrastivelearn}, the statement of Theorem~\ref{thm:contrastive}, and its proof in Section~\ref{sec:contrastiveproof}.

\paragraph{Basic manifold geometry assumptions.} We endow $\M$ the Riemannian structure induced by the ambient space $\mathbb{R}^d$ and let $\dd\mathcal{V}$ be the volume form of $\M$ with respect to the induced metric tensor.
We assume that all the sectional curvatures of the closed manifold $\mathcal{M}$ are bounded above in absolute value by some $K<\infty$, the manifold reach $R<\infty$, and further, the manifold injectivity radius is bounded below by some $i_0>0$. We regard $K, R, i_0$ to be \textit{intrinsic}\footnote{unrelated to the notions of ``intrinsic'' and ``extrinsic'' in differential geometry} values of the manifold learning problem.

We use $B^m(0,r)$ and $B^d(0,r)$ to denote an $m$-dimensional and $d$-dimensional centered Euclidean ball, respectively,
and $\mathcal{B}(x,r)$ to denote a geodesic ball centered at $x$ on $\mathcal{M}$, all with radius $r$. We write $d(x,y)$ to indicate the geodesic distance on $\mathcal{M}$.
Note that, from \citep[Proposition~2]{garcia2020error},
\begin{equation} \label{eq:distancecompare}
    \|x-y\|\leq d(x,y)\leq \|x-y\| + \frac{8}{R^2}\|x-y\|^3,
\end{equation}
whenever $\|x-y\|\leq R/2$, with the first inequality holding universally.

At $x\in\M$, the affine \textit{tangent plane} to $\M$ is a translated isomorphic copy of $\mathbb{R}^m$. 
We distinguish between such physical plane $\tilde{T}_x\M$, which sits in $\mathbb{R}^d$, and its isomorphic copy, \textit{the tangent space} $T_x\M$, which can be embedded in $\mathbb{R}^m$.
Similarly, we distinguish between the affine \textit{normal plane} $\tilde{N}_x\M$ of $\M$ at $x$ in $\mathbb{R}^d$ and its translated isomorphic copy, \textit{the normal space} $N_x\M\cong\mathbb{R}^{d-m}$. See Figure~\ref{fig:manifoldandtangent} below for an illustration.

\begin{figure}[ht!]
    \begin{subfigure}[b]{0.60\textwidth}
        \centering
        \includegraphics[width=0.7\linewidth]{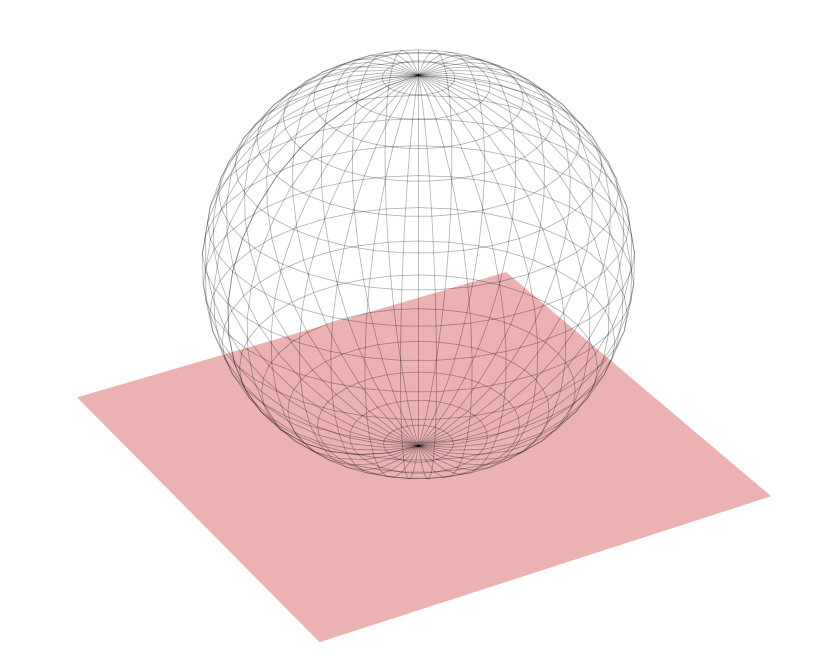}
        \caption{A two-dimensional sphere in $\mathbb{R}^3$ and one of its tangent planes colored in pink}
    \end{subfigure}
    \quad \quad \quad
    \begin{subfigure}[b]{0.35\textwidth}
        \centering
        \includegraphics[width=0.7\linewidth]{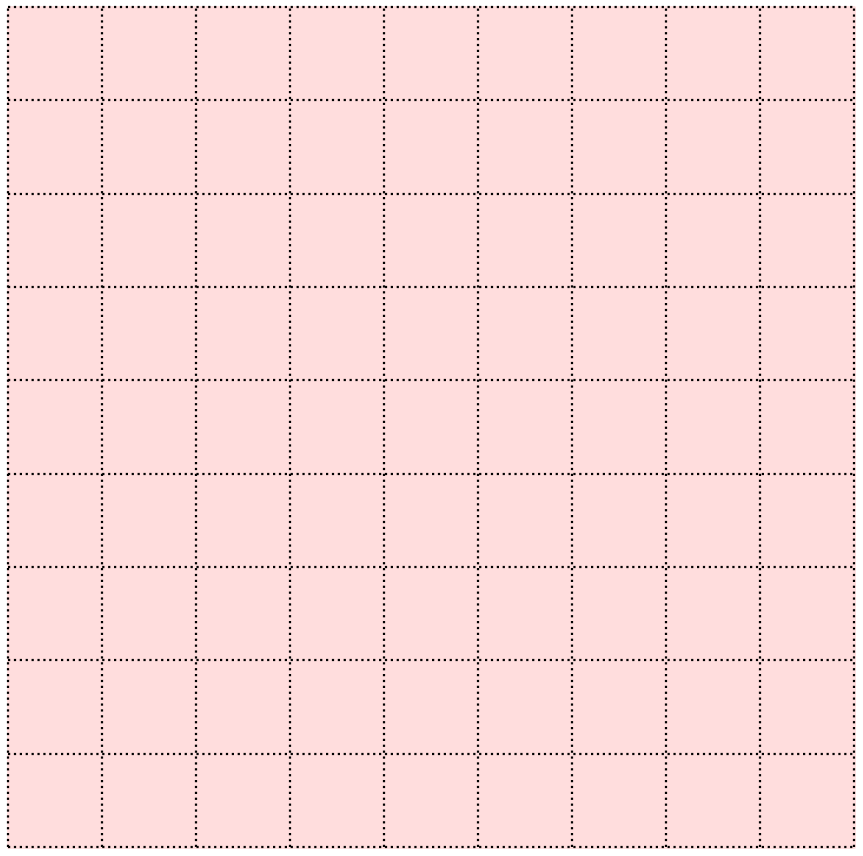}
        \caption{The Euclidean space $\mathbb{R}^2$ as the tangent space}
    \end{subfigure}
    \caption{A two-dimensional sphere, one of its tangent planes, and an isomorphic copy of the latter as the Euclidean space $\mathbb{R}^2$}
\label{fig:manifoldandtangent}
\end{figure}

We will further elaborate on these relationships in Appendix~\ref{appx:dg}, as they will then be necessary for our analysis. 

If $0<r<\min\{K^{-1/2}, i_0\}$, and $x\in\mathcal{M}$, then the exponential map 
\begin{equation} \label{expmap}
    {\rm Exp}_{x}: T_{x}\mathcal{M}\supset B^m(0,r) \to \mathcal{B}(x,r)\subset\mathcal{M},
\end{equation}
where $\T_{x}\mathcal{M}$ denotes the tangent plane of $\mathcal{M}$ at $x$, is a diffeomorphism. Likewise, its inverse, which is the logarithm map
\begin{equation} \label{logmap}
    {\rm Log}_{x}: \mathcal{M}\supset \mathcal{B}(x,r) \to B^m(0,r)\subset T_{x}\mathcal{M},
\end{equation}
is also a diffeomorphism. 

Let $J_{x}(v)$ denote the Jacobian of ${\rm Exp}_{x}$ at $v\in B^{m}(0,r)\subset T_{x}\M$. Then it follows as a consequence of the Rauch comparison theorem that \citep[Chapter~10, Theorem~2.3]{carmo1992riemannian} (see also \citep[Equation~(2.5)]{Dmitri2014Graph})
\begin{equation} \label{eq:Rauch}
    (1+CmK\|v\|^2)^{-1} \leq J_{x}(v)\leq 1 + CmK\|v\|^2,
\end{equation}
which implies 
\begin{equation} \label{volexchange}
    |\mathcal{V}(\mathcal{B}(x,r))-\alpha_{m}r^{m}|\leq CmKr^{m+2},
\end{equation}
where $\mathcal{V}(S) :=\int_{S} \dd\mathcal{V}(x)$, for $S\subset\M$, and $\alpha_m$ is the volume of the $m$-dimensional unit ball. 

\paragraph{From manifold to tangent plane.} The exponential and logarithm maps, \eqref{expmap} and \eqref{logmap} respectively, offer a framework for translating calculus from the manifold $\M$ to its tangent planes.
We introduce a system of changing variables and modified notations aligned with this translation, which will be employed throughout.

Let $x\in\M$, and let $y\in \mathcal{B}(x,r)$ for some $0<r<\min\{K^{-1/2}, i_0\}$. 
Let $g: \mathbb{R}_{\geq 0}\to \mathbb{R}$, and let $f:\M\to\mathbb{R}$ be measurable.
Then according to \eqref{expmap}, there exists a unique $v\in B^m(0,r)$ such that $y = {\rm Exp}_x(v)$, whence
\begin{equation*}
    g(d(x,y)) = g(\|v\|).
\end{equation*}
Further, we write 
\begin{equation} \label{localtildef}
    \tilde{f}(v) := f\circ {\rm Exp}_x(v).
\end{equation}
In this case, $\tilde{f}: B^m(0,r) \to\mathbb{R}$, and particularly, $\tilde{f}(0) = f(x)$. 

\paragraph{Smooth functions on $\M$.}
Note that \eqref{localtildef} allows a function $f$ on $\M$ to be expressed locally as a function $\tilde{f}$ on $B^m(0,r)$. We say that $f\in\mathcal{C}^i(\M)$ if at every $x\in\M$ its local version satisfies $\tilde{f}\in\mathcal{C}^i(B^m(0,r))$.
We write 
$\|f\|_{\mathcal{C}^i(\M)} := \sum_{j=0}^i \|D^{(j)} f\|_{L^{\infty}(\M)}$, where $D^{(j)}$ is the $j$th-differentiation operator.
Then by expressing the covariant derivatives $D^{(j)}f$ in normal coordinates around $x$ \citep[Chapter~2]{hebey1996sobolev}, and applying standard expansions for the Christoffel symbols and metric tensor in these coordinates, we can establish that
\begin{equation} \label{eq:equivalent}
    c\|\tilde{f}\|_{\mathcal{C}^i(B^m(0,r))} \leq \|f\|_{\mathcal{C}^i(\mathcal{B}(x,r))} \leq C\|\tilde{f}\|_{\mathcal{C}^i(B^m(0,r))},
\end{equation}
for some universal constants $C\geq c>0$. 

\paragraph{Basic density assumption.} We have assumed at the beginning of Section~\ref{sec:aug} that the natural data $x_{i}\sim\mu$ i.i.d. for some distribution $\mu$ supported on $\M$. We posit further that $\dd\mu = \rho  \dd \mathcal{V}$, where $\dd \mathcal{V}$ denotes the Riemannian volume form associated with $\M$, and $\rho:\M\to \R_+$ is a density function such that $\rho\in \mathcal{C}^2(\M)$ and that for every $x\in\M$,
\begin{equation} \label{boundeddensity}
        0<\rho_{\min} \leq \rho(x) \leq \rho_{\max}<\infty.
\end{equation}

\paragraph{Basic assumptions on $\eps$.}
We assume the following regarding the graph connectivity parameter $\eps$ in \eqref{def:augmentedLap}: 
\begin{equation} \label{epsassume}
    \eps\leq \min \Big\{1,K^{-1/2},i_0, \frac{R}{2} \Big\}.
\end{equation}
The requirement $\eps\leq\min\{K^{-1/2}, i_0\}$ is justified by the considerations \eqref{expmap}, \eqref{logmap}. 
Further, that $\eps\leq \frac{R}{2}$ implies \eqref{eq:distancecompare}, along with
\begin{equation*}
    \mathbb{P}\bigg(\bigcup_{\bar{x}\in\XX}\{ {\rm dist}(\bar{x},\mathcal{M})\leq R\}\bigg) \geq \mathbb{P}\bigg(\bigcup_{\bar{x}\in\XX}\{ {\rm dist}(\bar{x},\mathcal{M})\leq \eps_{\rm n}\}\bigg) \geq 1-2dn\exp\big(-\eta^{-2/d}/(2d)\big),
\end{equation*}
according to Lemma \ref{lem:probcompact}. 
Throughout, we maintain the option to rescale $\eps$ to $C\eps$, for some universal constant $C>0$, and assume that $C\eps$ continues to satisfy \eqref{epsassume} in place of $\eps$.

\paragraph{Analytic constants and their dependence.} We adhere to the convention that analytic majorant constants, denoted by $C$ or $c$, may vary between occurrences, with $C\geq 1$ and $0<c\leq 1$, implicitly.
All constants are permitted to depend on at most $\rho$ and the manifold $\M$. 
The former includes dependencies on $\rho_{\rm min}$, $\rho_{\rm max}$, $\|\rho\|_{\mathcal{C}^2(\M)}$, for instances, while the latter involves the intrinsic values $K, R, i_0$, the dimension $m$, as well as the quantity
\begin{equation} \label{mandiam}
    \mathfrak{D} := \sup_{x,x'\in\M} \|x-x'\|.
\end{equation}
For clarity, the dependencies of constants are always stated in result statements. When necessary, the constants are further distinguished in these statements.
Specific dependencies, such as on $\eta$ or a function, will be declared by a subscript.
For instance, for the latter, $C_f>0$ denotes a constant depending on $f$, in addition to $\M$, $\rho$ at most, and the nature of such dependence will be announced.
While most parametric dependencies can be made explicit, we do not pursue this.

\paragraph{The notation $\mathcal{O}$.} Let $a>0$. We write $A=\mathcal{O}(a)$ to mean $|A|\leq Ca$ or $\|A\|\leq Ca$, depending on the nature of $A$, where $C>0$ is either a universal constant or dependent on $\M$, $\rho$ at most.
For instance, if $A$, $B$ are scalars, then $A=B + \mathcal{O}(a)$ means $|A-B| \leq Ca$. 
Since the non-asymptotic settings we consider involve small $\varepsilon>0$, this notation convention does not introduce significant ambiguity. 
Further, a subscript will indicate a special dependence.
For instance, $A=\mathcal{O}_f(a)$ means $|A| \leq C_f a$, where $C_f>0$ depends on $f$.

\section{Setup} \label{sec:setup}

In the following discussions, the basic assumptions from Section~\ref{sec:assumption} concerning the geometry of $\M$, the regularity of the density $\rho$ of the natural data, and, in particular, the microscopic range \eqref{epsassume} of the parameter $\eps$, will be understood throughout.



\subsection{Setting for pointwise consistency} \label{sec:pointwiseconsistency}

Recall that $\eps_{\rm n}=\eps^{\tau+1}$, with $\tau>0$.
Recall further from Section~\ref{sec:assumption} that $\tilde{N}_x\M$ denotes the affine normal plane of $\M$ at $x$ in $\mathbb{R}^d$. 
We define
\begin{equation*} 
    \mathcal{N}_{\eps_{\rm n}}(\M) :=\{x+h: x\in\M, \, h \in \tilde{N}_x\M -x, \text{ and } \|h\| < \eps_{\rm n}\}
\end{equation*}
to be an $\eps_{\rm n}$-\textit{tubular neighborhood} of $\M$. 
For $\bar{x}\in \XX\cap\mathcal{N}_{\eps_{\rm n}}(\M)$, we denote by $Q_{\bar{x}}$ the closest point on $\M$ to $\bar{x}$. 
I.e., $Q_{\bar{x}}$ is the orthogonal projection point of $\bar{x}$ onto $\M$, and $\|\bar{x}-Q_{\bar{x}}\|={\rm dist}(\bar{x},\M)$.
By \eqref{epsassume} and the definition of manifold reach, $Q_{\bar{x}}$ exists uniquely for every such $\bar{x}$. 
Further, recall from Lemma~\ref{lem:probcompact} that $\XX=\XX\cap\mathcal{N}_{\eps_{\rm n}}(\M)$ with probability at least $1-2dn\exp(-\eta^{-2/d}/(2d))$.
Hence, from this point onward, we assume the following.

\begin{assumption} \label{assum:nearness}
$\XX\subset \mathcal{N}_{\eps_{\rm n}}(\M)$.
\end{assumption}

Moreover, we demonstrate that there exists a one-to-one correspondence between a point $\bar{x}$ and its projection $Q_{\bar{x}}$. 

\begin{lemma} \label{lem:fullprob}
Let $\bar{x},\bar{x}'\in \XX$ be such that $\bar{x}\not=\bar{x}'$. 
Let $Q_{\bar{x}}$, $Q_{\bar{x}'}$ denote the orthogonal projections of $\bar{x}$, $\bar{x}'$ onto $\M$, respectively. 
Then with probability $1$, $Q_{\bar{x}}\not=Q_{\bar{x}'}$. 
\end{lemma}

\begin{proof}[Proof of Lemma~\ref{lem:fullprob}]
The event $E_1$ that $Q_{\bar{x}}=Q_{\bar{x}'}$ is equivalent to the event $E_2$ that both $\bar{x}$, $\bar{x}'$ belong to $\tilde{N}_{x}\M$ for some $x\in\M$. However, as $\tilde{N}_{x}\M$ is of dimension $d-m$, and not of full dimension $d$, $\mathbb{P}(E_2)=0$. Integrating this probability over $x\sim\mu$, we obtain that $\mathbb{P}(E_1)=0$ as well.  
\end{proof}

A summary of the relevant parameters and hyperparameters is given in Table~\ref{tab:symbols}.

\begin{table}[htb]
\centering
    \begin{tabular}{@{}llll@{}}
    \toprule
        Symbol & & & Meaning \\
		\midrule
        $\eta$, $\tau$ & & & hyperparameters \\
		$\eps$ & & & graph connectivity parameter \\
		$\eps_{\rm w} = \eps^{\tau}$ & & & edge weight parameter\\
		$\eps_{\rm p} = \eta^{1/d}\eps^{\tau+1}$ & & & variance parameter of the data-generating probability\\
        $\eps_{\rm n} = \eps^{\tau+1}$ & & & neighborhood parameter of the augmented data\\
    \bottomrule
    \end{tabular}
    \caption{Parameters/hyperparameters and their meanings. Note that $\tau,\eta>0$ are treated as hyperparameters, i.e. given a priori and not subject to variation.} 
\label{tab:symbols}
\end{table}

Under Assumption~\ref{assum:nearness}, an augmented point $\bar{x}\in\XX$ can be expressed as $z+h$, where $z\in\M$ and $h\in\tilde{N}_z\M\cap\mathcal{N}_{\eps_{\rm n}}(\M) - z$. 
Here, $z$ is unique and represents the projected point $Q_{\bar{x}}$. 
Such points $z$ are identically distributed 
according to a measure $\nu$, which is supported on $\M$ and has a density
\begin{equation} \label{eqdef:q}
    q(z) := \frac{1}{(\eps_{\rm p}\sqrt{2\pi})^d} \int_{h+z\in \tilde{N}_z\M\cap\mathcal{N}_{\eps_{\rm n}(\M)}} \int_{\M} \exp\Big(-\frac{\|z+h-x\|^2}{2\eps_{\rm p}^2}\Big)\rho(x)\dd\mathcal{V}(x)\dd h.
\end{equation}
Furthermore, the regularity of $q$ can be established for a suitable range of $\eps$.

\begin{proposition} \label{prop:qtrunc}
Let $\tau>0$, $\eps>0$, and $\eps_{\rm n}=\eps^{\tau+1}$.
Let $\eps>0$ be further sufficiently small, so that 
\begin{equation} \label{thecondition1}
    \exp\Big(-\frac{c}{\eps_{\rm n}^{2/3}}\Big) \leq \eps_{\rm n}^{m+2}.
\end{equation}
Let $q$ be given in \eqref{eqdef:q}. 
Then under Assumption~\ref{assum:nearness}, for any $z\in\M$, we have
\begin{equation} \label{eq:qboundedextra}
    |q(z)-\rho(z)|\leq C_{\eta}\eps_{\rm n}^{2/3}.
\end{equation}
Moreover, we have $q\in\mathcal{C}^2(\M)$, which means, specifically, for any $z\in\M$,
\begin{equation} \label{eq:qbounded}
    0< c_{\eta} \leq q(z) \leq C_{\eta} <\infty,
\end{equation}
and for any unit vectors $\tilde{u}$, $\tilde{w}\in\tilde{T}_z\M$,
\begin{align}
    \label{eq:qprimebounded} |\partial_{\tilde{u}}q(z)| &\leq C_{\eta},\\
    \label{eq:qdoubleprimebounded} |\partial_{\tilde{w}}\partial_{\tilde{u}}q(z)| &\leq C_{\eta},
\end{align}
and finally, all the relevant partial derivatives are continuous in a neighborhood of $z$ in $\M$. 
The constant $C_{\eta}>0$ in \eqref{eq:qboundedextra} depends on $\eta$ and $\M$. 
Then constants $C_{\eta}, c_{\eta}>0$ in \eqref{eq:qbounded}, \eqref{eq:qprimebounded}, \eqref{eq:qdoubleprimebounded} depend on $\eta$ and at most $\M$, $\rho$.
\end{proposition}

A detailed proof of Proposition~\ref{prop:qtrunc} can be found in Appendix~\ref{appx:qtrunc}.
We briefly remark that \eqref{thecondition1} serves as a crucial condition to ensure the boundedness of the first and second derivatives of $q$ in our proof, which relies heavily on a divide-and-conquer technique.
However, in order to further guarantee the pointwise consistency of $\mathcal{L}_{\rm aug}$, it will be necessary for us to enforce an additional condition, that is
\begin{equation} \label{thecondition2}
    \exp\Big(-\frac{c\eps^{4/3}}{\eps^{2\tau/3}}\Big) \leq \eps_{\rm w}^m\eps^2.
\end{equation}
It is also evident that $\tau>2$ for \eqref{thecondition2} to hold. 
We summarize the key requirements for $\tau$, $\eta$, $\eps$ as follows.

\begin{assumption}[Parameter assumption] \label{assum:para}
Let {$\tau>2$}, $\eta>0$, and let $\eps>0$ be sufficiently small as to satisfy \eqref{epsassume}, \eqref{thecondition1}, \eqref{thecondition2}.
\end{assumption}

We are now ready to formulate the setting for our first main result.

\begin{setting} \label{set:I}
Let $\eta, \tau>0$ and $\eps>0$. 
Let $\eps_{\rm w}$, $\eps_{\rm p}$, $\eps_{\rm n}$ be given in Table~\ref{tab:symbols}.
Let $\M\subset\mathbb{R}^d$ be a closed manifold of dimension $m$, satisfying the fundamental geometry outlined in Section~\ref{sec:assumption}.
Let $\X = \{x_1,\dots,x_N\}$ be the set of natural data, where $x_i\sim\mu$, whose density $\rho\in\mathcal{C}^2(\M)$ satisfies \eqref{boundeddensity}.
Let $\XX = \{\bar{x}_1,\dots,\bar{x}_n\}$ be the set of augmented data generated from $\X$ according to the law \eqref{def:condprob}. 
Let Assumption~\ref{assum:partindependent} hold.
Let $G_{\XX}$ be the augmentation graph on $\XX$ as described in Section~\ref{sec:aug}, and let the associated augmentation graph Laplacian $\mathcal{L}_{\rm aug}$ be given in \eqref{def:augmentedLap}.
Let Assumptions~\ref{assum:nearness},~\ref{assum:para} hold. 
Let $q\in\mathcal{C}^2(\M)$ be the density \eqref{eqdef:q} of the measure $\nu$ to which the projections of the augmented data are distributed on $\M$.
Finally, let $\Delta_{\rm aug}$ be a Laplace-Beltrami operator defined as follows:
\begin{equation} \label{Laplace-Beltrami}
    \Delta_{\rm aug} f(x) := -\frac{1}{2} {\rm div}(q^2\nabla f)(x),
\end{equation}
for every $f\in\mathcal{C}^2(\M)$ and $x\in\M$.
\end{setting}

\begin{remark} \label{rem:modification}
It can be shown, by extending the techniques from the proof of Proposition~\ref{prop:qtrunc}, that $q$ and $\rho$ are also close in $\mathcal{C}^2(\M)$ sense, i.e. 
\begin{equation} \label{closeC2}
    \|q-\rho\|_{\mathcal{C}^2(\M)} = \mathcal{O}_{\eta}(\eps).
\end{equation}
We choose to omit this due to the already extensive technical details involved.
However, once \eqref{closeC2} holds, then we can further rewrite the limit in \eqref{Laplace-Beltrami} to $-\frac{1}{2} {\rm div}(\rho^2\nabla f)(x)$, making the natural data sampling density explicit in the expression.
\end{remark}

\subsection{Setting for spectral consistency} \label{sec:spectralconsistency}

In this section, we demonstrate that the eigenvalues and eigenvectors of $\mathcal{L}_{\rm aug}$ converge to the corresponding eigenvalues and eigenfunctions of the continuum operator $\Delta_{\rm aug}$.
To lay the groundwork, we review the fundamental properties of a graph Laplacian and Laplace-Beltrami operator, with a specific focus on the respective examples of $\mathcal{L}_{\rm aug}$ and $\Delta_{\rm aug}$.

Define on $\XX=\{\bar{x}_1,\dots,\bar{x}_n\}$ a discrete probability measure $\vartheta_n$, that is
\begin{equation} \label{eqdef:varthetan}
    \vartheta_n := \frac{1}{n} \sum_{i=1}^n \delta_{\bar{x}_i},
\end{equation}
where $\delta$ denotes the Dirac delta measure. 
We then endow $L^2_{\vartheta_n}(\XX)$ with an inner product 
\begin{equation*} 
    \langle  f, g \rangle_{L^2_{\vartheta_n}(\XX)} := \frac{1}{n} \sum_{i=1}^n f(\bar{x}_i)g(\bar{x}_i).
\end{equation*}
It is known that $\mathcal{L}_{\rm aug}$ is a positive semi-definite and self-adjoint operator on $L^2_{\vartheta_n}(\XX)$ \citep{von2007tutorial}. Particularly, we can list its eigenvalues in non-decreasing order (repeated according to multiplicity) as follows:
\begin{equation} \label{spectrumgraph}
    0= \hat{\lambda}^{\rm aug}_1\le  \hat{\lambda}^{\rm aug}_2 \le \dots \le \hat{\lambda}^{\rm aug}_n.
\end{equation}
An orthonormal basis for $L^2_{\vartheta_n}(\XX)$ can be constructed from their associated eigenvectors, denoted $\hat{f}^{\rm aug}_1,\dots,\hat{f}^{\rm aug}_n$, where $\mathcal{L}_{\rm aug}\hat{f}^{\rm aug}_l = \hat{\lambda}^{\rm aug}_l\hat{f}^{\rm aug}_l$.
Furthermore, the eigenvalues in \eqref{spectrumgraph} can also be characterized variationally according to the Courant-Fisher min-max principle: 
\begin{equation}\label{minimax}
    \hat{\lambda}^{\rm aug}_l=\min_{S\in \mathbb{G}^l_{\XX}}\max_{f\in S\setminus \{0\}}\frac{E_{\XX}(f)}{\| f\|^2_{L^2_{\vartheta_n}(\XX)}}.
\end{equation}
Here, $\mathbb{G}^l_{\XX}$ denotes the Grassmannian manifold of all linear subspaces of $L^2_{\vartheta_n}(\XX)$ of dimension $l$, and $E_{\XX}$ represents a \textit{graph-based Dirichlet energy}, defined such that for $f\in L^2_{\vartheta_n}(\XX)$: 
\begin{equation} \label{eqn:GraphDirichlet}
    E_{\XX}(f)
    :=\frac{1}{n^2\eps_{\rm w}^m\eps^{m+2}}
    \sum_{i,j=1}^n
    \omega(\bar{x}_i, \bar{x}_j) \mathbbm{1}_{\{0<\|\bar{x}_i-\bar{x}_j\|\leq \eps\}} (f(\bar{x}_i)-f(\bar{x}_j))^2
    =2 \langle\mathcal{L}_{\rm aug}f,f\rangle_{\ell^2(\XX)},
\end{equation}
with $\mathbbm{1}$ being the indicator function.
Thus, for each orthonormal eigenvector $\hat{f}_l$, its graph-based Dirichlet energy $E_{\XX}(\hat{f}_l)$ is given by 
\begin{equation*}
    E_{\XX}(\hat{f}_l) =  \frac{1}{n^2\eps_{\rm w}^m\eps^{m+2}} \sum_{i,j=1}^n \omega(\bar{x}_i, \bar{x}_j) \mathbbm{1}_{\{0<\|\bar{x}_i-\bar{x}_j\|\leq \eps\}} (\hat{f}_l(\bar{x}_i)-\hat{f}_l(\bar{x}_j))^2.
\end{equation*}

Continuing, we endow $L^2(\M)$ with an inner product
\begin{equation*} 
    \langle  f, g \rangle_{L^2(\M)} := \int_{\M} f(x)g(x) \dd\mathcal{V}(x).
\end{equation*}
Next, recall the Laplace-Beltrami operator $\Delta_{\rm aug}$ defined in \eqref{Laplace-Beltrami}.
Then the operator $\Delta_{\rm aug}$ is a positive semi-definite and self-adjoint operator on $L^2(\M)$. 
Consequently, it has a well-defined point spectrum, allowing its eigenvalues to be listed in increasing order (repeated according to multiplicity) as 
\begin{equation} \label{spectrumman}
    0\leq \lambda_1 \leq\lambda_2\leq\dots,
\end{equation}
such that $\lim_{n\to\infty} \lambda_n = \infty$.
Moreover, we can build an orthonormal basis for $L^2(\M)$ comprising eigenfunctions $f_l$ of $\Delta_{\rm aug}$, where $\Delta_{\rm aug} f_l=\lambda_l f_l$.
Denote by $H^1(\M)$ the Sobolev space of functions possessing a weak gradient in $L^2(\M)$.
Alternatively, $H^1(\M)$ is a subspace of $L^2(\M)$ consisting of functions with a finite \textit{continuum Dirichlet energy}, which is expressed for $f\in L^2(\M)$, as
\begin{equation}\label{eqn:ManDirichlet}
    E_{\M}(f):= 
    \begin{cases} 
    \int_{\M} |\nabla f(x)|^2 q^2(x) \dd\mathcal{V}(x) = 2\langle \Delta_{\rm aug} f, f\rangle_{L^2(\M)} &\quad \text{ if } f \in H^1(\M), \\  
    +\infty  &\quad \text{ otherwise.}
    \end{cases}
\end{equation}
Similarly as before, the eigenvalues in \eqref{spectrumman} of $\Delta_{\rm aug}$ can be written variationally in terms of Dirichlet energy; that is
\begin{equation} \label{minmaxman}
    \lambda_{l}=\min_{S \in \mathbb{G}^l_{\M}} \max _{f\in S\setminus\{0\}} \frac{E_{\M}(f)}{\|f\|^2_{L^2(\M)}},
\end{equation}
where $\mathbb{G}^l_{\M}$ denotes the Grassmannian manifold of all linear subspaces of $L^2(\M)$ of dimension $l$.

Inspired by \citep{calder2022improved}, our spectral convergence result hinges on a crucial interpolation between the graph-based and continuum energies. In our context, this interpolation is achieved through an optimal transport that, roughly speaking, connects the sampling probability measure $\mu$ on $\M$ with the discrete probability measure $\vartheta_n$ on $\XX$, mediated by a suitable discrete probability measure $\upsilon_n$ on $\M$, as will be specified later in Proposition~\ref{prop:inftydistance}. To facilitate the construction of such a map, and to assess the error in our subsequent analysis of discrete-to-continuum approximation, we introduce two additional key parameters: $\iota_1$, $\iota_2$. The first measures the proximity of $\upsilon_n$ to $\vartheta_n$ in an optimal transport sense, and the second reflects the uniform closeness of $\upsilon_n$ to $\rho$. Furthermore, a stricter condition on $\eps$ must be imposed, requiring it to lie within a narrower range than specified in \eqref{thecondition2}, that is
\begin{equation} \label{thecondition3}
    \exp\Big(-\frac{c\eps^{4/3}}{\eps^{2\tau/3}}\Big) \leq \eps_{\rm w}^m\eps^3.
\end{equation}
The following summarizes the technical smallness assumptions on all the involved parameters and hyperparameters, ensuring we operate in the appropriate regime for the validity of our theorems.

\begin{assumption}[Parameter assumption] \label{assum:thetadelta}
Let $\tau>2$, $\eta>0$, and let $\eps>0$ be sufficiently small as to satisfy \eqref{epsassume}, \eqref{thecondition1}, \eqref{thecondition3}.
Further, let $\iota_1, \iota_2>0$ satisfy: 
\begin{enumerate}
    \item $\max\{n^{-1/m},C\eps^{3/2}\} < \iota_2\leq \frac{\eps}{4}$, where $C>0$ depends on $\M$;
    \item $\iota_1+\iota_2\leq c$, where $c>0$ depends on $\M$, $\rho$.
\end{enumerate}
\end{assumption}

Observe that $\iota_1$, $\iota_2$ respectively correspond to $\tilde{\delta}$, $\theta$ in \citep{calder2022improved}.
The additional lower bound $C\eps^{3/2}\leq\iota_2$ in Assumption~\ref{assum:thetadelta}, compared to \citep[Assumption~2.3]{calder2022improved}, reflects an extra step necessary in our proof of Theorem~\ref{thm:spectralconvergence}, specifically in Lemma~\ref{lem:discretetolocal}.

We present the setting for our second main result. 

\begin{setting} \label{set:II}
Let $\eta, \tau>0$ and $\eps,\iota_1,\iota_2>0$. 
Let $\eps_{\rm w}$, $\eps_{\rm p}$, $\eps_{\rm n}$ be given in Table~\ref{tab:symbols}.
Let $\M\subset\mathbb{R}^d$ be a closed manifold of dimension $m$, satisfying the fundamental geometry outlined in Section~\ref{sec:assumption}.
Let $\X = \{x_1,\dots,x_N\}$ be the set of natural data, where $x_i\sim\mu$, whose density $\rho\in\mathcal{C}^2(\M)$ satisfies \eqref{boundeddensity}.
Let $\XX = \{\bar{x}_1,\dots,\bar{x}_n\}$ be the set of augmented data generated from $\X$ according to the law \eqref{def:condprob}. 
Let Assumption~\ref{assum:partindependent} hold.
Let $G_{\XX}$ be the augmentation graph on $\XX$ as described in Section~\ref{sec:aug}.
Let the augmentation graph Laplacian $\mathcal{L}_{\rm aug}$ be given in \eqref{def:augmentedLap}, possessing the spectrum \eqref{spectrumgraph} with the corresponding eigenvectors $\hat{f}^{\rm aug}_l$, $l=1,\dots,n$.
Let Assumptions~\ref{assum:nearness},~\ref{assum:thetadelta} hold. 
Finally, let $\Delta_{\rm aug}$ be the Laplace-Beltrami operator defined in \eqref{Laplace-Beltrami}, possessing the spectrum \eqref{spectrumman} with the corresponding eigenfunctions $\hat{f}_l$, $l\in\mathbb{N}$.
\end{setting}

\subsection{Setting for realizability in contrastive learning} \label{sec:contrastivelearn}

We begin by providing a matrix form for $\mathcal{L}_{\rm aug}$ \eqref{def:augmentedLap}.
Let $\mathsf{A}_{\rm aug}\in\R^{n\times n}$ denote the symmetric adjacency matrix of the graph $G_{\XX}$ having the entries 
\begin{equation} \label{adjacency}
    (\mathsf{A}_{\rm aug})_{ij} := \frac{1}{\eps_{\rm w}^m\eps^{m+2}} \omega(\bar{x}_i,\bar{x}_j) \mathbbm{1}_{\{0<\|\bar{x}_i-\bar{x}_j\|\leq\eps\}} \quad\text{ for }\quad i, j= 1,\dots,n.
\end{equation}
Let $\mathsf{D}_{\rm aug}\in\R^{n\times n}$ be the corresponding diagonal degree matrix. Then $(\mathsf{D}_{\rm aug})_{ij}=0$ if $i\not=j$, and
\begin{equation} \label{degree}
    (\mathsf{D}_{\rm aug})_{ii} := {\sum_{j=1}^n} (\mathsf{A}_{\rm aug})_{ij} = {\frac{1}{\eps_{\rm w}^m\eps^{m+2}}} \sum_{j=1}^n \omega(\bar{x}_i,\bar{x}_j) \mathbbm{1}_{\{0<\|\bar{x}_i-\bar{x}_j\|\leq\eps\}}.
\end{equation}
For $f:\XX\to\mathbb{R}$, we identify $f$ with the column vector $f=(f(\bar{x}_1),\dots,f(\bar{x}_n))^{\top}\in\mathbb{R}^n$.
Then under such identification, we verify that
\begin{equation*} 
    \mathcal{L}_{\rm aug} f = {\frac{1}{n}} [\mathsf{D}_{\rm aug} - \mathsf{A}_{\rm aug}] f.
\end{equation*}
{Indeed, adapted from definition \eqref{def:augmentedLap}, at $\bar{x}_i$, we have
\begin{align*}
    \mathcal{L}_{\rm aug} f(\bar{x}_i) &= \frac{1}{n\eps_{\rm w}^m\eps^{m+2}}\sum_{\bar{x}_j\in\XX: 0<\|\bar{x}_j-\bar{x}_i\|\leq\eps} \omega(\bar{x}_i,\bar{x}_j)(f(\bar{x}_i)-f(\bar{x}_j)) \\
    &= \frac{1}{n}\Big(f(\bar{x}_i) (\mathsf{D}_{\rm aug})_{ii} - \sum_{j=1}^n (\mathsf{A}_{\rm aug})_{ij}f(\bar{x}_j)\Big) \\
    &= \frac{1}{n}\Big(\sum_{j=1}^n \delta_{ij} (\mathsf{D}_{\rm aug})_{ij} f(\bar{x}_j) - \sum_{j=1}^n (\mathsf{A}_{\rm aug})_{ij}f(\bar{x}_j)\Big) \\
    &= \frac{1}{n}\sum_{j=1}^n ((\mathsf{D}_{\rm aug})_{ij} -(\mathsf{A}_{\rm aug})_{ij}) f(\bar{x}_j) = \frac{1}{n}[\mathsf{D}_{\rm aug} -\mathsf{A}_{\rm aug}]f,
\end{align*}
where the second equality follows from definitions~\eqref{adjacency},~\eqref{degree}, and the fourth equality uses the fact that $\mathsf{D}_{\rm aug}$ is diagonal.}
Therefore, we will refer to the graph Laplacian operator $\mathcal{L}_{\rm aug}$ as the matrix ${\frac{1}{n}}[\mathsf{D}_{\rm aug} - \mathsf{A}_{\rm aug}]\in\mathbb{R}^{n\times n}$ whenever appropriate\footnote{{The choice to keep the factor $1/n$ explicit is intended to emphasize the underlying empirical Monte–Carlo averaging process, which is natural from the manifold-learning perspective. If instead one rescales $\mathsf{A}\leftarrow\mathsf{A}/n$ and $\mathsf{D}\leftarrow\mathsf{D}/n$ then one recovers the standard graph-theoretic definition of the unnormalized graph Laplacian $\mathcal{L}_{\rm aug} = \mathsf{D}-\mathsf{A}$.}}.
Recall from Section~\ref{sec:assumption} that $\|\cdot\|$ denotes the Euclidean vector norm or the Frobenius matrix norm. 
Let $k\in\mathbb{N}$, which acts as a parameter determining the spectral embedding dimension.
Let $a\geq \max\{(\mathsf{D}_{\rm aug})_{ii}\}$. It follows that $a\mathsf{I}_{n\times n} - \mathcal{L}_{\rm aug}$, where $\mathsf{I}_{n\times n}$ is the $n\times n$ identity matrix, is positive semidefinite.
Then by Eckart-Young-Minsky theorem \citep{eckart1936approximation}, any minimizer $Y^*\in\mathbb{R}^{n\times k}$, of the optimization objective\footnote{Note that \eqref{eq:optimization} can be rewritten as $\min_{Y \in \R^{n \times k}} \ell(Y)$ in terms of the loss function $\ell(Y) := \|YY^{\top} - a\mathsf{I}_{n\times n} + \mathcal{L}_{\rm aug}\|^2$. Then by the reasoning given in \citep{haochen2021provable}, this optimization problem, involving matrix factorization, corresponds to an optimization problem based on their population spectral contrastive loss.}
\begin{align} \label{eq:optimization}
     \min_{Y \in \R^{n \times k}} \|YY^{\top} - a\mathsf{I}_{n\times n} + \mathcal{L}_{\rm aug}\|^2,
\end{align}
contains, as its columns, the scaled versions of the last $k$ normalized (in the Euclidean norm) eigenvectors of $a\mathsf{I}_{n\times n} - \mathcal{L}_{\rm aug}$ (corresponding to the largest $k$ eigenvalues), up to multiplication on the right by an $k\times k$ orthogonal matrix. 
These eigenvectors are in turn, the first $k$ normalized (in the Euclidean norm) eigenvectors of $\mathcal{L}_{\rm aug}$ (corresponding to the smallest $k$ eigenvalues). 
In other words, solving \eqref{eq:optimization} is equivalent to performing spectral embedding of $G_{\XX}$ using $\mathcal{L}_{\rm aug}$.
We aim to learn a minimizer of \eqref{eq:optimization} over a hypothesis set of neural networks (NNs) with parameterizable architectures. 
Specifically, we solve the following optimization problem:
\begin{equation}\label{eq:parameterized optimization}
    \min_{\theta\in\Theta} \| Y_{\theta} Y_{\theta}^{\top} -a\mathsf{I}_{n\times n} + \mathcal{L}_{\rm aug}\|^2,
\end{equation}
where $\theta$ denotes the parameter vector of a NN $f_{\theta}: \R^d \rightarrow \R^k$, and $\Theta$ denotes the space of all admissible parameter vectors. 
Moreover, $Y_{\theta}\in\mathbb{R}^{n\times k}$ contains rows that are the outputs $f_{\theta}(\bar{x}_i)$, representing the embeddings of $\bar{x}_i\in\mathbb{R}^d$ into $\mathbb{R}^k$. 
Thus, the mapping realized by the NN\footnote{In this paper, we will not differentiate between a network architecture and its realization.} $f_{\theta}$ can be interpreted as a \textit{feature} or \textit{representation map} for the input data cloud $\XX$.
We specify the NNs $f_{\theta}$ as multi-layer ReLU neural networks (ReLU NNs), with the architectures given as follows. 
Let $\mathsf{M},\mathsf{L},\mathsf{p}\in\N$ and $\mathsf{m}>0$. 
The NNs $f_{\theta}$ are functions of the form:
\begin{align}\label{eq-def:ReLU}
    \R^d \ni x \quad \mapsto \quad W_{\mathsf{L}} (\operatorname{ReLU}(W_{\mathsf{L}-1} \cdots \operatorname{ReLU}(W_{1} x+b_{1}) \cdots+b_{\mathsf{L}-1})) +b_{\mathsf{L}} \in \R^k.
\end{align}
Here, $W_i\in\mathbb{R}^{\mathsf{p}\times\mathsf{p}}, b_i\in\mathbb{R}^{\mathsf{p}}$ for $i = 2,\dots, \mathsf{L}-1$, $W_1\in\mathbb{R}^{\mathsf{p}\times d}, b_1\in\mathbb{R}^{\mathsf{p}}$, and $W_{\mathsf{L}}\in\mathbb{R}^{k\times\mathsf{p}}, b_{\mathsf{L}} \in\mathbb{R}^k$.
It follows that each network architecture described in \eqref{eq-def:ReLU} is characterized by a parameter vector
\begin{equation} \label{smalltheta}
    \theta := (W_1,b_1,\dots,W_{\mathsf{L}},b_{\mathsf{L}}).
\end{equation}
It can be seen from \eqref{eq-def:ReLU} that $\mathsf{L}$ corresponds to the \textit{network depth} $\mathsf{L}(f_{\theta})$ of $f_{\theta}$, and $\mathsf{p}$ corresponds to the \textit{network width} $\mathsf{p}(f_{\theta})$.
We say a parameter $\theta$ is admissible if, additionally, the following two conditions on \textit{network complexity} hold:
\begin{equation} \label{maxcondition}
    \|W_i\|_{\infty} \leq\mathsf{m} \quad\text{ and }\quad \|b_i\|_{\infty} \leq\mathsf{m} \quad\text{ for }\quad i =1,\dots,\mathsf{L},
\end{equation}
and
\begin{equation} \label{nuclearcondition}
    \sum_{i=1}^{\mathsf{L}} \|W_i\|_0 + \|b_i\|_0\leq \mathsf{M}. 
\end{equation}
Condition \eqref{maxcondition} imposes a bound on the magnitude $\mathsf{m}(f_{\theta})$ of the network weights, whereas condition \eqref{nuclearcondition} restricts the number $\mathsf{M}(f_{\theta})$ of nonzero weights. 
Equivalently, the space $\Theta$ is defined as
\begin{equation} \label{bigtheta}
    \Theta := \{\theta \text{ given in } \eqref{smalltheta}: \theta \text{ satisfies } \eqref{maxcondition}, \eqref{nuclearcondition}\}.
\end{equation}

We conclude the subsection by introducing the setting for our third main result.

\begin{setting} \label{set:III}
Let Setting~\ref{set:II} hold. 
For $k\in\mathbb{N}$, consider the optimization problem \eqref{eq:parameterized optimization}, where 
\begin{equation} \label{Ytheta}
    Y_{\theta} = (f_{\theta}(\bar{x}_1),\dots, f_{\theta}(\bar{x}_n))^{\top}\in\mathbb{R}^{n\times k},
\end{equation}
and $f_{\theta}:\mathbb{R}^d\to\mathbb{R}^k$ is a ReLU NN of the form \eqref{eq-def:ReLU} for some $\mathsf{L}, \mathsf{p}\in\N$.
Moreover, $\theta$ given in \eqref{smalltheta} is a network parameter vector, and the space $\Theta$ of admissible parameters is defined in \eqref{bigtheta} for some $\mathsf{M}\in\N$, $\mathsf{m}>0$.
\end{setting}

\section{Main results} \label{sec:results}

Our main results are framed as probabilistic statements, expressed in terms of all relevant parameters and hyperparameters.

\paragraph{A pointwise consistency result.}
Let 
\begin{equation}\label{eq-def:alpha}
    \alpha := \int_{\mathbb{R}^m} \exp (-\|y\|^2) \dd y
\end{equation}
and 
\begin{equation}\label{eq-def:beta}
    \beta := \int_{B^m(0,1)} |y^1|^2 \dd y
\end{equation}
where,\footnote{A simple computation shows that $\beta = \sigma_m/(m+2)$, with $\sigma_m$ being the volume of the $m$-dimensional Euclidean unit ball.} recall that $y^1$ denotes the first coordinate of $y\in\mathbb{R}^m$. 
Our first result demonstrates that, with high probability, $\mathcal{L}_{\rm aug}$ behaves pointwise in close approximation to an $\alpha\beta\Delta_{\rm aug}$. 
A proof is given in Section~\ref{sec:pointwiseproof}.

\begin{theorem} \label{thm:pointwise}
Assume Setting~\ref{set:I}. 
Let $f\in\mathcal{C}^3(\M)$.
Then for $\eps\leq \delta \leq \eps^{-1}$,
\begin{equation*}
    \mathbb{P}\Big(\max_{\bar{x}\in\XX} \Big|\mathcal{L}_{\rm aug}f(\bar{x}) - \alpha\beta\Delta_{\rm aug} f(Q_{\bar{x}})\Big| > C_{\eta}\delta\|f\|_{\mathcal{C}^3(\mathcal{N}_{\eps_{\rm n}}(\M))} \Big) \leq 8n\exp(-c_{\eta}\delta^2n\eps^{m+2}),
\end{equation*}
for some $C_{\eta},c_{\eta}>0$ depending on $\eta$ and $\M$, $\rho$ at most.
\end{theorem}

We briefly remark that included in Setting~\ref{set:I} is Assumption~\ref{assum:nearness}. 
By Lemma~\ref{lem:probcompact}, the event specified in the assumption holds with probability at least $1-2dn\exp(-\eta^{-2/d}/(2d))$.
This probability can also be incorporated into the statement of Theorem~\ref{thm:pointwise}.

\paragraph{A spectral consistency result.}
We introduce the concept of the eigengap.
Let $\lambda_l$ be the $l$th eigenvalue of $\Delta_{\rm aug}$.
Let $\underline{\lambda}_{i_l}$ denote the $i_l$th \textit{distinct} eigenvalue of $\Delta_{\rm aug}$, listed in an increasing order, that corresponds to $\lambda_l$, i.e. $\lambda_l = \underline{\lambda}_{i_l}$. Then the \textit{eigengap} associated with $\lambda_l$ is defined to be 
\begin{equation} \label{lgap}
    \gamma_{\lambda_l} := \frac{1}{2} \min\{|\underline{\lambda}_{i_l} - \underline{\lambda}_{i_l-1}|, |\underline{\lambda}_{i_l} - \underline{\lambda}_{i_l+1}|\}.
\end{equation}
Our second result establishes the spectral convergence of $\mathcal{L}_{\rm aug}$ to $\alpha\beta\Delta_{\rm aug}$ with high probability. 
A proof is given in Section~\ref{sec:spectralproof}.

\begin{theorem}\label{thm:spectralconvergence}
Assume Setting~\ref{set:II}. 
Let $l\in\mathbb{N}$. 
Let $\lambda_l$, $\hat{\lambda}^{\rm aug}_l$ denote the $l$th eigenvalue of $\Delta_{\rm aug}$, $\mathcal{L}_{\rm aug}$, respectively. 
Suppose it holds that
\begin{equation} \label{lambdalbound}
    \sqrt{\lambda_l}\eps + C(\iota_1+\iota_2) \leq \frac{1}{2l},
\end{equation}
and that
\begin{equation} \label{lambdalgap}
    C(\eps(\sqrt{\lambda_l}+1) + \eps + \iota_1)\lambda_l \leq \gamma_{\lambda_l},
\end{equation}
where $\gamma_{\lambda_l}$ is given in \eqref{lgap}.

Then with probability at least $1-n\exp(-cn\iota_1^2\iota_2^m)-10n\exp(-c_{\eta}n\eps^{m+4})$, for the $l$th normalized eigenvector $\hat{f}^{\rm aug}_l$ (with respect to $L^2_{\vartheta_n}(\XX)$) of $\mathcal{L}^{\rm aug}$ associated with $\hat{\lambda}^{\rm aug}_l$, there exists a normalized eigenfunction $f_l$ (with respect to $L^2(\M)$) of $\Delta_{\rm aug}$ associated with $\lambda_l$, such that
\begin{equation} \label{L2closeness}
    \|\hat{f}^{\rm aug}_l(\cdot) - f_l(Q_{\cdot})\|_{L^2_{\vartheta_n}(\XX)}\leq C_{\eta,l}\eps,
\end{equation}
where $Q_{\cdot}$ denotes the orthogonal projection of a point in $\XX$ onto $\M$.
Here, all the constants $C,c>0$ depend on $\M$, $\rho$ at most; the constant $c_{\eta}>0$ additionally depends on $\eta$, and the constant $C_{\eta,l}>0$ further depends on $l$. 
\end{theorem}

The additional dependence on $l$ of the constant $C_{\eta,l}$ can be explained via Theorem~\ref{thm:pointwise} as follows. 
Since $f_l\in\mathcal{C}^3(\M)$, we can extend $f_l$ to a function $\bar{f}_l$ in the neighborhood $\mathcal{N}_{\eps_{\rm n}}(\M)$, ensuring that $\bar{f}_l(x)=\bar{f}_l(x+h)$ whenever $x+h\in\mathcal{N}_{\eps_{\rm n}}(\M)\cap\tilde{N}_x\M$ (as per the notation in Section~\ref{sec:pointwiseconsistency}). 
Note that this extension is well-defined, because a nonempty intersection between $\mathcal{N}_{\eps_{\rm n}}(\M)\cap\tilde{N}_x\M$ and $\mathcal{N}_{\eps_{\rm n}}(\M)\cap\tilde{N}_{x'}\M$ for $x\not= x'$ on $\M$ would imply $\eps_{\rm n}>R$, contradicting \eqref{epsassume}.
Thus, the dependence on $l$ manifests as a factor of $\|\bar{f}_l\|_{\mathcal{C}^3(\mathcal{N}_{\eps_{\rm n}}(\M))}$, which equals $\|f_l\|_{\mathcal{C}^3(\M)}$.

Note that, under Assumption~\ref{assum:thetadelta}, if $\eps$ scales like $C\big(\frac{\log n}{n}\big)^{\frac{1}{m+4}} \leq \eps\ll 1$, and we set $\iota_1=C\eps^2$ and $\iota_2=c\eps$ for some $C>0$ and $c\in (0,1/4]$, then with high probability the spectral convergence rate obtained from \eqref{L2closeness} scales linearly in $\eps$.
Particularly, if $\eps = C\big(\frac{\log n}{n}\big)^{\frac{1}{m+4}}$, the convergence rate obtained is of $\mathcal{O}_{\eta}\big(\big(\frac{\log n}{n})^{\frac{1}{m+4}}\big)$.
For sufficiently large $n$, such parameter selection also ensures that \eqref{lambdalbound}, \eqref{lambdalgap} are automatically satisfied.

\paragraph{A realizability result.}
Lastly, we provide an approximation guarantee for neural networks in spectral contrastive learning tasks. This result relaxes the realizability assumption \citep[Assumption 3.7]{haochen2021provable} in the spectral contrastive learning framework by asserting the existence of a neural network with bounded complexity capable of learning a minimizer of \eqref{eq:optimization}. 
A proof is given in Section~\ref{sec:contrastiveproof}.

\begin{theorem} \label{thm:contrastive} 
Assume Setting~\ref{set:III}.
Suppose \eqref{lambdalbound}, \eqref{lambdalgap} hold for all $l=1,\dots,k$.
Then with probability at least $1-kn\exp(-cn\iota_1^2\iota_2^m)-k10n\exp(-c_{\eta}n\eps^{m+4})$, the following holds. 
For every $\tilde{\delta}\in (0,1)$ and every minimizer $Y^*\in\mathbb{R}^{n\times k}$ of the problem \eqref{eq:optimization}, there exists a ReLU NN $f_{\theta}$ such that for $Y_{\theta}$ in \eqref{Ytheta},
\begin{equation*}
    \|Y^* - Y_{\theta}\|_{\infty} \leq C_{\eta,k}(\eps + \tilde{\delta}) \quad\text{ and }\quad
    \|Y^* - Y_{\theta}\| \leq C_{\eta,k}(\eps + \tilde{\delta}).
\end{equation*}
Moreover, 
\begin{equation*} 
    \mathsf{L}(f_{\theta}) \leq C\Big(\log\Big(\frac{1}{\tilde{\delta}}\Big) + \log d\Big) \quad\text{ and }\quad
    \mathsf{p}(f_{\theta}) \leq Ck(\tilde{\delta}^{-m} + d),
\end{equation*}
and $\theta\in\Theta$, for $\mathsf{M}\in\N$, $\mathsf{m}>0$, such that
\begin{equation} \label{networkcomplex}
    \mathsf{m} \leq C\quad\text{ and }\quad \mathsf{M} \leq k\Big((\tilde{\delta}^{-m}+d) \log\Big(\frac{1}{\tilde{\delta}}\Big) + d\log d\Big).
\end{equation}
Here, all the constants $C,c>0$ depend only on $\M$; the constant $c_{\eta}>0$ additionally depends on $\eta$, and the constants $C_{\eta,k}>0$ further depend on $k$.
\end{theorem}

The dependence of the constants $C_{\eta,k}$ on $k$ arises partly from a similar dependence in \eqref{L2closeness}. Additionally, these constants include at most an extra factor of $\sqrt{k} \max_{l=1,\dots,k} \sqrt{\hat{\lambda}^{\rm aug}_l}$.

Similarly as above, if $\eps$ scales as $C\big(\frac{\log n}{n}\big)^{\frac{1}{m+4}} \leq \eps \ll 1$, and we set $\iota_1 = C_1\eps^2$, $\iota_2 = c_2\eps$, and $\tilde{\delta} = C_3\eps$ for some constants $C_1, C_3 > 0$ and $c_2 \in (0, 1/4]$, then with high probability, the required network depth scales logarithmically with respect to the inverse of $\eps$, while the network width and the number of neurons scale sublinearly. Specifically, when $\eps = C\big(\frac{\log n}{n}\big)^{\frac{1}{m+4}}$, the network depth is of $\mathcal{O}\big(\log\big(\frac{n}{\log n}\big)\big)$, and the network width and the number of neurons are of $\mathcal{O}\big(\big(\frac{n}{\log n}\big)^{\frac{m}{m+4}}\big)$, up to an extra log factor.
Note that these rates depend solely on the embedded dimension $m$ and not on the ambient dimension $d$, effectively mitigating the curse of dimensionality.

\section{Numerical experiments} \label{sec:numerics}

In this section, we implement the spectral contrastive learning algorithm\footnote{The original implementation is available at \url{https://github.com/jhaochenz96/spectral_contrastive_learning}.} using parameter choices informed by our theoretical analysis\footnote{Our implementation is in \url{https://github.com/chl781/Augmented_spectral_contrastive_learning}.}. We follow the same setup as in \citep{haochen2021provable} and test on benchmark vision datasets. We minimize the empirical spectral contrastive loss with an encoder network $f$ and generate augmentation in each iteration.

\paragraph{Encoder.} The encoder function $f$ contains a backbone network, a projection MLP, and a projection function. The backbone network is a standard ResNet architecture.
The projection MLP is a fully connected network with the backbone network applied to each layer, and ReLU activation applied to each except for the last layer. The projection function takes a vector and projects it onto a sphere with radius $\sqrt{r}$, where $r>0$ is a tunable hyperparameter. 
Additionally, we specify the relationship between the graph connectivity and data generation parameters according to our theory (see \eqref{def:augmentedLap} and Table~\ref{tab:symbols}): we use $k$NN to determine the connectivity parameter (for algorithm stability) and use quantile heuristic to compute $\eps$ which is used to generate Gaussian noise for the augmented dataset. 

\paragraph{Linear evaluation protocol.} Given the pre-trained encoder network, we follow the standard linear
evaluation protocol \citep{chen2021exploring} and train a supervised linear classifier on frozen representations.

\paragraph{Results.} We present the accuracy results on CIFAR-10/100 \citep{krizhevsky2009learning} in Table \ref{tab:accuracy}. Our empirical findings show that our spectral contrastive learning approach either outperforms or performs comparably to two widely used baseline algorithms, SimCLR \citep{chen2020big} and SimSiam \citep{chen2020simple}. 
Notably, our algorithm {builds upon} \citep{haochen2021provable} by incorporating an explicit, {theoretically motivated} choice of the noise parameter. {It therefore avoids} exhaustive hyperparameter tuning, {while exhibiting} similar empirical behavior to that of \citep{haochen2021provable}.

\begin{table}[ht]
\centering
\begin{tabular}{l|ccc|ccc}
\hline Datasets & \multicolumn{3}{c|}{ CIFAR-10 } & \multicolumn{3}{c}{ CIFAR-100 }  \\
\hline Epochs & 200 & 400 & 800 & 200 & 400 & 800  \\
\hline SimCLR & 83.73 & 87.72 & 90.60 & 54.74 & 61.05 & 63.88\\
SimSiam & 87.54 &\textbf{90.31} & 91.40 & 61.56 & 64.96 & 65.87  \\
{\citep{haochen2021provable}} & {\textbf{88.66}} & {90.17} & {\textbf{92.07}} & {\textbf{62.45}} & {\textbf{65.82}} & {66.18} \\
\hline
Ours & 87.80 & 89.79 & 91.70 & 58.64 & 63.81 & \textbf{66.36}
\end{tabular}
\caption{Top-1 accuracy under linear evaluation protocol. SimCLR and SimSiam results are reported in \citep{haochen2021provable}.}
\label{tab:accuracy}
\end{table}

\section{Pointwise consistency of augmentation graph Laplacian} \label{sec:pointwisewhole}

\subsection{From augmentation graph to projected graph} \label{sec:augtoproj}

We maintain Assumption~\ref{assum:nearness}. 
For $\bar{x}\not=\bar{x}'\in \XX$ and each $x\in\M$, we define the random variable
\begin{equation} \label{eqdef:X}
   X=X(x;\bar{x},\bar{x}'):=\mathbbm{1}_{\{\|x-Q_{\bar{x}}\|< \eps_{\rm n}^{2/3} \} \wedge \{\|x-Q_{\bar{x}'}\|< \eps_{\rm n}^{2/3} \}},
\end{equation}
where $Q_{\bar{x}}, Q_{\bar{x}'}$ are the respective projections onto $\M$ of $\bar{x}$, $\bar{x}'$. 
For each such pair, we equip with a connection strength
\begin{equation} \label{eqdef:lambda}
    \varsigma(Q_{\bar{x}}, Q_{\bar{x}'}) := \E_{x\sim \mu} \Big[ X \exp \Big(\frac{-\|Q_{\bar{x}}-x\|^2-\|Q_{\bar{x}'}-x\|^2}{2\eps_{\rm w}^2}\Big) \Big].
\end{equation}
Let $\mathcal{Q}_{\XX}$ denote the set of all $Q_{\bar{x}}$ from $\bar{x}\in\XX$. 
Let $G_{\mathcal{Q}_{\XX}}$ be a graph on the vertex set $\mathcal{Q}_{\XX}$, where each edge $(Q_{\bar{x}}, Q_{\bar{x}'})$ corresponds to an edge $(\bar{x}, \bar{x}')$ in $\XX$, i.e. $0<\|\bar{x}-\bar{x}'\|\leq\eps$, and vice versa.
Construct a projected graph Laplacian operator\footnote{The subscript ``proj'' stands for ``projection.''} $\mathcal{L}_{\rm proj}$ such that for $f: \mathcal{Q}_{\XX}\to\mathbb{R}$,
\begin{equation} \label{eqdef:projLap}
    \mathcal{L}_{\rm proj} f(Q_{\bar{x}}) 
    := \frac{1}{n\eps_{\rm w}^m \eps^{m+2}} \sum_{Q_{\bar{x}'}: 0<\|\bar{x} -\bar{x}'\|\leq \eps} \varsigma(Q_{\bar{x}},Q_{\bar{x}'})(f(Q_{\bar{x}}) - f(Q_{\bar{x}'})) 
\end{equation}
The augmentation graph Laplacian $\mathcal{L}_{\rm aug}$ serves as an approximation to the graph Laplacian $\mathcal{L}_{\rm proj}$.
This is formally stated in the following proposition, whose proof can be found at the end of this subsection.

\begin{proposition} \label{prop:augtoproj} 
Let $\tau> 2$, $\eps>0$, and $\eps_{\rm w}=\eps^{\tau}$, $\eps_{\rm n}=\eps^{\tau+1}$.
Let Assumption~\ref{assum:nearness} hold.
Let $\mathcal{L}_{\rm aug}$ and $\mathcal{L}_{\rm proj}$ be respectively defined as in \eqref{def:augmentedLap} and in \eqref{eqdef:projLap}. Let $f\in \mathcal{C}^1(\mathcal{N}_{\eps_{\rm n}}(\M))$. Then for $\eps$ sufficiently small satisfying \eqref{thecondition2}, and $\eps^2\leq\delta\leq 1$, we have
\begin{equation*}
    \mathbb{P}\Big(\max_{\bar{x}\in\XX} \Big|\mathcal{L}_{\rm aug}f(\bar{x}) - \mathcal{L}_{\rm proj} f(Q_{\bar{x}})\Big| > C_{\eta}\eps\|f\|_{\mathcal{C}^1(\mathcal{N}_{\eps_{\rm n}}(\M))} (1+\delta) \Big) \leq 6n\exp(-c\delta^2n\eps^{m}),
\end{equation*}
where $c>0$ depends on $\M$, $\rho$ at most, and $C_{\eta}>0$ further depends on $\eta$.
\end{proposition}

In preparation, we introduce a key prerequisite result. Using \eqref{eqdef:lambda}, we define an intermediate connection between $\bar{x}, \bar{x}'$ that simultaneously approximates $\omega(\bar{x}, \bar{x}')$ and $\varsigma(Q_{\bar{x}}, Q_{\bar{x}'})$ up to a factor related to the distance of $\XX$ to the manifold $\M$; that is
\begin{equation} \label{def:omega}
    \xi(\bar{x},\bar{x}') 
    := \exp\Big(\frac{-\|\bar{x}-Q_{\bar{x}}\|^2-\|\bar{x}'-Q_{\bar{x}'}\|^2}{2\eps_{\rm w}^2}\Big)\varsigma(Q_{\bar{x}}, Q_{\bar{x}'}).
\end{equation}
The relationship between $\omega$, $\xi$ is clarified by the following proposition, whose technical proof can be found in Appendix~\ref{appx:etaandw}.

\begin{proposition} \label{prop:etaandw} 
Let $\tau>0$, $\eps>0$. 
Let $\omega$, $\xi$ be as in \eqref{def:edgeweight}, \eqref{def:omega}, respectively.  
Then under Assumption~\ref{assum:nearness}, the following holds,
\begin{equation} \label{etaandwNr}
    \exp \big(-C\eps^2\big)\xi(\bar{x},\bar{x}') \leq \omega(\bar{x},\bar{x}') \leq \exp\big(C'\eps^2\big) \xi(\bar{x},\bar{x}') + \exp\Big(-\frac{c\eps^{4/3}}{\eps^{2\tau/3}}\Big),
\end{equation} 
for some universal $c>0$, and $C, C'>0$ depending on $\M$.
\end{proposition}

{Proposition~\ref{prop:etaandw} provides an approximate Pythagorean decomposition. Indeed, this is encoded in the definition~\eqref{def:omega} of $\xi(\bar{x},\bar{x}')$, which separates the contributions $\|x-Q_{\bar{x}}\|^2$ and $\|x-Q_{\bar{x}'}\|^2$ from the contributions $\|\bar{x}-Q_{\bar{x}}\|^2$ and $\|\bar{x}'-Q_{\bar{x}'}\|^2$ to the ambient Euclidean distances $\|x-\bar{x}\|^2$ and $\|x-\bar{x}'\|^2$, respectively, appearing in $\omega(\bar{x},\bar{x}')$.
Since the manifold is curved, such an approximate decomposition holds up to small curvature-induced additive error terms given in \eqref{etaandwNr}. These terms are controlled via the choice of range $\varepsilon_{\rm n}^{2/3}$ in definition~\eqref{eqdef:X} of $X$. In turn, they ensure that the resulting error contributions remain sufficiently lower order in the subsequent estimates.}
Next, we present the final lemma before moving on to the proof of Proposition~\ref{prop:augtoproj}. 
This lemma explains the additional scaling factor $\eps_{\rm w}^m$ in definition \eqref{def:augmentedLap}.

\begin{lemma} \label{lem:omegamag} 
Let $\tau> 0$, $\eps>0$, and $\eps_{\rm w}=\eps^{\tau}$. Then
\begin{equation*}
    \max_{\bar{x},\bar{x}'\in \XX} \varsigma(Q_{\bar{x}}, Q_{\bar{x}'}) \leq C\eps_{\rm w}^m \quad\text{ and }\quad 
    \max_{\bar{x},\bar{x}'\in \XX} \xi(\bar{x}, \bar{x}') \leq C\eps_{\rm w}^m,
\end{equation*}
for some $C>0$ depending on $\M$, $\rho$.
\end{lemma}

\begin{proof}[Proof of Lemma~\ref{lem:omegamag}]
From definition \eqref{def:omega}, it suffices to prove $\|\varsigma\|_{\infty} \leq C\eps_{\rm w}^m$.
By utilizing \eqref{eq:distancecompare} and definition \eqref{eqdef:lambda}, we get
\begin{align} \label{omegamagstep1}
    \nonumber \varsigma(Q_{\bar{x}},Q_{\bar{x}'}) &\leq \int_{\{x:\|x-Q_{\bar{x}}\|< \eps_{\rm n}^{2/3} \}\cap\M} \exp\Big(-\frac{\|Q_{\bar{x}}-x\|^2}{2\eps_{\rm w}^2}\Big)\rho(x)\dd\mathcal{V}(x)\\
    &\leq \int_{\{x: d(Q_{\bar{x}},x)< c\eps_{\rm n}^{2/3} \}\cap\M} \exp\Big(-\frac{cd(Q_{\bar{x}},x)^2}{\eps_{\rm w}^2}\Big)\rho(x)\dd\mathcal{V}(x),
\end{align}
for some $c>0$ depending on $\M$.
Since $\eps$ is sufficiently small \eqref{epsassume}, the map ${\rm Exp}_{Q_{\bar{x}}}: B^m(0,c\eps_{\rm n}^{2/3})\to \mathcal{B}(Q_{\bar{x}},c\eps_{\rm n}^{2/3})$ is a diffeomorphism. 
Hence, recalling the change of variables and notation system given in Section~\ref{sec:assumption}, for $x \in \mathcal{B}(Q_{\bar{x}}, c\eps_{\rm n}^{2/3})$, there exists a unique $y\in B^m(0, c\eps_{\rm n}^{2/3})$ such that ${\rm Exp}_{Q_{\bar{x}}}(y) = x$. Writing $\tilde{\rho}(y) = \rho({\rm Exp}_{Q_{\bar{x}}}(y))$, we obtain from \eqref{boundeddensity}, \eqref{eq:Rauch}, \eqref{omegamagstep1} 
\begin{align*}
    \varsigma(Q_{\bar{x}},Q_{\bar{x}'}) 
    &\leq \int_{B^m(0,c\eps_{\rm n}^{2/3})\subset T_{Q_{\bar{x}}\M}} \exp\Big(-\frac{c\|y\|^2}{\eps_{\rm w}^2}\Big)\tilde{\rho}(y)J_{Q_{\bar{x}}}(y)\dd y\\
    &\leq C\rho_{\rm max}\eps_{\rm w}^m \int_{B^m(0,c\eps_{\rm w}^{-1/3}\eps^{2/3})} \exp(-c\|y\|^2) \dd y \\
    &\leq C\eps_{\rm w}^m,
\end{align*}
and the desired conclusion follows.
\end{proof}

\begin{proof}[Proof of Proposition~\ref{prop:augtoproj}]
We define two auxiliary operators.
For $f:\XX\to\R$, let
\begin{equation*}
    \mathcal{L}_{\XX} f(\bar{x}) 
    := \frac{1}{n\eps_{\rm w}^m\eps^{m+2}} \sum_{\bar{x}': 0<\|\bar{x}-\bar{x}'\|\leq\eps} \xi(\bar{x},\bar{x}') (f(\bar{x})-f(\bar{x}')),
\end{equation*}
and for $f:\mathcal{Q}_{\XX}\to\R$, let
\begin{equation*}
    \mathcal{L}_{\mathcal{Q}_{\XX}} f(Q_{\bar{x}}) 
    := \frac{1}{n\eps_{\rm w}^m\eps^{m+2}} \sum_{\bar{x}': 0<\|\bar{x}-\bar{x}'\|\leq\eps} \xi(\bar{x},\bar{x}')(f(Q_{\bar{x}})-f(Q_{\bar{x}'})).
\end{equation*}
Recall that $f\in \mathcal{C}^1(\mathcal{N}_{\eps_{\rm n}}(\M))$. We make the following claims:
\begin{align}
    \label{claim1} \Big|\mathcal{L}_{\rm aug}f(\bar{x}) - \mathcal{L}_{\XX} f(\bar{x})\Big| &\leq \frac{C\eps\|f\|_{\mathcal{C}^1(\mathcal{N}_{\eps_{\rm n}}(\M))}}{n\eps^{m}} \sum_{Q_{\bar{x}'}: 0<\|Q_{\bar{x}}-Q_{\bar{x}'}\|\leq 2\eps} 1,\\
    \label{claim2} \Big|\mathcal{L}_{\XX} f(\bar{x}) - \mathcal{L}_{\mathcal{Q}_{\XX}} f(Q_{\bar{x}})\Big| &\leq \frac{C\eps^{\tau-1}\|f\|_{\mathcal{C}^1(\mathcal{N}_{\eps_{\rm n}}(\M))}}{n\eps^{m}} \sum_{Q_{\bar{x}'}: 0<\|Q_{\bar{x}}-Q_{\bar{x}'}\|\leq 2\eps} 1,\\
    \label{claim3} \Big| \mathcal{L}_{\mathcal{Q}_{\XX}} f(Q_{\bar{x}}) -  \mathcal{L}_{\rm proj} f(Q_{\bar{x}}) \Big| &\leq \frac{C\eps\|f\|_{\mathcal{C}^1(\mathcal{N}_{\eps_{\rm n}}(\M))}}{n\eps^{m}} \sum_{Q_{\bar{x}'}: 0<\|Q_{\bar{x}}-Q_{\bar{x}'}\|\leq 2\eps} 1.
\end{align}
We will validate \eqref{claim1}, \eqref{claim2}, \eqref{claim3} sequentially.  
First, note that if $\|\bar{x}-\bar{x}'\|\leq \eps$, then 
\begin{equation} \label{essentialinch1}
    \|Q_{\bar{x}}-Q_{\bar{x}'}\| \leq \|\bar{x} - \bar{x}'\| + \|Q_{\bar{x}}-\bar{x}\| + \|\bar{x}'-Q_{\bar{x}'}\| \leq c\eps + 2\eps_{\rm n}\leq 2\eps.
\end{equation}
\noindent \textbf{Proof of \eqref{claim1}.}
Recall from Proposition~\ref{prop:etaandw} that
\begin{equation} \label{compareweights1}
    \Big[\exp \big(-C\eps^2\big) - 1\Big] \xi(\bar{x},\bar{x}') \leq \omega(\bar{x},\bar{x}') - \xi(\bar{x},\bar{x}') \leq \Big[\exp \big(C\eps^2\big) - 1\Big] \xi(\bar{x},\bar{x}') + \exp\Big(-\frac{c\eps^{4/3}}{\eps^{2\tau/3}}\Big),
\end{equation}
which, together with Lemma~\ref{lem:omegamag} and that $\tau>2$, delivers
\begin{equation} \label{eq:prop3f1} 
    |\omega(\bar{x},\bar{x}') - \xi(\bar{x},\bar{x}')| \leq C\eps_{\rm w}^m\eps^2,
\end{equation}
for sufficiently small $\eps$ satisfying \eqref{thecondition2}.
Using \eqref{essentialinch1}, \eqref{eq:prop3f1}, we conclude for $\bar{x}\in\XX$, 
\begin{align*}
    \Big| \mathcal{L}_{\rm aug} f(\bar{x}) -  \mathcal{L}_{\XX} f(\bar{x}) \Big| &\leq \frac{1}{n\eps_{\rm w}^m\eps^{m+2}} \sum_{\bar{x}': 0<\|\bar{x}-\bar{x}'\|\leq\eps} |\omega(\bar{x},\bar{x}') - \xi(\bar{x},\bar{x}')| |f(\bar{x})-f(\bar{x}')|\\
    &\leq \frac{C\eps\|f\|_{\mathcal{C}^1(\mathcal{N}_{\eps_{\rm n}}(\M))}}{n\eps^{m}} \sum_{\bar{x}': 0<\|\bar{x}-\bar{x}'\|\leq\eps} 1\\
    &\leq \frac{C\eps\|f\|_{\mathcal{C}^1(\mathcal{N}_{\eps_{\rm n}}(\M))}}{n\eps^{m}} \sum_{Q_{\bar{x}'}: 0<\|Q_{\bar{x}}-Q_{\bar{x}'}\|\leq 2\eps} 1,
\end{align*}
which completes the proof for \eqref{claim1}.

\noindent \textbf{Proof of \eqref{claim2}.} We note that Assumption~\ref{assum:nearness} implies 
\begin{equation} \label{eq:prop3f2}
    |f(\bar{x})-f(Q_{\bar{x}})| \leq \eps_{\rm n}\|f\|_{\mathcal{C}^1(\mathcal{N}_{\eps_{\rm n}}(\M))} \quad\text{ and }\quad
    |f(\bar{x}')-f(Q_{\bar{x}'})| \leq \eps_{\rm n}\|f\|_{\mathcal{C}^1(\mathcal{N}_{\eps_{\rm n}}(\M))}.
\end{equation}
Hence, from \eqref{essentialinch1}, \eqref{eq:prop3f2} and Lemma~\ref{lem:omegamag} again, we can finish \eqref{claim2}:
\begin{align*}
    \Big| \mathcal{L}_{\XX} f(\bar{x}) -  \mathcal{L}_{\mathcal{Q}_{\XX}} f(\bar{x}) \Big| &\leq \frac{1}{n\eps_{\rm w}^m\eps^{m+2}} \sum_{\bar{x}': 0<\|\bar{x}-\bar{x}'\|\leq\eps} \xi(\bar{x},\bar{x}')\Big( |f(\bar{x})- f(Q_{\bar{x}})| + |f(\bar{x}')-f(Q_{\bar{x}'})|\Big)\\
    &\leq \frac{2\eps_{\rm n}\|f\|_{\mathcal{C}^1(\mathcal{N}_{\eps_{\rm n}}(\M))}}{n\eps^{m+2}} \sum_{Q_{\bar{x}'}: 0<\|Q_{\bar{x}}-Q_{\bar{x}'}\|\leq 2\eps} 1\\
    &\leq \frac{2\eps^{\tau-1}\|f\|_{\mathcal{C}^1(\mathcal{N}_{\eps_{\rm n}}(\M))}}{n\eps^{m}} \sum_{Q_{\bar{x}'}: 0<\|Q_{\bar{x}}-Q_{\bar{x}'}\|\leq 2\eps} 1.
\end{align*}
This completes the proof for \eqref{claim2}.

\noindent \textbf{Proof of \eqref{claim3}.} Assumption~\ref{assum:nearness} gives
\begin{equation*}
    \exp\Big(\frac{-\|\bar{x}-Q_{\bar{x}}\|^2-\|\bar{x}'-Q_{\bar{x}'}\|^2}{2\eps_{\rm w}^2}\Big) \geq \exp(-C\eps^2),
\end{equation*}
whence, from definition~\eqref{def:omega}, it follows that
\begin{equation} \label{compareweights2}
    |\xi(\bar{x},\bar{x}')-\varsigma(Q_{\bar{x}},Q_{\bar{x}'})| \leq \varsigma(Q_{\bar{x}},Q_{\bar{x}'}) \big(1-\exp(-C\eps^2)\big) \leq C\eps_{\rm w}^m\eps^2.
\end{equation}
Subsequently, from \eqref{compareweights2}, and \eqref{essentialinch1} once more,
\begin{align*}
    \Big| \mathcal{L}_{\mathcal{Q}_{\XX}} f(\bar{x}) -  \mathcal{L}_{\rm proj} f(Q_{\bar{x}}) \Big| &\leq \frac{1}{n\eps_{\rm w}^m\eps^{m+2}} \sum_{\bar{x}': 0<\|\bar{x}-\bar{x}'\|\leq\eps} |\xi(\bar{x},\bar{x}')-\varsigma(Q_{\bar{x}},Q_{\bar{x}'})||f(Q_{\bar{x}})- f(Q_{\bar{x}'})| \\
    &\leq \frac{C\eps\|f\|_{\mathcal{C}^1(\mathcal{N}_{\eps_{\rm n}}(\M))}}{n\eps^{m}} \sum_{Q_{\bar{x}'}: 0<\|Q_{\bar{x}}-Q_{\bar{x}'}\|\leq 2\eps} 1.
\end{align*}
This completes the proof for \eqref{claim3}.

We now estimate the right-hand sides of \eqref{claim1}, \eqref{claim2}, \eqref{claim3}. 
Recall from Section~\ref{sec:pointwiseconsistency} that the projection points $Q_{\bar{x}} \sim \nu$ on $\M$, with the probability density $q$. 
By Proposition~\ref{prop:qtrunc}, $q\in \mathcal{C}^2(\M)$ and $q(x)\leq C_{\eta}$ for all $x\in\M$.
By rescaling, we can assume $2\eps$ satisfies \eqref{epsassume} in place of $\eps$.
Let $\eps^2\leq \delta\leq 1$. 
We note that the points $Q_{\bar{x}}$ are not independent, as they are projections of the augmented data points. 
However, under Lemma~\ref{lem:fullprob} and Assumption~\ref{assum:partindependent}, we can assert that with probability $1$, each $Q_{\bar{x}}$ has at most $C_0$ other $Q_{\bar{x}'}$ dependent on $Q_{\bar{x}}$.
Therefore, an application of Lemma~\ref{lem:varcalder2} gives
\begin{equation} \label{D2:step1}
    \begin{split}
        \frac{1}{n\eps^m} \sum_{Q_{\bar{x}'}: \, 0<\|Q_{\bar{x}}-Q_{\bar{x}'}\|\leq 2\eps} 1 \leq \frac{1}{\eps^m}\int_{\mathcal{B}(Q_{\bar{x}}, 2\eps)} q(x)\dd\mathcal{V}(x) + C\delta
    \end{split}   
\end{equation}
with probability at least $1-2\exp(-c\delta^2n\eps^{m})$. 
Using the fact that ${\rm Exp}_{Q_{\bar{x}}}: \mathcal{B}(Q_{\bar{x}},2\eps)\to B^m(0,2\eps)$ is a diffeomorphism, we acquire
\begin{equation} \label{eq:shrink}
    \int_{\mathcal{B}(Q_{\bar{x}},2\eps)} q(x)\dd\mathcal{V}(x) = \int_{B^m(0,2\eps)\subset T_{Q_{\bar{x}}}\M} \tilde{q}(v)J_{Q_{\bar{x}}}(v)\dd v\leq C_{\eta}\eps^m
\end{equation}
where, as outlined in Section~\ref{sec:assumption}, we apply the change of variables $x = {\rm Exp}_{Q_{\bar{x}}}(v)$ and write $\tilde{q}(v) = q({\rm Exp}_{Q{\bar{x}}}(v))$. 
Putting \eqref{eq:shrink} back in \eqref{D2:step1}, we conclude
\begin{equation} \label{D2:step2}
    \frac{1}{n\eps^m} \sum_{Q_{\bar{x}'}: \, 0<\|Q_{\bar{x}}-Q_{\bar{x}'}\|\leq 2\eps} 1 \leq C_{\eta} + C\delta,
\end{equation}
with probability at least $1-2\exp(-c\delta^2n\eps^{m})$.
By combining \eqref{D2:step2} with \eqref{claim1}, \eqref{claim2}, \eqref{claim3}, and recalling that $\tau>2$, we arrive at
\begin{multline*} 
    |\mathcal{L}_{\rm aug}f(\bar{x}) - \mathcal{L}_{\rm proj} f(Q_{\bar{x}})| \\
    \leq |\mathcal{L}_{\rm aug}f(\bar{x}) - \mathcal{L}_{\XX} f(\bar{x})| + |\mathcal{L}_{\XX} f(\bar{x}) - \mathcal{L}_{\mathcal{Q}_{\XX}} f(\bar{x})| + | \mathcal{L}_{\mathcal{Q}_{\XX}} f(\bar{x}) -  \mathcal{L}_{\rm proj} f(\bar{x})| \\
    \leq C_{\eta}\eps\|f\|_{\mathcal{C}^1(\mathcal{N}_{\eps_{\rm n}}(\M))}(1+\delta),
\end{multline*}
with probability at least $1-6\exp(-c\delta^2n\eps^{m})$.
The proposition can now be concluded by another application of union bound over $\bar{x}\in\XX$.
\end{proof}

\subsection{From discrete to nonlocal}

We extend definition \eqref{eqdef:lambda} to a function $\varsigma$ on $\M\times\M$, that is
\begin{equation} \label{eqdef:varsigmaext}
    \varsigma(x,y) := 
    \int_{\{\|z-x\|< \eps_{\rm n}^{2/3} \} \cap \{\|z-y\|< \eps_{\rm n}^{2/3} \}\cap\M} \exp \Big(-\frac{\|z-x\|^2}{2\eps_{\rm w}^2}\Big) \exp \Big(-\frac{\|z-y\|^2}{2\eps_{\rm w}^2}\Big) \rho(z) \dd\mathcal{V}(z).
\end{equation}
Observe that the conclusion of Lemma~\ref{lem:omegamag} remains valid for this extension; namely
\begin{equation} \label{omegamagext}
    \sup_{x,y\in\M} \varsigma(x,y) \leq C\eps_{\rm w}^m. 
\end{equation}
Define a \textit{nonlocal} operator ${\bf L}_{\M}$ such that for $f\in L^2(\M)$, 
\begin{equation} \label{eqdef:nonlocal}
    {\bf L}_{\M} f(x) 
    := 
    \frac{1}{\eps_{\rm w}^m\eps^{m+2}} \int_{\mathcal{B}(x,\eps)} \varsigma(x,y) (f(x) - f(y)) q(y) \dd \mathcal{V}(y).
\end{equation}

The augmentation graph Laplacian $\mathcal{L}_{\rm aug}$ provides, with high probability, a pointwise approximation of the nonlocal operator ${\bf L}_{\M}$, as detailed below.

\begin{lemma} \label{lem:projtononlocal}  
Let {$\tau>2$}, $\eps>0$, and $\eps_{\rm n} = \eps^{\tau+1}$. Let $\mathcal{L}_{\rm aug}$ and ${\bf L}_{\M}$ be respectively defined as in \eqref{def:augmentedLap} and in \eqref{eqdef:nonlocal}. Let $f\in \mathcal{C}^1(\mathcal{N}_{\eps_{\rm n}}(\M))$. Then for $\eps\leq \delta \leq \eps^{-1}$,
\begin{equation*}
    \mathbb{P}\Big(\max_{\bar{x}\in\XX} \Big|\mathcal{L}_{\rm aug}f(\bar{x}) - {\bf L}_{\M} f(Q_{\bar{x}})\Big| > C_{\eta}\delta\|f\|_{\mathcal{C}^1(\mathcal{N}_{\eps_{\rm n}}(\M))} \Big) \leq 8n\exp(-c_{\eta}\delta^2n\eps^{m+2}),
\end{equation*}
for some $C_{\eta},c_{\eta}>0$ depending on $\eta$ and $\M$, $\rho$ at most.
\end{lemma}

\begin{proof}[Proof of Lemma~\ref{lem:projtononlocal}]
We resume from the conclusion of the proof of Proposition~\ref{prop:augtoproj}. 
Let $\bar{x}\in\XX$.
At this stage, we require a concentration result provided by Lemma~\ref{lem:varcalder1}, a variant of Lemma~\ref{lem:varcalder2}.
We apply the concentration result of Lemma~\ref{lem:varcalder1} to $\mathcal{L}_{\rm proj} f(Q_{\bar{x}})$, with the bounded open set $\Omega$ in the lemma corresponding to the orthogonal projection of $B^d(\bar{x},\eps)\cap\mathcal{N}_{\eps_{\rm n}}(\M)$ onto $\M$.
Let $y\in\mathcal{N}_{\eps_{\rm n}}(\M)$ be such that $\|y-\bar{x}\|=\eps$, and let $Q_y$ denote the orthogonal projection of $y$ onto $\M$. 
Then similar to the calculation in \eqref{essentialinch1},
\begin{equation} \label{Qdistanceexp}
    \|Q_{\bar{x}} - Q_y\| \leq \eps + 2\eps_{\rm n} = \eps(1+ 2\eps_{\rm w}),
\end{equation}
and moreover, 
\begin{equation} \label{Qdistanceshr}
    \|Q_{\bar{x}}-Q_y\| \geq \|\bar{x} - y\| - \|Q_{\bar{x}}-\bar{x}\| - \|y-Q_y\| \geq \eps - 2\eps_{\rm n} = \eps(1-2\eps_{\rm w}).
\end{equation}
It follows from \eqref{Qdistanceexp}, \eqref{Qdistanceshr} that $\Omega$, as the orthogonal projection of $B^d(\bar{x},\eps)\cap\mathcal{N}_{\eps_{\rm n}}(\M)$ onto $\M$, satisfies
\begin{equation} \label{eq:twoballs}
    B^d(Q_{\bar{x}},\eps(1- 2\eps_{\rm w}))\cap\M \subset \Omega\subset B^d(Q_{\bar{x}},\eps(1+ 2\eps_{\rm w}))\cap\M,
\end{equation}
and thus, by considering \eqref{eq:distancecompare}, \eqref{volexchange},
\begin{equation} \label{inch1}
    c\eps^m \leq \mathcal{V}(\Omega) \leq C\eps^m. 
\end{equation}
We now apply Lemma~\ref{lem:varcalder1} to $\mathcal{L}_{\rm proj} f(Q_{\bar{x}})$, taking into account Proposition~\ref{prop:qtrunc} and \eqref{omegamagext}, \eqref{inch1}, which gives us, for $0<\delta\leq 1$,
\begin{equation} \label{nearnonlocal1}
    \Big|\eps^{m}\mathcal{L}_{\rm proj} f(Q_{\bar{x}}) - \frac{1}{\eps_{\rm w}^m\eps^2} \int_{\Omega \cap\M} \varsigma(Q_{\bar{x}},y) (f(Q_{\bar{x}}) - f(y)) q(y)\dd \mathcal{V}(y) \Big| 
    \leq C_{\eta}\|f\|_{\mathcal{C}^1(\mathcal{N}_{\eps_{\rm n}}(\M))} \eps^m\delta\eps^{-1}
\end{equation}
with probability at least $1-2\exp (-c_{\eta}\delta^2 n \eps^m)$.
Similarly, from \eqref{eq:distancecompare}, \eqref{eq:Rauch}, as well as Proposition~\ref{prop:qtrunc}, \eqref{omegamagext}, \eqref{eq:twoballs} again, we get
\begin{align} \label{nearnonlocal2}
    \nonumber \frac{1}{\eps_{\rm w}^m\eps^2} \Big| \int_{\Omega\Delta (B^d(Q_{\bar{x}},\eps)\cap\M)}& \varsigma(Q_{\bar{x}},y) (f(Q_{\bar{x}}) - f(y)) q(y)\dd \mathcal{V}(y) \Big| \\
    \nonumber&\leq \frac{C_{\eta}\|f\|_{\mathcal{C}^1(\mathcal{N}_{\eps_{\rm n}(\M)})}}{\eps} \Big| \int_{\Omega\Delta (B^d(Q_{\bar{x}},\eps)\cap\M)} \dd \mathcal{V}(y) \Big| \\
    \nonumber &\leq \frac{C_{\eta}\eps_{\rm w}\eps^m\|f\|_{\mathcal{C}^1(\mathcal{N}_{\eps_{\rm n}(\M)})}}{\eps} \\
    &= C_{\eta} \|f\|_{\mathcal{C}^1(\mathcal{N}_{\eps_{\rm n}(\M)})} \eps^{\tau-1}\eps^m.
\end{align}
Above, we use $\Delta$ to denote the set symmetric difference. 
Therefore, combining \eqref{nearnonlocal1}, \eqref{nearnonlocal2}, yields
\begin{multline} \label{nearnonlocal3}
    \Big|\mathcal{L}_{\rm proj} f(Q_{\bar{x}}) - \frac{1}{\eps_{\rm w}^m\eps^{m+2}} \int_{B^d(Q_{\bar{x}},\eps) \cap\M} \varsigma(Q_{\bar{x}},y) (f(Q_{\bar{x}}) - f(y)) q(y)\dd \mathcal{V}(y) \Big| 
    \\
    \leq C_{\eta} \|f\|_{\mathcal{C}^1(\mathcal{N}_{\eps_{\rm n}}(\M))} \max\{\delta \eps^{-1}, \eps^{\tau-1}\},
\end{multline}
with probability at least $1-2\exp(-c_{\eta}\delta^2 n\eps^m)$.
Furthermore, as part of the conclusion of Lemma~\ref{lem:varcalder2} (or Lemma~\ref{lem:calder2}) and \eqref{omegamagext},
\begin{multline} \label{nearnonlocal4}
    \frac{1}{\eps_{\rm w}^m\eps^{m+2}} \Big|\int_{B^d(Q_{\bar{x}},\eps) \cap\M} \varsigma(Q_{\bar{x}},y) (f(Q_{\bar{x}}) - f(y)) q(y)\dd \mathcal{V}(y) \\
    - \int_{\mathcal{B}(Q_{\bar{x}},\eps)} \varsigma(Q_{\bar{x}},y) (f(Q_{\bar{x}}) - f(y)) q(y)\dd \mathcal{V}(y) \Big| 
    \leq C\|f\|_{\mathcal{C}^1(\mathcal{N}_{\eps_{\rm n}}(\M))}\eps.
\end{multline}
Now recall from the proof of Proposition~\ref{prop:augtoproj} that for $\eps^2\leq\delta\leq 1$,
\begin{equation} \label{nearnonlocal5}
    |\mathcal{L}_{\rm aug}f(\bar{x}) - \mathcal{L}_{\rm proj} f(Q_{\bar{x}})| \leq C\eps\|f\|_{\mathcal{C}^1(\mathcal{N}_{\eps_{\rm n}}(\M))}(1+\delta)
\end{equation}
with probability at least $1-6\exp(-c\delta^2 n\eps^m)$.
Recall that $\tau>2$.
Then by {setting} $\delta' = \delta/\eps$ for $\eps^2\leq\delta\leq 1$, and using \eqref{nearnonlocal3}, \eqref{nearnonlocal4}, \eqref{nearnonlocal5}, along with the definition \eqref{eqdef:nonlocal}, we obtain
\begin{equation*}
    |\mathcal{L}_{\rm aug}f(\bar{x}) - {\bf L}_{\M} f(Q_{\bar{x}})| \leq C_{\eta}\|f\|_{\mathcal{C}^1(\mathcal{N}_{\eps_{\rm n}}(\M))}\delta',
\end{equation*}
for $\eps\leq \delta' \leq \eps^{-1}$, with probability at least $1-8\exp(-c_{\eta}(\delta')^2 n\eps^{m+2})$.
Finally, we conclude the proof with an application of union bound over $\bar{x}\in\XX$.
\end{proof}

\subsection{Proof of Theorem~\ref{thm:pointwise}} \label{sec:pointwiseproof}

\begin{proof}
Lemma~\ref{lem:projtononlocal} establishes that, with high probability $\mathcal{L}_{\rm aug}$ behaves pointwise in close approximation of ${\bf L}_{\M}$. Therefore, to prove Theorem~\ref{thm:pointwise}, it is enough to show that ${\bf L}_{\M}$ is approximately $\alpha\beta\Delta_{\rm aug}$, with a small additive error. 
In the ensuing proof, we will employ the equivalence demonstrated in \eqref{eq:equivalent}.

We define, for $x,y\in\M$,
\begin{equation*}
    \varsigma^*(x,y) := \int_{\{z\in\mathbb{R}^d: d(x,z) < \eps_{\rm n}^{2/3} \wedge d(y,z)< \eps_{\rm n}^{2/3}\} \cap\M} \exp\Big(-\frac{d(x,z)^2}{2\eps_{\rm w}^2}\Big) \exp\Big(-\frac{d(y,z)^2}{2\eps_{\rm w}^2}\Big) \rho(z)\dd\mathcal{V}(z),
\end{equation*}
the version of $\varsigma(x,y)$ in \eqref{eqdef:varsigmaext} based exclusively on geodesic distances.
The following lemma, whose proof is given in Appendix~\ref{appx:varvarstar}, applies. 

\begin{lemma} \label{lem:varvarstar}
Let $\varsigma$, $\varsigma^*$ be defined as above. 
Then for $x,y\in\M$,
\begin{equation*}
    \varsigma(x,y) = \varsigma^*(x,y) + \mathcal{O}(\eps_{\rm w}^m\eps^2).
\end{equation*}
\end{lemma}

Leveraging Lemma~\ref{lem:varvarstar}, we can substitute $\varsigma(x,y)$ with $\varsigma^*(x,y)$ in the expression \eqref{eqdef:nonlocal} for ${\bf L}_{\M}f(x)$. More specifically, since
\begin{align}
    \nonumber \frac{1}{\eps_{\rm w}^m\eps^{m+2}} \bigg|\int_{\mathcal{B}(x,\eps)} (\varsigma(x,y) - \varsigma^*(x,y))& (f(x) - f(y)) q(y)\dd \mathcal{V}(y)\bigg| \\
    \nonumber &\leq \frac{C\eps_{\rm w}^m\eps^2}{\eps_{\rm w}^m\eps^{m+2}} \int_{\mathcal{B}(x,\eps)}  |f(x) - f(y)| q(y)\dd \mathcal{V}(y) \\
    \nonumber &\leq \frac{C_{\eta}\|f\|_{\mathcal{C}^1(\mathcal{B}(x,\eps))} \eps}{\eps^m} \int_{\mathcal{B}(x,\eps)} \dd \mathcal{V}(y) \\
    \label{eq:1stTaystring} &\leq C_{\eta}\|f\|_{\mathcal{C}^1(\mathcal{B}(x,\eps))} \eps,
\end{align}
we get
\begin{equation*}
    {\bf L}_{\M} f(x) = \frac{1}{\eps_{\rm w}^m\eps^{m+2}} \int_{\mathcal{B}(x,\eps)} \varsigma^*(x,y) (f(x) - f(y)) q(y)\dd \mathcal{V}(y) + \mathcal{O}_{f,\eta}(\eps).
\end{equation*}
For $y\in\mathcal{B}(x,\eps)$, we write $\varsigma^*_{x}(y) := \varsigma^*(x,y)$.
Following the change of variables and notation modification introduced in Section~\ref{sec:assumption}, we let $y\in\mathcal{B}(x,\eps)$ correspond uniquely to $v\in B^m(0, \eps)$ such that ${\rm Exp}_{x}(v)=y$, and so ${\rm Exp}_{x}(0)=x$.
We further write $\tilde{\varsigma}^*_{0}(v)=\varsigma^*_{x}(y)$, $\tilde{f}(v) = f(y)$, $\tilde{q}(v)=q(y)$. Thus
\begin{equation} \label{eq:boldL1}
    {\bf L}_{\M} \tilde{f}(0) = -\frac{1}{\eps_{\rm w}^m\eps^{m+2}} \int_{B^m(0,\eps)\subset T_x\M} \tilde{\varsigma}^*_0(v) (\tilde{f}(v)-\tilde{f}(0)) \tilde{q}(v) J_x(v) \dd v + \mathcal{O}_{f,\eta}(\eps).
\end{equation}
From Taylor's expansion, we obtain the following result, with a proof provided in Appendix~\ref{appx:necessityTay}. 

\begin{lemma} \label{lem:necessityTay}
Let $v\in B^m(0,\eps)$.
Then 
\begin{equation} \label{eq:alphaclaim}
    \tilde{\varsigma}^*_0(0) = \eps_{\rm w}^m(\alpha \rho(x) + \mathcal{O}(\eps)).
\end{equation}
Moreover, 
\begin{equation} \label{eq:derivativeclaim}
    \nabla \tilde{\varsigma}^*_0(0) = 0 \quad\text{ and }\quad \|D^2\tilde{\varsigma}^*_0(0)\| = \mathcal{O}(\eps_{\rm w}^m),
\end{equation}
and in particular,
\begin{equation*} 
    \tilde{\varsigma}^*_0(v) = \tilde{\varsigma}^*_0(0) + \mathcal{O}(\eps_{\rm w}^m\eps^2).
\end{equation*}
\end{lemma}

Utilizing Lemma~\ref{lem:necessityTay}, along with \eqref{eq:equivalent}, we see that, as per the computations presented in \eqref{eq:1stTaystring}, 
\begin{align} \label{eq:2ndTaystring}
    \nonumber \frac{1}{\eps_{\rm w}^m\eps^{m+2}} \bigg|\int_{B^m(0,\eps)\subset T_x\M} &(\tilde{\varsigma}^*_0(v) - \tilde{\varsigma}^*_0(0)) (\tilde{f}(v)-\tilde{f}(0)) \tilde{q}(v) J_x(v) \dd v \bigg|\\
    &\leq \frac{C}{\eps^m} \int_{B^m(0,\eps)}  |\tilde{f}(v)-\tilde{f}(0)| \dd v 
    \leq C\|f\|_{\mathcal{C}^1(\mathcal{B}(x,\eps))} \eps.
\end{align}
Consequently, from \eqref{eq:boldL1}, \eqref{eq:2ndTaystring},
\begin{align} \label{eq:boldL2}
    \nonumber {\bf L}_{\M} \tilde{f}(0) 
    &= -\frac{\tilde{\varsigma}^*_0(0)}{\eps_{\rm w}^m\eps^{m+2}} \int_{B^m(0,\eps)\subset T_x\M} (\tilde{f}(v)-\tilde{f}(0)) \tilde{q}(v) J_x(v) \dd v + \mathcal{O}_{f,\eta}(\eps) \\
    &= -\frac{\tilde{\varsigma}^*_0(0)}{\eps_{\rm w}^m\eps^2} \int_{B^m(0,1)} (\tilde{f}(\eps v)-\tilde{f}(0)) \tilde{q}(\eps v) J_x(\eps v) \dd v + \mathcal{O}_{f,\eta}(\eps).
\end{align}
Using Taylor's expansion once again, and considering Proposition~\ref{prop:qtrunc}, we acquire
\begin{equation} \label{eq:Taycomps}
    \begin{split}
        \tilde{f}(\eps v) &= \tilde{f}(0) + \eps \langle \nabla\tilde{f}(0), v\rangle + \frac{1}{2} \eps^2 \langle D^2\tilde{f}(0) v, v \rangle + \mathcal{O}_{f,\eta}(\eps^3) \\
        \tilde{q}(\eps v) &= \tilde{q}(0) + \eps \langle \nabla\tilde{q}(0), v\rangle + \mathcal{O}_{\eta}(\eps^2).
    \end{split}
\end{equation}
Moreover, from \eqref{eq:Rauch}, 
\begin{equation} \label{eq:Jacobian}
    J_{x}(\eps v) = 1 + \mathcal{O}(\eps^2).
\end{equation}
Then putting \eqref{eq:Taycomps} in \eqref{eq:boldL2}, together with simple calculations, delivers 
\begin{equation} \label{eq:boldL3}
    {\bf L}_{\M}\tilde{f}(v) = \frac{\tilde{\varsigma}^*_0(0)}{\eps_{\rm w}^m} ({\bf L}_1 + {\bf L}_2 + {\bf L}_3) + \mathcal{O}_{f,\eta}(\eps),
\end{equation}
where
\begin{align*}
    {\bf L}_1 &:= -\frac{\tilde{q}(0)}{\eps} \int_{B^m(0,1)}  \langle \nabla\tilde{f}(0), v\rangle J_{x}(\eps v)  \dd v \\
    {\bf L}_2 &:= - \int_{B^m(0,1)}  \langle \nabla\tilde{f}(0), v\rangle \langle \nabla\tilde{q}(0), v\rangle J_{x}(\eps v)  \dd v \\
    {\bf L}_3 &:= -\frac{\tilde{q}(0)}{2} \int_{B^m(0,1)}  \langle D^2\tilde{f}(0) v, v \rangle J_{x}(\eps v)  \dd v.
\end{align*}
Observe, due to the symmetry of odd function, 
\begin{equation*}
    \int_{B^m(0,1)}  \langle \nabla\tilde{f}(0), v\rangle  \dd v 
    = 0,
\end{equation*}
which, combined with \eqref{eq:Jacobian}, subsequently gives us
\begin{equation} \label{L1L4}
    {\bf L}_1 = \mathcal{O}_{f,\eta}(\eps).
\end{equation}
Regarding ${\bf L}_2$, ${\bf L}_3$, we also get from \eqref{eq:Jacobian} and a standard computation,
\begin{equation} \label{L2}
    {\bf L}_2 = -\int_{B^m(0,1)}  \langle \nabla\tilde{f}(0), v\rangle \langle \nabla\tilde{q}(0), v\rangle J_{x}(\eps v)  \dd v = - \beta \langle \nabla\tilde{f}(0), \nabla\tilde{q}(0)\rangle + \mathcal{O}_f(\eps^2),
\end{equation}
and 
\begin{equation} \label{L3}
    {\bf L}_3 = -\frac{\tilde{q}(0)}{2} \int_{B^m(0,1)}  \langle D^2\tilde{f}(0) v, v \rangle J_{x}(\varepsilon v)  \dd v = -\frac{\beta\tilde{q}(0)}{2}\Delta_{\M}\tilde{f}(0) + \mathcal{O}_f(\eps^2).
\end{equation}
By combining \eqref{eq:boldL3}, \eqref{L1L4}, \eqref{L2}, \eqref{L3}, and applying Lemma~\ref{lem:necessityTay}, we arrive at
\begin{align*} 
    {\bf L}_{\M} \tilde{f}(0) 
    &= -\frac{\tilde{\varsigma}^*_0(0)\beta}{\eps_{\rm w}^m} \big(\langle \nabla\tilde{f}(0), \nabla\tilde{q}(0)\rangle + \frac{\tilde{q}(0)}{2}\Delta_{\M}\tilde{f}(0) \big) + \mathcal{O}_{f,\eta}(\eps) \\
    &= -\alpha \beta \rho(x) \big(\langle \nabla\tilde{f}(0), \nabla\tilde{q}(0)\rangle + \frac{\tilde{q}(0)}{2}\Delta_{\M}\tilde{f}(0) \big) + \mathcal{O}_{f,\eta}(\eps),
\end{align*}
which simplifies to 
\begin{equation*}
    {\bf L}_{\M} f(x) = {\bf L}_{\M} \tilde{f}(0) = -\frac{\alpha\beta\rho(x)}{2\tilde{q}(0)} {\rm div}(\tilde{q}^2\nabla\tilde{f})(0) + \mathcal{O}_{f,\eta}(\eps),
\end{equation*}
or equivalently, 
\begin{equation} \label{eq:prelocal}
    {\bf L}_{\M} f(x) = -\frac{\alpha\beta\rho(x)}{2q(x)} {\rm div}(q^2\nabla f)(x) + \mathcal{O}_{f,\eta}(\eps).
\end{equation}
Recall from Proposition~\ref{prop:qtrunc} that $\|\rho-q\|_{L^\infty(\M)}\leq C_{\eta}\eps_{\rm n}^{2/3}$. 
Thus, we have, subsequently from \eqref{eq:prelocal}, 
\begin{equation} \label{eq:local}
    {\bf L}_{\M} f(x) = -\frac{\alpha\beta}{2} {\rm div}(q^2\nabla f)(x) + \mathcal{O}_{f,\eta}(\eps) = -\alpha\beta\Delta_{\rm aug} f(x) + \mathcal{O}_{f,\eta} (\eps).
\end{equation}
Finally, to complete the proof, it remains to integrate \eqref{eq:local} with Lemma~\ref{lem:projtononlocal}. 
\end{proof}

\section{Spectral consistency of augmentation graph Laplacian} \label{sec:spectralwhole}

\subsection{Outline for the proof of Theorem~\ref{thm:spectralconvergence}} \label{sec:outlinespectral}

With Theorem~\ref{thm:pointwise} established, the proof of Theorem~\ref{thm:spectralconvergence} is primarily influenced by and aligned with the strategy in \citep{calder2022improved}. 
Before proceeding with further discussion, however, it is necessary to delineate two key differences between our setting and that in the aforementioned work. 
For the first difference, the augmentation graph $G_{\XX}$ is not built from points sampled directly on the manifold $\M$, meaning that, for example, $\vartheta_n$ in \eqref{eqdef:varthetan} is not a measure on $\M$. 
Nevertheless, the groundwork in Section~\ref{sec:pointwiseconsistency} ensures the construction of a unique projection point $Q_{\bar{x}}\in\M$ for each $\bar{x} \in \XX$.
This bijective correspondence, guaranteed with probability $1$ by Lemma~\ref{lem:fullprob}, provides a matching between $\XX=\{\bar{x}_1,\dots,\bar{x}_n\}$ and the set $\mathcal{Q}_{\XX}=\{Q_{\bar{x}_1},\dots,Q_{\bar{x}_n}\}$, thus enabling the transfer between functions on the graph $G_{\XX}$ and those on the graph $G_{\mathcal{Q}_{\XX}}$, as defined at the beginning of Section~\ref{sec:augtoproj}, and vice versa.
In particular, Proposition~\ref{prop:augtoproj} establishes, with high probability, the pointwise proximity between the augmentation graph Laplacian $\mathcal{L}_{\rm aug}$ \eqref{def:augmentedLap} and the projected graph Laplacian $\mathcal{L}_{\rm proj}$ \eqref{eqdef:projLap}. 
For the second difference, the projection points $Q_{\bar{x}}\sim\nu$, which is not a probability measure on $\M$; see \eqref{eqdef:q}.
However, as shown in Proposition~\ref{prop:qtrunc}, $q$ closely approximates the sampling density $\rho$ of $\mu$, which is a probability measure.
Therefore, with $\mathcal{L}_{\rm proj}$ serving as a substitute for $\mathcal{L}_{\rm aug}$, and $\rho$ replacing $q$, the theory developed in \citep{calder2022improved} can be directly applied, with some additional considerations. We outline the application of this theory next.

Define a discrete probability measure $\breve{\vartheta}_n$ to be 
\begin{equation*} 
    \breve{\vartheta}_n := \frac{1}{n} \sum_{i=1}^n \delta_{Q_{\bar{x}_i}},
\end{equation*}
which is a measure on $\M$ with support on $\mathcal{Q}_{\XX}$.
Further, we define a graph-based Dirichlet energy, on the graph $G_{\mathcal{Q}_{\XX}}$, such that for $f\in L^2_{\breve{\vartheta}_n}(\mathcal{Q}_{\XX})$, 
\begin{align} \label{eqdef:EQ}
    \nonumber E_{\mathcal{Q}_{\XX}}(f) := \frac{1}{n^2\eps_{\rm w}^m\eps^{m+2}}
    \sum_{i,j=1}^n
    &\varsigma(Q_{\bar{x}_i}, Q_{\bar{x}_j}) \mathbbm{1}_{\{0<\|\bar{x}_i-\bar{x}_j\|\leq \eps\}} (f(Q_{\bar{x}_i})-f(Q_{\bar{x}_j}))^2 \\
    &= 2\langle \mathcal{L}_{\rm proj} f, f\rangle_{L^2_{\breve{\vartheta}_n}(\mathcal{Q}_{\XX})},
\end{align}
where
\begin{equation*}
    \langle  f, g \rangle_{L^2_{\breve{\vartheta}_n}(\mathcal{Q}_{\XX})} := \frac{1}{n} \sum_{i=1}^n f(Q_{\bar{x}_i})g(Q_{\bar{x}_i})
\end{equation*}
is an inner product on $L^2_{\breve{\vartheta}_n}(\mathcal{Q}_{\XX})$. 
Then $\mathcal{L}_{\rm proj}$ is a positive semi-definite and self-adjoint operator on $L^2_{\breve{\vartheta}_n}(\mathcal{Q}_{\XX})$, whose eigenvalues are given by
\begin{equation} \label{minimaxproj}
    \hat{\lambda}^{\rm proj}_l=\min_{S\in \mathbb{G}^l_{\mathcal{Q}_{\XX}}} \max_{f\in S\setminus \{0\}} \frac{E_{\mathcal{Q}_{\XX}}(f)}{\| f\|^2_{L^2_{\breve{\vartheta}_n}(\mathcal{Q}_{\XX})}},
\end{equation}
where $\mathbb{G}^l_{\mathcal{Q}_{\XX}}$ denotes the Grassmannian manifold of all linear subspaces of $L^2_{\breve{\vartheta}_n}(\mathcal{Q}_{\XX})$ of dimension $l$, and $l=1,\dots,n$.
An orthonormal basis of $L^2_{\breve{\vartheta}_n}(\mathcal{Q}_{\XX})$ can be formed by the associated eigenvectors $\hat{f}^{\rm proj}_l$ of $\mathcal{L}_{\rm proj}$.
Establishing a quantitative interpolation between the graph-based Dirichlet energy $E_{\mathcal{Q}_{\XX}}$ and the continuum energy $E_{\M}$ \eqref{eqn:ManDirichlet} is critical for our spectral convergence analysis.
Such interpolation can be realized through two maps $P: L^2(\M) \to L^2_{\breve{\vartheta}_n}(\mathcal{Q}_{\XX})$ and $\mathcal{I}: L^2_{\breve{\vartheta}_n}(\mathcal{Q}_{\XX}) \to L^2(\M)$, which act as near-isometries when restricted to functions with finite Dirichlet energy (in the discrete and continuum settings).
Precisely, these maps further satisfy estimates of the form
\begin{alignat*}{2}
    E_{\mathcal{Q}_{\XX}}(Pf) &\leq \alpha\beta(1+\text{ small error})E_{\M}(f), \quad &&\forall f\in L^2(\M), \\
    \alpha\beta E_{\M} (\mathcal{I}f), &\leq (1+\text{ small error})E_{\mathcal{Q}_{\XX}}(f) \quad &&\forall f\in L^2_{\breve{\vartheta}_n}(\mathcal{Q}_{\XX});
\end{alignat*}
see Proposition~\ref{prop:discretoeng} in the next subsection. 
When combined with the result of Theorem~\ref{thm:pointwise} and the variational identities \eqref{minimaxproj}, \eqref{minmaxman}, these inequalities translate to proximity between the spectra of $\mathcal{L}_{\rm proj}$ and $\alpha\beta\Delta_{\rm aug}$, following the steps in \citep{calder2022improved}.
Then the final step of our analysis involves leveraging Proposition~\ref{prop:augtoproj} and \eqref{minimax} to demonstrate the spectral alignment of $\mathcal{L}_{\rm aug}$ and $\alpha\beta\Delta_{\rm aug}$. 

For the remainder of this subsection, we outline the construction of the maps $P$, $\mathcal{I}$, which are constructed using an $\infty$-\textit{optimal transport} ($\infty$-OT) map between $\breve{\vartheta}_n$ and a suitably chosen probability measure on $\M$, provided below. 

\begin{proposition} \label{prop:inftydistance}
Let $\iota_1, \iota_2>0$ satisfy\footnote{The proof of \citep[Proposition 2.12]{calder2022improved} only requires that $\iota_2$ be of small, fixed scale, greater than $n^{-1/m}$.}:
\begin{enumerate}
    \item $n^{-1/m} < \iota_2\leq 1$;
    \item $\iota_1+\iota_2\leq c$, where $c>0$ depends on $\M$, $\rho$.
\end{enumerate}
Then with probability at least $1-n\exp(-c'n\iota_1^2\iota_2^m)$, there exist a probability measure $\upsilon_n$ with density $p_n: \M\to\mathbb{R}$, such that
\begin{equation*} 
    \min_{S_{\#}\upsilon_n=\breve{\vartheta}_n} \sup_{x\in\M} d(x,S(x)) \leq \iota_2,
\end{equation*}
and that
\begin{equation*} 
    \|p_n - \rho\|_{L^{\infty}(\M)} \leq C(\iota_1 + \iota_2).
\end{equation*}
Here, the constant $C>0$ depends on $\M$, $\rho$, and the constant $c'>0$ depends solely on\footnote{specifically on its dimension $m$.} $\M$. 
\end{proposition}

This proposition directly follows from \citep[Proposition 2.12]{calder2022improved} and its proof.

We denote by $T$ an $\infty$-OT map between $\breve{\vartheta}_n$ and $\upsilon_n$ given in Proposition~\ref{prop:inftydistance}. 
Let $U_i := T^{-1}(\{Q_{\bar{x}_i}\})$.
We define an \textit{adaptive} contraction discretization map $P: L^2(\M)\to L^2_{\breve{\vartheta}_n}(\mathcal{Q}_{\XX})$ by
\begin{equation} \label{eqdef:P}
    Pf (Q_{\bar{x}_i}) := \frac{n\int_{U_i} f(x)p_n(x) \dd \mathcal{V}(x)}{\rho(Q_{\bar{x}_i})^{1/2}},
\end{equation}
as well as an adaptive extension map $P^*: L^2_{\breve{\vartheta}_n}(\mathcal{Q}_{\XX})\to L^2(\M)$, by
\begin{equation} \label{eqdef:P*}
    P^*f (x) := \sum_{i=1}^n f(Q_{\bar{x}_i}) \rho(Q_{\bar{x}_i})^{1/2} \mathbbm{1}_{U_i}(x).
\end{equation}
Next, we describe the construction of an \textit{interpolation} map $\mathcal{I}: L^2_{\breve{\vartheta}_n}(\mathcal{Q}_{\XX})\to H^1(\M)\subset L^2(\M)$.
It will be apparent below that such a map inherently possesses a smoothing property.
Let $\phi: [0,\infty)\to [0,\infty)$ be a non-increasing function of choice with support on the interval $[0,1]$, whose restriction to $[0,1)$ is Lipschitz continuous and non-vanishing. 
Let $\psi: [0,\infty)\to [0,\infty)$ be a smooth function, given by
\begin{equation*}
    \psi(t) := \frac{1}{\beta} \int_t^{\infty} \phi(s) s\dd s,
\end{equation*}
where we recall $\beta$ is defined in \eqref{eq-def:beta}.

Lastly, for a \textit{bandwidth} $r>0$, let $\Lambda_r: L^2(\M)\to H^1(\M)$ be a convolution operator with kernel
\begin{equation*}
    K_r(x,y) := \frac{1}{r^m} \psi\Big(\frac{d(x,y)}{r}\Big),
\end{equation*}
such that
\begin{equation*}
    \Lambda_r f(x) := \Big(\int_{\M} K_r(x,y) \dd \mu(y)\Big)^{-1} \int_{\M} K_r(x,y) f(y) \dd \mu(y).
\end{equation*}
We then build an interpolation map $\mathcal{I}: L^2_{\breve{\vartheta}_n}(\mathcal{Q}_{\XX})\to H^1(\M)$ with bandwidth $\eps - 2\iota_2$, specifically
\begin{equation} \label{eqdef:calI}
    \mathcal{I}f(x) := \Lambda_{\eps-2\iota_2} P^*f(x).
\end{equation}

\subsection{Proof of Theorem~\ref{thm:spectralconvergence}} \label{sec:spectralproof}

\begin{proof}
The beginning of the proof follows almost the same approach as that of \citep[Theorem~2.4 (i)]{calder2022improved}, and, subsequently, that of \citep[Theorem~2.7 (ii)]{calder2022improved}, with some modifications. 
To further clarify, we introduce Propositions~\ref{prop:discretoeng},~\ref{prop:eignveccompare} below.
They correspond to \citep[Propositions~4.1,~4.2]{calder2022improved}, respectively, which are the two main components for the proofs of \citep[Theorems~2.4,~2.7]{calder2022improved}. 
We recall the discretization map $P$ defined in \eqref{eqdef:P} and the interpolation map $\mathcal{I}$ defined in \eqref{eqdef:calI}.

\begin{proposition} \label{prop:discretoeng}
Let $\eps>0$ and $\iota_1, \iota_2>0$ satisfy Assumption~\ref{assum:thetadelta}. 
Then both of the following hold with probability at least $1-n\exp(-cn\iota_1^2\iota_2^m)$,
\begin{alignat}{2}
    \label{discretetoeng} E_{\mathcal{Q}_{\XX}}(Pf) &\leq \alpha\beta\Big(1 + C_{\eta}\Big(\frac{\iota_2}{\eps} + \eps + \iota_1\Big)\Big) E_{\M}(f), \quad &&\forall f\in L^2(\M), \\
    \label{engtodiscrete} \alpha\beta E_{\M}(\mathcal{I}f) &\leq \Big(1 + C_{\eta}\Big(\frac{\iota_2}{\eps} + \eps + \iota_1\Big)\Big) E_{\mathcal{Q}_{\XX}}(f), \quad &&\forall f\in L^2_{\breve{\vartheta}_n}(\mathcal{Q}_{\XX}),
\end{alignat}
where $\alpha$ and $\beta$ are defined in \eqref{eq-def:alpha} and \eqref{eq-def:beta}, respectively; the constants $C_{\eta}>0$ depend on $\eta$ and $\M$, $\rho$, and the constant $c>0$ depends solely on $\M$.
\end{proposition}

\begin{proposition} \label{prop:eignveccompare}
Let $\eps>0$ and $\iota_1, \iota_2>0$ satisfy:
\begin{enumerate}
    \item $n^{-1/m} < \iota_2\leq \frac{\eps}{4}$;
    \item $\iota_1+\iota_2\leq c$, where $c>0$ depends on $\M$, $\rho$.
\end{enumerate}
Then both of the following hold with probability at least $1-n\exp(-c'n\iota_1^2\iota_2^m)$,
\begin{alignat}{2}
    \label{eignveccompare1} 
    \Big| \|f\|_{L^2(\M)}^2 {-} \|Pf\|_{L^2_{\breve{\vartheta}_n}(\mathcal{Q}_{\XX})}^2 \Big| &{\leq} C\iota_2\|f\|_{L^2(\M)}\sqrt{E_{\M}(f)} {+} C(\iota_1{+}\iota_2)\|f\|_{L^2(\M)}^2,\hspace{2mm} &&\forall f\in L^2(\M), \\
    \label{eignveccompare2} \Big| \|f\|_{L^2_{\breve{\vartheta}_n}(\mathcal{Q}_{\XX})}^2 {-} \|\mathcal{I}f\|_{L^2(\M)}^2 \Big| &{\leq} C\eps\|f\|_{L^2_{\breve{\vartheta}_n}(\mathcal{Q}_{\XX})}\sqrt{E_{\mathcal{Q}_{\XX}}(f)} {+} C(\iota_1{+}\iota_2)\|f\|_{L^2_{\breve{\vartheta}_n}(\mathcal{Q}_{\XX})}^2,\hspace{2mm} &&\forall f\in L^2_{\breve{\vartheta}_n}(\mathcal{Q}_{\XX}).
\end{alignat}
Here, all the constants $C>0$ depend on $\M$, $\rho$, and the constant $c'>0$ depends solely on $\M$.
\end{proposition}

Note that these are the near-isometric inequalities discussed at the beginning of Section~\ref{sec:outlinespectral}.

The proof of Proposition~\ref{prop:discretoeng} mirrors that of \citep[Proposition~4.1]{calder2022improved}, and the proof of Proposition~\ref{prop:eignveccompare} mirrors that of \citep[Proposition~4.2]{calder2022improved}. The former, given in Section~\ref{sec:discretoeng}, is more involved than the latter, outlined in Section~\ref{sec:eignveccompare}.
One modification in these proofs arises from our definition of discretization and extension maps, \eqref{eqdef:P} and \eqref{eqdef:P*}, which explicitly incorporate a local density factor. 
This definition is guided by the specific formulation \eqref{eqdef:EQ} of the energy $E_{\mathcal{Q}_{\XX}}$ and contrasts with the discretization and extension maps in \citep{calder2022improved}.
For instance, see the proof of Lemma~\ref{lem:discretetolocal} for how this difference is addressed.
Another modification stems from the fact that an edge $(Q_{\bar{x}},Q_{\bar{x}'})$ in the graph $G_{\mathcal{Q}_{\XX}}$ is designed to correspond to the edge $(\bar{x},\bar{x}')$ in the graph $G_{\XX}$, thus causing the Euclidean distance between $Q_{\bar{x}}, Q_{\bar{x}'}$ to vary from $\eps(1-2\eps_{\rm w})$ to $\eps(1+2\eps_{\rm w})$. 
This differs from the $\eps$- and $k$NN graphs considered in \citep{calder2022improved}, where all edges have a fixed length.
To handle this, we will take advantage of the negligible difference between $\eps(1-2\eps_{\rm w})\approx \eps \approx \eps(1+2\eps_{\rm w})$, and refer to the technical result in \citep[Lemma~5]{garcia2020error} for further details.

Given Propositions~\ref{prop:discretoeng},~\ref{prop:eignveccompare}, we can follow the reasoning in \citep[Theorem 2.4 (i)]{calder2022improved} (minmax principle) to conclude that, for the first step, if
\begin{equation} \label{lambdalcondition1}
    \sqrt{\lambda_l}\iota_2 + C(\iota_1+\iota_2) \leq \sqrt{\lambda_l}\eps + C(\iota_1+\iota_2) \leq \frac{1}{2l},
\end{equation}
then with probability at least $1-n\exp(-cn\iota_1^2\iota_2^m)$, 
\begin{equation} \label{projeigvalprox}
    |\lambda_l - \alpha\beta\hat{\lambda}^{\rm proj}_l|\leq C\Big(\eps(\sqrt{\lambda_l}+1) + \eps + \iota_1\Big)\lambda_l.
\end{equation}
Note that the probability derived here matches the probability that the conclusions of Propositions~\ref{prop:discretoeng} and~\ref{prop:eignveccompare} hold, which, in turn, is guaranteed whenever the conclusion of Proposition~\ref{prop:inftydistance} holds.
For the second step, we suppose in addition that 
\begin{equation} \label{lambdalcondition2}
    C(\eps(\sqrt{\lambda_l}+1) + \eps + \iota_1)\lambda_l \leq \gamma_{\lambda_l}
\end{equation}
holds.
We infer from the reasoning near the end of the proof in \citep[Theorem 2.7(ii)]{calder2022improved} (Davis-Kahan-type expansion), and the proof of Lemma~\ref{lem:projtononlocal}, the existence of an eigenfunction $f_l$ of $\Delta_{\rm aug}$, with probability at least $1-n\exp(-cn\iota_1^2\iota_2^m)-2n\exp(-c_{\eta} n\eps^{m+4})$, such that
\begin{equation} \label{projspectra}
    \|f_l - \hat{f}^{\rm proj}_l\|_{L^2_{\breve{\vartheta}_n}(\mathcal{Q}_{\XX})} \leq C_{\eta,l}\eps. 
\end{equation}
The additional reduction in probability stems from the pointwise consistency probability, transitioning from $\mathcal{L}_{\rm proj}$ to the nonlocal operator ${\bf L}_{\M}$ \eqref{eqdef:nonlocal} (with $\delta=\eps$) in the proof of Lemma~\ref{lem:projtononlocal}, and ultimately to $\Delta_{\rm aug}$. 

With \eqref{projspectra}, we complete the first part of our analysis, which establishes the spectral alignment between $\mathcal{L}_{\rm proj}$ and $\Delta_{\rm aug}$, as discussed at the beginning of Section~\ref{sec:outlinespectral}.
In order to further extend \eqref{projspectra} to a conclusion about the spectral proximity between $\mathcal{L}_{\rm aug}$ and $\alpha\beta\Delta_{\rm aug}$, we define a \textit{matching} map ${\rm M}$ that sends $\bar{x}\in\XX$ to $Q_{\bar{x}}\in\mathcal{Q}_{\XX}$. 
By Lemma~\ref{lem:fullprob}, this map ${\rm M}$ establishes a bijection between $\XX$ and $\mathcal{Q}_{\XX}$ with probability $1$.
Thus, if $f\in L^2_{\breve{\vartheta}_n}(\mathcal{Q}_{\XX})$, then $f\circ {\rm M}\in L^2_{\vartheta_n}(\XX)$, and conversely, if $f\in L^2_{\vartheta_n}(\XX)$, then $f\circ {\rm M}^{-1}\in L^2_{\breve{\vartheta}_n}(\mathcal{Q}_{\XX})$. 
Analogous to Propositions~\ref{prop:discretoeng},~\ref{prop:eignveccompare}, we have the following proposition.

\begin{proposition} \label{prop:augtodiscret}
Under Assumption~\ref{assum:nearness}, it holds with probability $1$ that 
\begin{alignat}{2}
    \label{eignveccompare31} \|f\|_{ L^2_{\vartheta_n}(\XX)} &= \|f\circ {\rm M}^{-1}\|_{ L^2_{\breve{\vartheta}_n}(\mathcal{Q}_{\XX})}, \quad &&\forall f\in  L^2_{\vartheta_n}(\XX), \\
    \label{eignveccompare32} \|f\|_{ L^2_{\breve{\vartheta}_n}(\mathcal{Q}_{\XX})} &= \|f\circ {\rm M} \|_{ L^2_{\vartheta_n}(\XX)}, \quad  &&\forall f\in  L^2_{\breve{\vartheta}_n}(\mathcal{Q}_{\XX}).
\end{alignat}
Moreover, let $\eps>0$ be sufficiently small to satisfy \eqref{thecondition3}. 
Then both of the following hold with probability at least $1-2n\exp(-c_{\eta}n\eps^{m+4})$,
\begin{alignat}{2}
    \label{augtodiscrete} 
    E_{\XX}(f) &\leq (1+C\eps^2) E_{\mathcal{Q}_{\XX}}(f\circ {\rm M}^{-1}) + C_{\eta}\eps\|f\|_{ L^2_{\vartheta_n}(\XX)}^2, \quad &&\forall f\in  L^2_{\vartheta_n}(\XX), \\
    \label{discretetoaug} 
    E_{\mathcal{Q}_{\XX}}(f) &\leq (1+C\eps^2) E_{\XX}(f\circ {\rm M}), \quad &&\forall f\in  L^2_{\breve{\vartheta}_n}(\mathcal{Q}_{\XX}).
\end{alignat}
Here, the constants $C_{\eta},c_{\eta}>0$ depend on $\eta$ and $\M$, $\rho$ at most, while the constants $C>0$ in \eqref{augtodiscrete}, \eqref{discretetoaug} depend on $\M$.
\end{proposition}

The proof of Proposition~\ref{prop:augtodiscret} is presented in Section~\ref{sec:augtodiscret}.
By combining \eqref{discretetoeng}, \eqref{augtodiscrete}, we obtain, for $f\in L^2(\M)$,
\begin{equation} \label{thm2key1}
    E_{\XX}(Pf\circ {\rm M}) \leq \alpha\beta (1+C\eps^2) \Big(1 + C_{\eta}\Big(\frac{\iota_2}{\eps} + \eps + \iota_1\Big)\Big) E_{\M}(f) + C_{\eta}\eps\|Pf\circ {\rm M}\|_{L^2_{\vartheta_n}(\XX)}^2,
\end{equation}
and by combining \eqref{engtodiscrete}, \eqref{discretetoaug}, we get, for $f\in  L^2_{\vartheta_n}(\XX)$,
\begin{equation} \label{thm2key2}
    \alpha\beta E_{\M}(\mathcal{I}(f\circ {\rm M}^{-1})) \leq (1+C\eps^2) \Big(1 + C_{\eta}\Big(\frac{\iota_2}{\eps} + \eps + \iota_1\Big)\Big) E_{\XX}(f).
\end{equation}
Similarly, by combining \eqref{eignveccompare1} with \eqref{eignveccompare32} and \eqref{eignveccompare2} with \eqref{eignveccompare31}, \eqref{discretetoaug}, we respectively derive 
\begin{alignat}{2}
    \label{thm2key3} \Big| \|f\|_{L^2(\M)}^2 {-} \|Pf\circ {\rm M}\|_{L^2_{\vartheta_n}(\XX)}^2 \Big| &{\leq} C\iota_2\|f\|_{L^2(\M)}\sqrt{E_{\M}(f)} {+} C(\iota_1+\iota_2)\|f\|_{L^2(\M)}^2, \hspace{2mm} &&\forall f\in L^2(\M), \\
    \label{thm2key4} \Big| \|f\|_{L^2_{\vartheta_n}(\XX)}^2 {-} \|\mathcal{I}(f\circ {\rm M}^{-1})\|_{L^2(\M)}^2 \Big| &{\leq} C\eps\|f\|_{ L^2_{\vartheta_n}(\XX)}\sqrt{E_{\XX}(f)} {+} C(\iota_1{+}\iota_2)\|f\|_{ L^2_{\vartheta_n}(\XX)}^2, \hspace{2mm} &&\forall f\in  L^2_{\vartheta_n}(\XX).
\end{alignat}
With \eqref{thm2key1}, \eqref{thm2key2}, \eqref{thm2key3}, \eqref{thm2key4} established, we proceed similarly to before, directly adapting the remainder of the proof to the variational argument in \citep[Theorem 2.4 (i)]{calder2022improved}, and subsequently to \citep[Theorem 2.7 (ii)]{calder2022improved}, line by line.
In particular, if \eqref{lambdalcondition1} holds, we derive that in place of \eqref{projeigvalprox}, 
\begin{equation*} 
    |\lambda_l - \alpha\beta\hat{\lambda}^{\rm aug}_l|\leq C\Big(\eps(\sqrt{\lambda_l}+1) + \eps + \iota_1\Big)\lambda_l,
\end{equation*}
with probability at least $1-n\exp(-cn\iota_1^2\iota_2^m)-2n\exp(-c_{\eta}n\eps^{m+4})$.
The additional reduction in probability, relative to what is obtained previously, arises from the probability that both \eqref{augtodiscrete} and \eqref{discretetoaug} hold.
Finally, if \eqref{lambdalcondition2} is also met, we are guaranteed with probability at least $1-n\exp(-cn\iota_1^2\iota_2^m)-10n\exp(-c_{\eta}n\eps^{m+4})$, the existence of an eigenfunction $f_l$ of $\Delta_{\rm aug}$ such that
\begin{equation*}
    \|\hat{f}^{\rm aug}_l(\cdot) - f_l(Q_{\cdot})\|_{ L^2_{\breve{\vartheta}_n}(\mathcal{Q}_{\XX})} \leq C_{\eta,l}\eps. 
\end{equation*}
The further reduction in probability comes from leveraging Proposition~\ref{prop:augtoproj}.
The proof is now concluded. 
\end{proof}

\subsection{Proof of Proposition~\ref{prop:discretoeng}} \label{sec:discretoeng}

\begin{proof}
Let $r>0$ be sufficiently small. We define two non-local energy forms for $f\in L^2(\M)$, 
\begin{align*}
    {\bf E}_{\M; r}(f) &:= \frac{\alpha}{r^{m+2}} \int_{\M} \int_{\M} \mathbbm{1}_{\{d(x,y)\leq r\}} (f(x)-f(y))^2 p_n(x) p_n(y) d\mathcal{V}(x) d\mathcal{V}(y), \\
    {\bf E}^*_{\M;r}(f) &:= \frac{\alpha}{r^{m+2}} \int_{\M} \int_{\M} \mathbbm{1}_{\{d(x,y)\leq r\}} (f(x)-f(y))^2 q(x)q(y) d\mathcal{V}(x) d\mathcal{V}(y).
\end{align*}
Moreover, we temporarily assume the validity of the following lemma, whose proof is given shortly after.

\begin{lemma} \label{lem:discretetolocal}
Let $\eps>0$ and $\iota_1, \iota_2>0$ satisfy Assumption~\ref{assum:thetadelta}. 
Let the discretization map $P$, the extension map $P^*$ be given in \eqref{eqdef:P}, \eqref{eqdef:P*}, respectively.
Then
\begin{alignat}{2}
    \label{discretetononlocal} E_{\mathcal{Q}_{\XX}}(Pf) &\leq \Big(1 + \frac{C\iota_2}{\eps}\Big){\bf E}_{\M;\eps+2\iota_2}(f), \quad &&\forall f\in L^2(\M), \\
    \label{nonlocaltodiscrete} {\bf E}_{\M;\eps-2\iota_2}(P^*f) &\leq \Big(1 + \frac{C\iota_2}{\eps}\Big)E_{\mathcal{Q}_{\XX}}(f), \quad &&\forall f\in  L^2_{\breve{\vartheta}_n}(\mathcal{Q}_{\XX}). 
\end{alignat}
Here, the constants $C>0$ depend on $\M$, $\rho$. 
\end{lemma}

The remainder of the proof now follows closely the approach in the proof of \citep[Proposition~4.1]{calder2022improved}.
For instance, given Lemma~\ref{lem:discretetolocal}, we can apply Propositions~\ref{prop:qtrunc},~\ref{prop:inftydistance} to obtain
\begin{alignat}{2}
    \label{discretetolocal} E_{\mathcal{Q}_{\XX}}(Pf) &\leq \Big(1 + \frac{C\iota_2}{\eps}\Big)\Big(1 + C(\iota_1+\iota_2)\Big)\Big(1 + C_{\eta}\eps_{\rm n}^{2/3} \Big) {\bf E}^*_{\M;\eps+2\iota_2}(f), \quad &&\forall f\in L^2(\M), \\
    \label{localtodiscrete} {\bf E}^*_{\M;\eps-2\iota_2}(P^*f) &\leq \Big(1 + \frac{C\iota_2}{\eps}\Big)\Big(1 + C(\iota_1+\iota_2)\Big)\Big(1 + C_{\eta}\eps_{\rm n}^{2/3} \Big) E_{\mathcal{Q}_{\XX}}(f), \quad &&\forall f\in  L^2_{\breve{\vartheta}_n}(\mathcal{Q}_{\XX}),
\end{alignat}
where, both \eqref{discretetolocal}, \eqref{localtodiscrete} hold with probability at least $1-n\exp(-cn\iota_1^2\iota_2^m)$.
Then, applying the arguments from \citep[Lemmas~5, 9]{garcia2020error}, we deduce, respectively for $f\in L^2(\M)$ and $r>0$ sufficiently small,
\begin{align}
    \label{localtoeng} {\bf E}^*_{\M;r}(f) &\leq \alpha\beta\Big(1 + Cr+Cr^2\Big) E_{\M}(f), \\
    \label{engtolocal} \alpha\beta E_{\M}(f)(\Lambda_r f) &\leq \Big(1 + Cr + Cr^2\Big)\Big(1 + Cr^2\Big) {\bf E}^*_{\M;r}(f).
\end{align}
With minor simplifications, combining \eqref{discretetolocal}, \eqref{localtoeng} for $r=\eps+2\iota_2$, gives us \eqref{discretetoeng}, while combining \eqref{localtodiscrete}, \eqref{engtolocal} for $r=\eps-2\iota_2$, gives us \eqref{engtodiscrete}. 
\end{proof}

It remains for us to prove Lemma~\ref{lem:discretetolocal}.

\begin{proof}[Proof of Lemma~\ref{lem:discretetolocal}]
We recall from Lemma~\ref{lem:necessityTay} and its proof, that if $d(x,y)\leq r$, for $r>0$ sufficiently small, then 
\begin{equation} \label{recallvarsigma}
    \frac{\varsigma(x,y)}{\eps_{\rm w}^m} = \alpha\rho(x)^{1/2}\rho(y)^{1/2} + \mathcal{O}(r^{1/2}).
\end{equation}
Let $\eps^*:= \eps(1+2\eps_{\rm w})$, $\eps_* := \eps(1-2\eps_{\rm w})$, and $\epsilon^* := \eps^* + \frac{8}{R^2}(\eps^*)^3$; all are small, positive quantities comparable to $\eps$.
We start with \eqref{discretetononlocal}. The proof mirrors the approach outlined in \citep[Lemma~13]{garcia2020error}, albeit with a few modifications due to the adaptive definitions of discretization and extension maps, as well as the fluctuation of edge lengths in the graph $G_{\mathcal{Q}_{\XX}}$, as pointed out earlier.
If $f\in L^2(\M)$, then by definition $Pf\in L^2_{\breve{\vartheta}_n}(\mathcal{Q}_{\XX})$.
Thus from \eqref{Qdistanceexp} and definition \eqref{eqdef:EQ},
\begin{equation*}
    E_{\mathcal{Q}_{\XX}}(Pf) \leq \frac{1}{n^2\eps_{\rm w}^m\eps^{m+2}} \sum_{i,j=1}^n \varsigma(Q_{\bar{x}_i}, Q_{\bar{x}_j}) \mathbbm{1}_{\{0<\|Q_{\bar{x}_i}-Q_{\bar{x}_j}\|\leq \eps^*\}} (Pf(Q_{\bar{x}_i})-Pf(Q_{\bar{x}_j}))^2,
\end{equation*}
which, using \eqref{eq:distancecompare}, \eqref{recallvarsigma} with $r=\eps^*$, simplifies to
\begin{equation} \label{Etransform1}
    E_{\mathcal{Q}_{\XX}}(Pf) \leq \bar{E}_1(Pf) + C(\epsilon^*)^{1/2} \bar{E}_2(Pf),
\end{equation}
where
\begin{equation*} 
    \begin{split}
        \bar{E}_1(Pf) &:= \frac{\alpha}{n^2\eps^{m+2}} \sum_{i,j=1}^n \rho(Q_{\bar{x}_i})^{1/2} \rho(Q_{\bar{x}_j})^{1/2} \mathbbm{1}_{\{0<\|Q_{\bar{x}_i}-Q_{\bar{x}_j}\|\leq \eps^*\}} (Pf(Q_{\bar{x}_i})-Pf(Q_{\bar{x}_j}))^2 \\
        \bar{E}_2(Pf) &:= \frac{1}{n^2\eps^{m+2}} \sum_{i,j=1}^n \mathbbm{1}_{\{0<\|Q_{\bar{x}_i}-Q_{\bar{x}_j}\|\leq \eps^*\}} (Pf(Q_{\bar{x}_i})-Pf(Q_{\bar{x}_j}))^2.
    \end{split}
\end{equation*}
Considering $\bar{E}_1(Pf)$, we directly obtain from the definition \eqref{eqdef:P} of the discretization map $P$ that 
\begin{align*} 
    \bar{E}_1(Pf) 
    &\leq \frac{\alpha}{n\eps^{m+2}} \sum_{j=1}^n \int_{U_j} \rho(Q_{\bar{x}_j})^{1/2} \mathbbm{1}_{\{d(x,Q_{\bar{x}_j})\leq \epsilon^*+2\iota_2\}} (f(x)-Pf(Q_{\bar{x}_j}))^2 p_n(x) d\mathcal{V}(x) \\
    &= \frac{\alpha}{\eps^{m+2}} \int_{\M}\int_{\M} \mathbbm{1}_{\{d(x,y)\leq \epsilon^*+2\iota_2\}} (f(x)-f(y))^2 p_n(x) p_n(y) d\mathcal{V}(x) d\mathcal{V}(y).
\end{align*}
Then from this point on, we proceed identically to the proof of \citep[Lemma~13]{garcia2020error} (and applying the technical \citep[Lemma~5]{garcia2020error} to exchange $\epsilon^*+2\iota_2\approx \eps+2\iota_2$) to get
\begin{equation} \label{barE1consideration}
    \bar{E}_1(Pf)
    \leq \Big(1 + \frac{C\iota_2}{\eps}\Big){\bf E}^*_{\M;\eps+2\iota_2}(f).
\end{equation}
Regarding $\bar{E}_2(Pf)$, we observe by \eqref{boundeddensity} that
\begin{multline*} 
    \bar{E}_2(Pf) \\
    \leq \frac{C}{n^2\eps^{m+2}} \sum_{i,j=1}^n \rho(Q_{\bar{x}_i})^{1/2} \rho(Q_{\bar{x}_j})^{1/2} \mathbbm{1}_{\{0<\|Q_{\bar{x}_i}-Q_{\bar{x}_j}\|\leq \eps^*\}} (Pf(Q_{\bar{x}_i})-Pf(Q_{\bar{x}_j}))^2 
    = C \bar{E}_1(Pf);
\end{multline*}
thus from \eqref{barE1consideration}
\begin{equation} \label{barE2consideration}
    (\epsilon^*)^{1/2}\bar{E}_2(Pf) \leq C(\epsilon^*)^{1/2}\Big(1 + \frac{C\iota_2}{\eps}\Big) {\bf E}^*_{\M;\eps+2\iota_2}(f).
\end{equation}
Since $C\eps^{3/2}\leq\iota_2$ implies that $C(\epsilon^*)^{1/2}\leq \iota_2/\eps$, we combine \eqref{Etransform1}, \eqref{barE1consideration}, \eqref{barE2consideration} to conclude \eqref{discretetononlocal}.

For \eqref{nonlocaltodiscrete}, note that if $f\in  L^2_{\breve{\vartheta}_n}(\mathcal{Q}_{\XX})$ then $P^*f\in L^2(\M)$.
We derive from \eqref{eq:distancecompare}, \eqref{Qdistanceshr}, \eqref{recallvarsigma} with $r=\eps_*$, and definition \eqref{eqdef:EQ} that
\begin{align*}
    &E_{\mathcal{Q}_{\XX}}(f) 
    \geq \frac{1}{n^2\eps_{\rm w}^m\eps^{m+2}} \sum_{Q_{\bar{x}_i}, Q_{\bar{x}_j}\in\mathcal{Q}_{\XX}} \varsigma(Q_{\bar{x}_i}, Q_{\bar{x}_j}) \mathbbm{1}_{\{0<\|Q_{\bar{x}_i}-Q_{\bar{x}_j}\|\leq \eps_{*}\}} (f(Q_{\bar{x}_i})-f(Q_{\bar{x}_j}))^2 \\
    &\geq \frac{1}{n^2\eps_{\rm w}^m\eps^{m+2}} \sum_{Q_{\bar{x}_i}, Q_{\bar{x}_j}\in\mathcal{Q}_{\XX}}
    (\alpha\rho(Q_{\bar{x}_i})^{1/2}\rho(Q_{\bar{x}_j})^{1/2} - C\eps_*^{1/2}) \mathbbm{1}_{\{0<\|Q_{\bar{x}_i}-Q_{\bar{x}_j}\|\leq \eps_*\}} (f(Q_{\bar{x}_i})-f(Q_{\bar{x}_j}))^2.
\end{align*}
The proof closely parallels that of \citep[Lemma~14]{garcia2020error}; thus adopting a similar approach to the one used earlier, we obtain
\begin{equation*}
    E_{\mathcal{Q}_{\XX}}(f) \geq \Big(1 - \frac{C\iota_2}{\eps}\Big) {\bf E}^*_{\M;\eps-2\iota_2}(P^*f),
\end{equation*}
which is \eqref{nonlocaltodiscrete}.
\end{proof}

\subsection{Proof of Proposition~\ref{prop:eignveccompare}} \label{sec:eignveccompare}

\begin{proof}
The proof is nearly word-for-word identical to that of \citep[Proposition~4.2]{calder2022improved}, with modifications similar to those encountered in the proof of Lemma~\ref{lem:discretetolocal}.
One modification involves incorporating the adaptive definitions \eqref{eqdef:P}, \eqref{eqdef:P*}, which yields
\begin{equation*}
    \|P^*Pf\|_{L^2_{\upsilon_n}(\M)}^2 = \|(Pf)\rho^{1/2}\|_{L^2_{\vartheta_n}(\mathcal{Q}_{\XX})}^2,
\end{equation*}
with the extra density above being the distinguishing factor from the corresponding quantities in \citep[Proposition~4.2]{calder2022improved}.
Another modification appears in the application of \citep[Lemma~8]{garcia2020error} near the end.
Specifically, $\eps$ needs to be replaced by $\eps(1+2\eps_{\rm w})$, and this small difference is again accounted for by the technical \citep[Lemma~5]{garcia2020error}. 
After these modifications, every other step in \citep[Proposition~4.2]{calder2022improved} remains unchanged.
Lastly, we mention that the probability $1-n\exp(-cn\iota_1^2\iota_2^m)$ aligns with the probability of the existence of $\upsilon_n$ in Proposition~\ref{prop:inftydistance}. 
Further details are omitted. 
\end{proof}

\subsection{Proof of Proposition~\ref{prop:augtodiscret}} \label{sec:augtodiscret}

\begin{proof}
Since \eqref{eignveccompare31}, \eqref{eignveccompare32} follow directly from the definition of the matching map ${\rm M}$, and \eqref{discretetoaug} is similar but less complex than \eqref{augtodiscrete}, we will prove only \eqref{augtodiscrete} and briefly address \eqref{discretetoaug} afterward.
Let $f\in L^2_{\vartheta_n}(\XX)$, and let $\bar{f}:= f\circ {\rm M}^{-1}\in L^2_{\breve{\vartheta}_n}(\mathcal{Q}_{\XX})$.  
Then
\begin{equation} \label{EXexpand}
    E_{\XX}(f) = 2\langle \mathcal{L}_{\rm aug} f, f\rangle_{ L^2_{\vartheta_n}(\XX)} = 2\langle \mathcal{L}_{\rm aug} f - (\mathcal{L}_{\rm proj} \bar{f})\circ {\rm M}, f\rangle_{ L^2_{\vartheta_n}(\XX)} + 2\langle (\mathcal{L}_{\rm proj} \bar{f})\circ {\rm M}, f\rangle_{ L^2_{\vartheta_n}(\XX)}.
\end{equation}
Observe, since $f(\bar{x}) = \bar{f}(Q_{\bar{x}})$, we have
\begin{align} \label{EXexpandterm1}
    \nonumber \langle (\mathcal{L}_{\rm proj} \bar{f})\circ {\rm M}, f\rangle_{ L^2_{\vartheta_n}(\XX)} 
    &= \frac{1}{n}\sum_{i=1}^n ((\mathcal{L}_{\rm proj} \bar{f})\circ {\rm M})(\bar{x}_i) f(\bar{x}_i) \\
    &= \frac{1}{n}\sum_{i=1}^n (\mathcal{L}_{\rm proj} \bar{f})(Q_{\bar{x}_i}) \bar{f}(Q_{\bar{x}_i})
    = \frac{1}{2} E_{\mathcal{Q}_{\XX}} (\bar{f}),
\end{align}
and moreover,
\begin{multline} \label{EXexpandterm2}
    \langle \mathcal{L}_{\rm aug} f - (\mathcal{L}_{\rm proj} \bar{f})\circ {\rm M}, f\rangle_{ L^2_{\vartheta_n}(\XX)} \\
    = \frac{1}{n^2\eps_{\rm w}^m\eps^{m+2}}\sum_{i=1}^n \sum_{j=1}^n (\omega(\bar{x}_i,\bar{x}_j) -\varsigma(Q_{\bar{x}_i},Q_{\bar{x}_j})) \mathbbm{1}_{\{0<\|\bar{x}_i-\bar{x}_j\|\leq\eps\}} (\bar{f}(Q_{\bar{x}_i})-\bar{f}(Q_{\bar{x}_j})) \bar{f}(Q_{\bar{x}_i}).
\end{multline}
Now, recalling the steps in the proof of Proposition~\ref{prop:augtoproj}, particularly \eqref{compareweights1}, \eqref{compareweights2}, we obtain
\begin{equation} \label{recallcompareweights}
    -C\eps^2\varsigma(Q_{\bar{x}_i},Q_{\bar{x}_j}) \leq \omega(\bar{x}_i,\bar{x}_j) -\varsigma(Q_{\bar{x}_i},Q_{\bar{x}_j}) \leq C\eps^2\varsigma(Q_{\bar{x}_i},Q_{\bar{x}_j}) + \exp\Big(-\frac{c\eps^{4/3}}{\eps^{2\tau/3}}\Big).
\end{equation}
By applying Lemma~\ref{lem:varcalder1} (with $\delta=\eps^2$), combined with a union bound, and adapting the argument leading to \eqref{inch1}, we conclude that
\begin{align} \label{EXexpandterm2part1}
   \nonumber \frac{1}{n^2\eps_{\rm w}^m\eps^{m+2}}\sum_{i=1}^n \sum_{j=1}^n \mathbbm{1}_{\{0<\|\bar{x}_i-\bar{x}_j\|\leq\eps\}} (\bar{f}(Q_{\bar{x}_i}))^2
   &= \frac{1}{n^2\eps_{\rm w}^m\eps^{m+2}} \sum_{i=1}^n (\bar{f}(Q_{\bar{x}_i}))^2 \sum_{Q_{\bar{x}_j}: 0<\|\bar{x}_i-\bar{x}_j\|\leq\eps} 1 \\
   &\leq \frac{C_{\eta}(1+\eps^2)\|f\|_{ L^2_{\vartheta_n}(\XX)}^2}{\eps_{\rm w}^m\eps^2}
\end{align}
holds with probability at least $1-2n\exp(-c_{\eta}n\eps^{m+4})$. 
Similarly, 
\begin{multline} \label{EXexpandterm2part2}
    \frac{1}{n^2\eps_{\rm w}^m\eps^{m+2}}\sum_{i=1}^n \sum_{j=1}^n \mathbbm{1}_{\{0<\|\bar{x}_i-\bar{x}_j\|\leq\eps\}} |\bar{f}(Q_{\bar{x}_i})| |\bar{f}(Q_{\bar{x}_j})| \\
    \leq \frac{1}{2n^2\eps_{\rm w}^m\eps^{m+2}}\sum_{i=1}^n \sum_{j=1}^n \mathbbm{1}_{\{0<\|\bar{x}_i-\bar{x}_j\|\leq\eps\}} (|\bar{f}(Q_{\bar{x}_i})|^2+ |\bar{f}(Q_{\bar{x}_j})|^2) \leq \frac{C_{\eta}(1+\eps^2)\|f\|_{ L^2_{\vartheta_n}(\XX)}^2}{\eps_{\rm w}^m\eps^2}
\end{multline}
holds with the same probability.
We now substitute the second inequality in \eqref{recallcompareweights} into \eqref{EXexpandterm2}, integrating \eqref{EXexpandterm2part1}, \eqref{EXexpandterm2part2}, as well as invoking \eqref{thecondition3}, to derive
\begin{align} \label{EXexpandterm2final}
    \nonumber \langle \mathcal{L}_{\rm aug} f - (\mathcal{L}_{\rm proj} \bar{f})\circ {\rm M}, f\rangle_{ L^2_{\vartheta_n}(\XX)} 
    &\leq C\eps^2 E_{\mathcal{Q}_{\XX}}(\bar{f}) + C_{\eta} \exp\Big(-\frac{c\eps^{4/3}}{\eps^{2\tau/3}}\Big) \frac{\|f\|_{ L^2_{\vartheta_n}(\XX)}^2}{\eps_{\rm w}^m\eps^2} \\
    &\leq C\eps^2 E_{\mathcal{Q}_{\XX}}(\bar{f}) + C_{\eta} \eps \|f\|_{ L^2_{\vartheta_n}(\XX)}^2.
\end{align}
Then, combining \eqref{EXexpandterm2final} with \eqref{EXexpand}, \eqref{EXexpandterm1} yields \eqref{augtodiscrete}.
It remains to see that the proof of \eqref{discretetoaug} follows in the same manner, but utilizes the first inequality in \eqref{recallcompareweights}. 
\end{proof}

\section{Proof of Theorem~\ref{thm:contrastive}}  \label{sec:contrastiveproof}

\begin{proof}
Let $l=1,\dots,k$.
Let $\hat{f}^{\rm aug}_l$ be the $j$th normalized eigenvector (with respect to $L^2_{\vartheta_n}(\XX)$) of $\mathcal{L}_{\rm aug}$, and, by Theorem~\ref{thm:spectralconvergence}, let $f_l$ be the $j$th normalized {eigenfunction} (with respect to $L^2(\M)$) of $\Delta_{\rm aug}$ that satisfies \eqref{L2closeness}.
We have
\begin{equation} \label{contrastive1}
    \frac{1}{\sqrt{n}} \|\hat{f}^{\rm aug}_l(\cdot) - f_l(Q_{\cdot})\|_{\ell^{\infty}(\XX)} 
    \leq \|\hat{f}^{\rm aug}_l(\cdot) - f_l(Q_{\cdot})\|_{L^2_{\vartheta_n}(\XX)} \leq C_{\eta,l}\eps,
\end{equation}
and
\begin{equation} \label{contrastive1a}
    \frac{1}{\sqrt{n}} \|\hat{f}^{\rm aug}_l(\cdot) - f_l(Q_{\cdot})\| 
    \leq \|\hat{f}^{\rm aug}_l(\cdot) - f_l(Q_{\cdot})\|_{L^2_{\vartheta_n}(\XX)} \leq C_{\eta,l}\eps,
\end{equation}
where $\|\cdot\|$ in \eqref{contrastive1a} denotes the Euclidean vector norm. 
Based on the elliptic theory of PDEs \citep[Section~6.3]{gilbarg1977elliptic} and Proposition \ref{prop:qtrunc}, $f_l$ is a Lipschitz function. 
Then as a consequence of \citep[Theorem 3.1]{chen2022nonparametric}, there exists a ReLU NN $(f_l)_{\theta}: \mathbb{R}^d\to\mathbb{R}$ of the form \eqref{eq-def:ReLU}, where $\theta\in\Theta$ given in \eqref{bigtheta}, such that for every $\tilde{\delta}\in (0,1)$
\begin{equation} \label{contrastive2}
    \frac{1}{\sqrt{n}} \|f_l - (f_l)_{\theta}\|_{L^{\infty}(\M)} \leq \|f_l - (f_l)_{\theta}\|_{L^{\infty}(\M)} \leq \tilde{\delta},
\end{equation}
and moreover
\begin{equation} \label{contrastive2a}
    \frac{1}{\sqrt{n}} \|f_l(Q_{\cdot}) - (f_l)_{\theta}(Q_{\cdot})\| \leq \|f_l - (f_l)_{\theta}\|_{L^{\infty}(\M)} \leq \tilde{\delta}.
\end{equation}
Thus, we obtain, from \eqref{contrastive1}, \eqref{contrastive2},
\begin{equation} \label{contrastive3}
    \frac{1}{\sqrt{n}} \|\hat{f}^{\rm aug}_l(\cdot) - (f_l)_{\theta}(Q_{\cdot})\|_{\ell^{\infty}(\XX)} \leq C_{\eta,l}\eps + \tilde{\delta},
\end{equation}
and from \eqref{contrastive1a}, \eqref{contrastive2a},
\begin{equation} \label{contrastive3a}
    \frac{1}{\sqrt{n}} \|\hat{f}^{\rm aug}_l(\cdot) - (f_l)_{\theta}(Q_{\cdot})\| \leq C_{\eta,l}\eps + \tilde{\delta},
\end{equation}
Further, by Eckart-Young-Minsky theorem, a solution to \eqref{eq:optimization} takes the form
\begin{equation*}
    Y^* = \frac{1}{\sqrt{n}} \begin{bmatrix} \sqrt{\hat{\lambda}^{\rm aug}_l} \hat{f}^{\rm aug}_l(\bar{x}_i) \end{bmatrix}_{i=1,\dots,n; l=1,\dots,k} \in\mathbb{R}^{n\times k},
\end{equation*}
assuming the rotation matrix is the $k\times k$ identity matrix.
Stack the ReLU NNs $\sqrt{\hat{\lambda}^{\rm aug}_l} (f_l)_{\theta}$, $l=1,\dots,k$, to obtain a ReLU NN $f_{\theta}:\R^d\to\R^l$, and let $Y_{\theta} = \frac{1}{\sqrt{n}}(f_{\theta}(\bar{x}_1),\dots, f_{\theta}(\bar{x}_n))^{\top}\in\mathbb{R}^{n\times l}$. We conclude respectively from \eqref{contrastive3}, \eqref{contrastive3a},
\begin{equation} \label{contrastive4}
    \|Y^* - Y_{\theta}\|_{\infty} \leq C_{\eta,k}(\eps + \tilde{\delta}) \quad\text{ and }\quad
    \|Y^* - Y_{\theta}\| \leq C'_{\eta,k}(\eps + \tilde{\delta}),
\end{equation}
as desired. 
Here, in \eqref{contrastive4}, both $C_{\eta,k}$ and $C'_{\eta,k}$ involve a factor of $\max_{l=1,\dots,k} \sqrt{\hat{\lambda}^{\rm aug}_l}$, with $C'_{\eta,k}$ further including a factor of $\sqrt{k}$.
With $\mathsf{L}\big(\sqrt{\hat{\lambda}^{\rm aug}_l}(f_l)_{\theta}\big)$ denoting the network depth of $\sqrt{\hat{\lambda}^{\rm aug}_l}(f_l)_{\theta}$, $\mathsf{p}\big(\sqrt{\hat{\lambda}^{\rm aug}_l}(f_l)_{\theta}\big)$ its width, and $\mathsf{m}\big(\sqrt{\hat{\lambda}^{\rm aug}_l}(f_l)_{\theta}\big)$, $\mathsf{M}\big(\sqrt{\hat{\lambda}^{\rm aug}_l}(f_l)_{\theta}\big)$ the complexity parameters from \eqref{maxcondition}, \eqref{nuclearcondition} respectively, we obtain in addition from \citep[Theorem 3.1]{chen2022nonparametric},
\begin{equation} \label{sizecontrastive}
    \mathsf{L}\Big(\sqrt{\hat{\lambda}^{\rm aug}_l}(f_l)_{\theta}\Big) \leq C\Big(\log\Big(\frac{1}{\tilde{\delta}}\Big) + \log d\Big) \quad\text{ and }\quad
    \mathsf{p}\Big(\sqrt{\hat{\lambda}^{\rm aug}_l}(f_l)_{\theta}\Big) \leq C(\tilde{\delta}^{-m} + d),
\end{equation}
and
\begin{equation} \label{complexcontrastive}
    \mathsf{m}\Big(\sqrt{\hat{\lambda}^{\rm aug}_l}(f_l)_{\theta}\Big) \leq C\max_{l=1,\dots,k} \sqrt{\hat{\lambda}^{\rm aug}_l} \quad\text{ and }\quad \mathsf{M}\Big(\sqrt{\hat{\lambda}^{\rm aug}_l}(f_l)_{\theta}\Big) \leq \Big((\tilde{\delta}^{-m}+d) \log\Big(\frac{1}{\tilde{\delta}}\Big) + d\log d\Big),
\end{equation}
where all the constants $C$ depend only on $\M$.
Therefore, it follows from \eqref{sizecontrastive}, \eqref{complexcontrastive} respectively that
\begin{equation*} 
    \mathsf{L}\Big(\frac{1}{\sqrt{n}} f_{\theta}\Big) \leq C\Big(\log\Big(\frac{1}{\tilde{\delta}}\Big) + \log d\Big) \quad\text{ and }\quad
    \mathsf{p}\Big(\frac{1}{\sqrt{n}} f_{\theta}\Big) \leq Ck(\tilde{\delta}^{-m} + d),
\end{equation*}
and
\begin{equation*} 
    \mathsf{m}\Big(\frac{1}{\sqrt{n}} f_{\theta}\Big) \leq \frac{C\max_{l=1,\dots,k} \sqrt{\hat{\lambda}^{\rm aug}_l}}{\sqrt{n}} \leq C \quad\text{ and }\quad \mathsf{M}\Big(\frac{1}{\sqrt{n}} f_{\theta}\Big) \leq k\Big((\tilde{\delta}^{-m}+d) \log\Big(\frac{1}{\tilde{\delta}}\Big) + d\log d\Big),
\end{equation*}
when $n$ is sufficiently large. 
Finally, we perform a union bound over the probabilities of \eqref{L2closeness} occurring for all $l=1,\dots,k$ to complete the proof. 
\end{proof}


\section*{Acknowledgments and disclosure of funding}

CL gratefully acknowledges the support from the IFDS at UW-Madison and the NSF through the TRIPODS grant 2023239. 
AMN is supported by the Austrian Science Fund (FWF) Project P-37010.

\vskip 0.2in
\bibliography{bib}

\appendix 

\section{Relevant concentration results}

We present the following two known concentration lemmas.

\begin{lemma} \label{lem:calder1} \citep[Lemma 3.1 and its proof]{calder2022improved} (see also \citep[Lemma~1]{calder2019game})
Let $\M$ satisfy the assumptions made in Section~\ref{sec:assumption}.
Let $\psi:\mathbb{R}^d\to\mathbb{R}$ be bounded and Borel measurable with compact support in a bounded open set $\Omega\subset\mathbb{R}^d$ such that $\Omega\cap\M\not=\emptyset$. 
Let $x_1,\dots,x_n$ be a sample of i.i.d points on $\M$ following a distribution on $\M$ with a positive continuous density $p$, such that $p(x)\leq p_{\rm max}<\infty$ for all $x\in\M$.
Let
\begin{equation*} 
    \Psi := \frac{1}{n} \sum_{i=1}^n \psi(x_{i}),
\end{equation*}
and let $a := \int_{\Omega\cap\M} \psi(z)p(z)\dd \mathcal{V}(z)$.
Then for every $0<\delta\leq 1$,
\begin{equation*} 
    \mathbb{P} \big(|\Psi- a| > C p_{\rm max}\|\psi\|_{L^{\infty}(\Omega)}\mathcal{V}(\Omega\cap\M)\, \delta \big)\leq 2\exp\big(-cp_{\rm max} n \mathcal{V}(\Omega\cap\M) \, \delta^2 \big),
\end{equation*}
where $C,c>0$ depend only on $\M$. 
\end{lemma}

When $\Omega = B^d(x,\eps)$ for $x\in\M$, Lemma~\ref{lem:varcalder2} below can be seen as a special case of Lemma~\ref{lem:varcalder1}.

\begin{lemma} \label{lem:calder2} \citep[Lemma 3.1 and its proof]{calder2022improved}
Let $\M$ satisfy the assumptions made in Section~\ref{sec:assumption}.
Let $\psi:\M\to\mathbb{R}$ be bounded and Borel measurable. Let $x_1,\dots,x_n$ be a sample of i.i.d points on $\M$ following a distribution on $\M$ with a positive continuous density $p$, such that $p(x)\leq p_{\rm max}<\infty$ for all $x\in\M$.
Let $\eps>0$. For $x\in\M$, we define $\Psi = \Psi(x)$ to be
\begin{equation*} 
    \Psi := \frac{1}{n} \sum_{i: \|x_{i}-x\|\leq\eps} \psi(x_{i}).
\end{equation*}
Let
\begin{equation*}
    a := \int_{B^d(x,\eps)\cap\M} \psi(z)p(z)\dd \mathcal{V}(z) 
    \quad\text{ and }\quad
    b := \int_{\mathcal{B}(x,\eps)} \psi(z)p(z)\dd \mathcal{V}(z).
\end{equation*}
Then we have
\begin{equation*}
    |a-b|\leq Cp_{\rm max}\|\psi\|_{L^{\infty}(\mathcal{B}(x,2\eps))} \eps^{m+2}.
\end{equation*}
Additionally, for any $\eps^2\leq\delta\leq 1$, 
\begin{equation*} 
    \mathbb{P}(|\Psi- b|\geq Cp_{\rm max}\|\psi\|_{L^{\infty}(\mathcal{B}(x,2\eps))}\delta \eps^{m}) \leq 2\exp(-cp_{\rm max}\delta^2 n\eps^{m}),
\end{equation*}
where $C,c>0$ depend only on $\M$.
\end{lemma}

Both Lemmas~\ref{lem:calder1}, \ref{lem:calder2} apply to i.i.d. random variables. However, in the context of our analysis, we require concentration results for the collection $\mathcal{Q}_{\XX}$ of projection points $Q_{\bar{x}}$. While these points can be viewed as samples from the measure $\nu$ with density $q$ \eqref{eqdef:q}, they are not independent. 
Nevertheless, under Assumption~\ref{assum:partindependent}, a similar partial independence condition holds for $\mathcal{Q}_{\XX}$, in that for every $Q_{\bar{x}}$, there exist at most $C_0\in\mathbb{N}$ other $Q_{\bar{x}'}$ dependent on $Q_{\bar{x}}$.
We present variants of both lemmas to address this scenario, as follows.

\begin{lemma} \label{lem:varcalder1}
Let $\M$ satisfy the assumptions made in Section~\ref{sec:assumption}.
Let $\psi:\mathbb{R}^d\to\mathbb{R}$ be bounded and Borel measurable with compact support in a bounded open set $\Omega\subset\mathbb{R}^d$ such that $\Omega\cap\M\not=\emptyset$. 
Let $x_1,\dots,x_n$ be a sample of i.i.d points on $\M$ following a distribution on $\M$ with a positive continuous density $p$, such that $p(x)\leq p_{\rm max}<\infty$ for all $x\in\M$.
Suppose further that for each point $x_i$ there are at most $C_0$ other points $x_j$ dependent on $x_i$.
Let
\begin{equation*} 
    \Psi := \frac{1}{n} \sum_{i=1}^n \psi(x_{i}),
\end{equation*}
and let $a := \int_{\Omega\cap\M} \psi(z)p(z)\dd \mathcal{V}(z)$.
Then for every $0<\delta\leq 1$,
\begin{equation} \label{varcalder1prob}
    \mathbb{P} \big(|\Psi- a| > C p_{\rm max}\|\psi\|_{L^{\infty}(\Omega)}\mathcal{V}(\Omega\cap\M)\, \delta \big)\leq 2\exp\big(-cp_{\rm max} n \mathcal{V}(\Omega\cap\M) \, \delta^2 \big),
\end{equation}
where $C,c>0$ depend only on $\M$.
\end{lemma}

\begin{lemma} \label{lem:varcalder2}
Let $\M$ satisfy the assumptions made in Section~\ref{sec:assumption}.
Let $\psi:\M\to\mathbb{R}$ be bounded and Borel measurable. Let $x_1,\dots,x_n$ be a sample of points on $\M$ following an identical distribution on $\M$ with a positive continuous density $p$, such that $p(x)\leq p_{\rm max}<\infty$ for all $x\in\M$. 
Suppose further that for each point $x_i$ there are at most $C_0$ other points $x_j$ dependent on $x_i$.
Let $\eps>0$. For $x\in\M$, we define $\Psi = \Psi(x)$ to be
\begin{equation*} 
    \Psi := \frac{1}{n} \sum_{i: \|x_{i}-x\|\leq\eps} \psi(x_{i}).
\end{equation*}
Let
\begin{equation*}
    a := \int_{B^d(x,\eps)\cap\M} \psi(z)p(z)\dd \mathcal{V}(z) 
    \quad\text{ and }\quad
    b := \int_{\mathcal{B}(x,\eps)} \psi(z)p(z)\dd \mathcal{V}(z).
\end{equation*}
Then we have
\begin{equation*}
    |a-b|\leq Cp_{\rm max}\|\psi\|_{L^{\infty}(\mathcal{B}(x,2\eps))} \eps^{m+2}.
\end{equation*}
Additionally, for any $\eps^2\leq\delta\leq 1$, 
\begin{equation} \label{varcalder2prob}
    \mathbb{P}(|\Psi- b|\geq Cp_{\rm max}\|\psi\|_{L^{\infty}(\mathcal{B}(x,2\eps))}\delta \eps^{m}) \leq 2\exp(-cp_{\rm max}\delta^2 n\eps^{m}),
\end{equation}
where $C,c>0$ depend only on $\M$.
\end{lemma}

The constant $c$ in both probabilistic estimates \eqref{varcalder1prob}, \eqref{varcalder2prob} implicitly includes a factor of $C_0^{-1}$, which we assume to be predetermined.
The proofs of Lemma~\ref{lem:varcalder1} and Lemma~\ref{lem:varcalder2} can be easily adapted from those in \citep[Lemma~3.1]{calder2022improved} and \citep[Theorem 2.1]{janson2004large}.
We skip further details.


\section{Relevant geometric construction for Propositions~\ref{prop:qtrunc},~\ref{prop:etaandw}} \label{appx:dg}

Let $r>0$ be sufficiently small. 
Then recall from Section~\ref{sec:assumption} that the logarithm map
\begin{equation*} 
    {\rm Log}_{x}: \M \supset \mathcal{B}(x,r) \to B^m(0,r)\subset \T_{x}\mathcal{M},
\end{equation*}
is a diffeomorphism. 
Let $\mathsf{I}_m: \mathbb{R}^m\to\mathbb{R}^d$, where $\mathsf{I}_m(y^1,\dots,y^m) = (y^1,\dots,y^m,0,\dots,0)$, be a \textit{canonical immersion} map (see \citep[Chapter~3]{guillemin2010differential}).
Let ${\rm A}_x$ denote a composition of a translation by $x$ and an appropriate rotation ${\rm R}_x$, such that ${\rm A}_x$ maps
\begin{equation*}
    \mathbb{R}^m\times\{0\}:=\{(y^1,\dots,y^m,0,\dots,0): y^i\in\mathbb{R}\}
\end{equation*}
onto $\tilde{T}_x\M$. See Figure~\ref{fig:reposition} below for an illustration.
Then ${\rm A}_x\circ \mathsf{I}_m$ maps $T_x\M\cong\mathbb{R}^m$ bijectively and isometrically to $\tilde{T}_x\M\subset\mathbb{R}^d$. 
Expressively, 
\begin{equation} \label{Gx}
    {\rm A}_x\circ \mathsf{I}_m (T_x\M) = \tilde{T}_x\M.
\end{equation}
By letting ${\rm L}_x := {\rm A}_x\circ \mathsf{I}_m\circ {\rm Log}_x$, we obtain for $x'\in\mathcal{B}(x,r)$ a unique image ${\rm L}_x(x')\in\tilde{T}_x\M$, such that
\begin{equation} \label{eq:Lxisom}
    \|{\rm L}_x(x') - x\| = \| {\rm Log}_x(x')\| = d(x,x'). 
\end{equation}

\begin{figure}[ht!]
\caption{A two-dimensional sphere and one of its tangent planes under rotation and translation}
    \begin{subfigure}[b]{0.40\textwidth}
        \centering
        \includegraphics[width=0.7\linewidth]{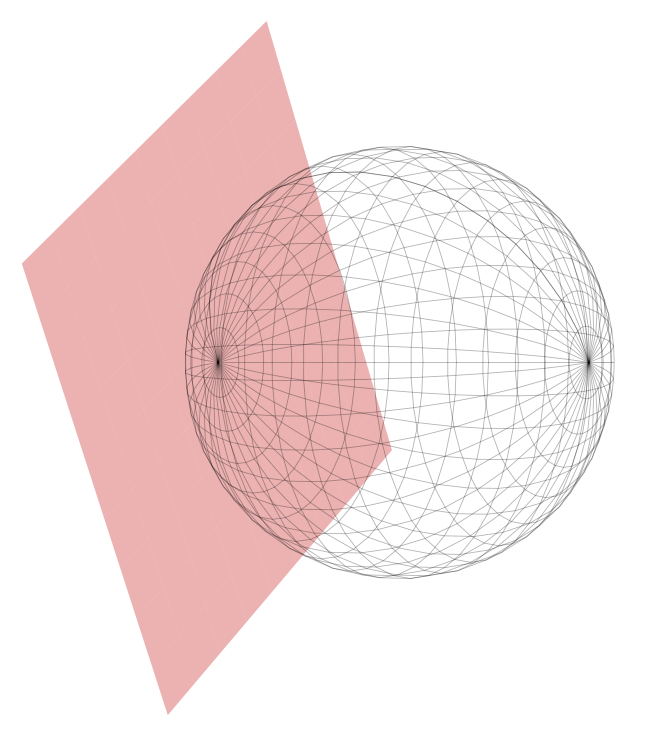}
        \caption{A two-dimensional sphere in $\mathbb{R}^3$ and one of its tangent planes colored in pink}
    \end{subfigure}
    \quad \quad \quad
    \begin{subfigure}[b]{0.47\textwidth}
        \centering
        \includegraphics[width=0.7\linewidth]{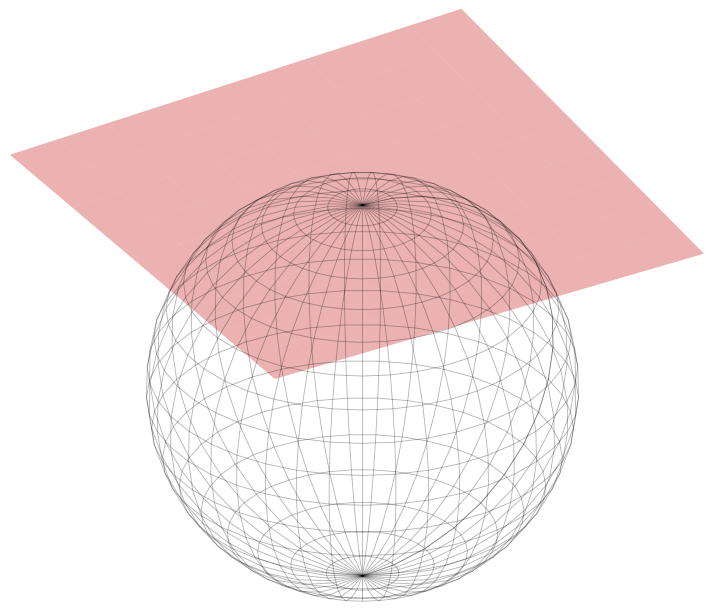}
        \caption{The same sphere now repositioned with the tangent plane as the $xy$-plane}
    \end{subfigure}
\label{fig:reposition}
\end{figure} 

Next, let $s>0$. 
We consider the neighborhood $\mathcal{N}_{s}(\mathcal{B}(x,r))$ of $\mathcal{B}(x,r)$, defined to be
\begin{equation*}
    \mathcal{N}_{s}(\mathcal{B}(x,r)) 
    := \{ x'+h: x'\in\mathcal{B}(x,r) \text{ and } h \in \tilde{N}_{x'}\M\cap \mathcal{N}_{s}(\M) - x'\}.
\end{equation*}
Note, $\tilde{N}_{x'}\M -x'\cong N_{x'}\M\cong \mathbb{R}^{d-m}$. 
Similarly, as before, such identification can be made explicit, for the rotation ${\rm R}_{x'}$ also maps 
\begin{equation*}
    \{0\}\times\mathbb{R}^{d-m} 
    := \{(0,\dots,0,y^1,\dots,y^{d-m}): y^i\in\mathbb{R}\}
\end{equation*}
onto $\tilde{N}_{x'}\M-x'$. 
Let ${\rm J}_{d-m}: \mathbb{R}^{d-m}\to\mathbb{R}^d$ be another immersion map, where 
$${\rm J}_{d-m}(y^1,\dots,y^{d-m}) = (0,\dots,0,y^1,\dots,y^{d-m}).$$
Then, through the bijective isometric map $({\rm R}_{x'}\circ {\rm J}_{d-m})^{-1}$, $h\in \tilde{N}_{x'}\M\cap \mathcal{N}_{s}(\M)-x'$ is uniquely identified with $k=({\rm R}_{x'}\circ {\rm J}_{d-m})^{-1}(h)\in B^{d-m}(0,s)$.
Now using the diffeomorphic map
\begin{equation*} 
    {\rm Exp}_{x}: \T_{x}\mathcal{M} \supset B^m(0,r) \to \mathcal{B}(x,r) \subset \M,
\end{equation*}
we can construct another diffeomorphic map: 
\begin{equation*} 
    \begin{split}
        \overline{{\rm Exp}}_x: B^m(0,r)\times B^{d-m}(0,s) 
        &\to \mathcal{N}_{s}(\mathcal{B}(x,r)) \\
        (v,k) &\mapsto x'+h,
    \end{split}  
\end{equation*}
where $x'={\rm Exp}_x(v)$ and, as noted, $h = {\rm R}_{x'}\circ {\rm J}_{d-m}(k)$.
It is readily checked that the volume form of $\overline{{\rm Exp}}_x$ at $(v,k)$ is $J_x(v)  \dd v\times \dd k$.

\section{Proof of Proposition~\ref{prop:etaandw}} \label{appx:etaandw}

\begin{proof}
We begin by establishing the necessary groundwork.
Let $\tilde{T}_{Q_{\bar{x}}} \M$ and $\tilde{T}_{Q_{\bar{x}'}} \M$ be the tangent planes of $\M$ at $Q_{\bar{x}}$ and at $Q_{\bar{x}'}$, respectively. 
It follows that 
\begin{equation} \label{perp}
    Q_{\bar{x}}-\bar{x} \perp \tilde{T}_{Q_{\bar{x}}} \M \quad\text{ and }\quad Q_{\bar{x}'}-\bar{x}' \perp \tilde{T}_{Q_{\bar{x}'}} \M,
\end{equation}
where $\perp$ denotes the usual orthogonal notion in $\mathbb{R}^{d}$; see Figure \ref{fig:project}. 
From \eqref{eq:distancecompare}, \eqref{epsassume}, we obtain\footnote{Whether these inequalities are presented as strict or non-strict will not affect the subsequent analysis.}
\begin{align*} 
    d(x,Q_{\bar{x}}) &\leq \|x-Q_{\bar{x}}\| + \frac{8}{R^2}\|x-Q_{\bar{x}}\|^3 \leq C\eps_{\rm n}^{2/3},\\
    d(x,Q_{\bar{x}'}) &\leq \|x-Q_{\bar{x}'}\| + \frac{8}{R^2}\|x-Q_{\bar{x}'}\|^3 \leq C\eps_{\rm n}^{2/3}.
\end{align*}
Hence, by applying \eqref{eq:Lxisom} to the locations $Q_{\bar{x}}$, $Q_{\bar{x}'}$, we acquire for $x$, respective unique images ${\rm L}_{Q_{\bar{x}}}(x)$, ${\rm L}_{Q_{\bar{x}'}}(x)$, such that
\begin{multline} \label{QxlogQxdistance}
    \|{\rm L}_{Q_{\bar{x}}}(x) - Q_{\bar{x}}\| = \|{\rm Log}_{Q_{\bar{x}}}(x)\| = d(x,Q_{\bar{x}})\leq C\eps_{\rm n}^{2/3}, \text{ and } \\
    \|{\rm L}_{Q_{\bar{x}'}}(x) - Q_{\bar{x}'}\| = \|{\rm Log}_{Q_{\bar{x}'}}(x)\| = d(x,Q_{\bar{x}'})\leq C\eps_{\rm n}^{2/3}.
\end{multline}
We now present the following two technical lemmas that will be indispensable in our analysis.
Their proofs are provided in Appendices \ref{sec:quadbehaviorproof}, \ref{sec:xQxlogQxproof}, respectively in the order they appear. 

\begin{figure}
	\centering
	\includegraphics[scale=0.6]{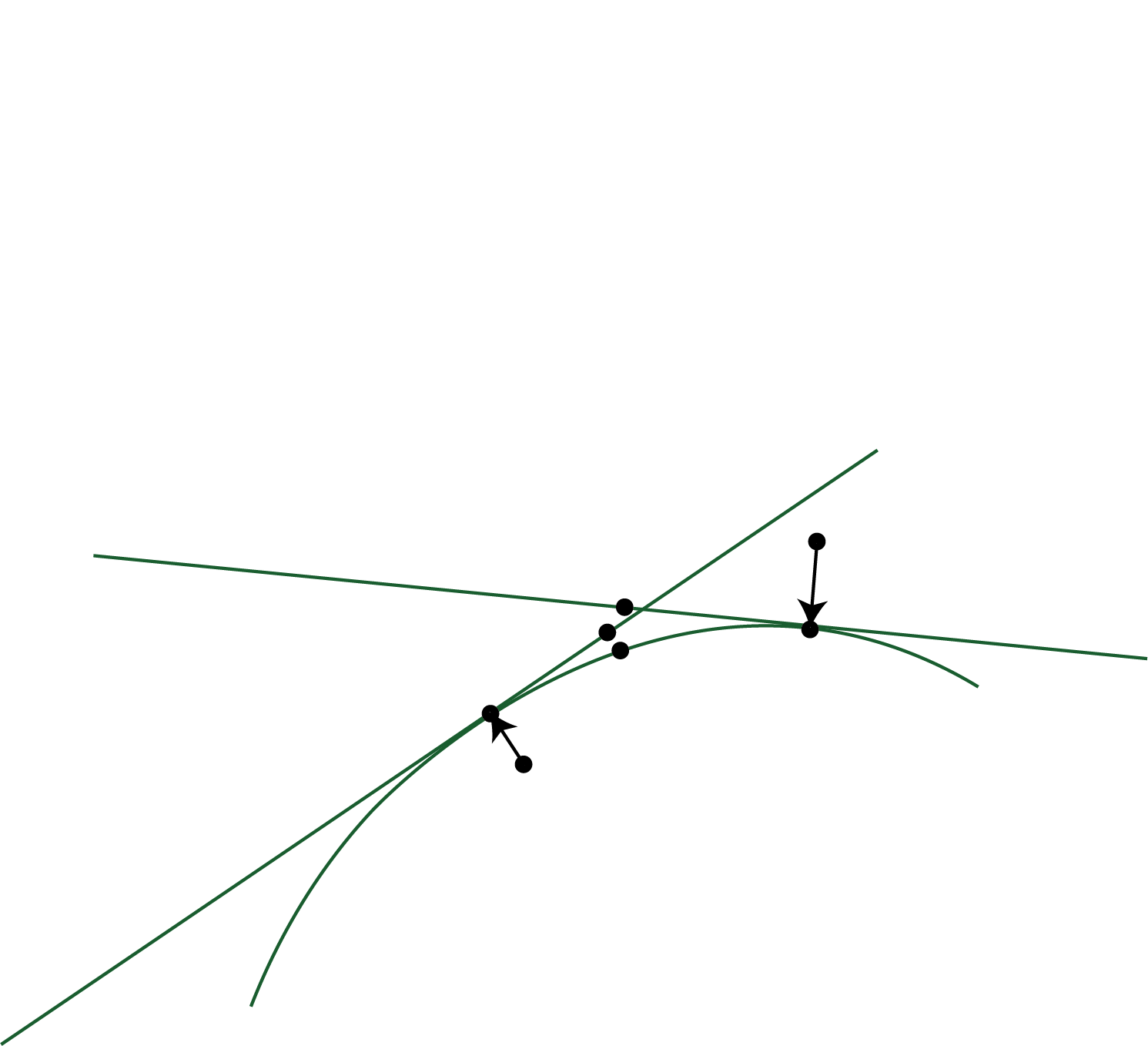}
        \put(-68,65){$Q_{\bar{x}}$}
        \put(-60,95){$\bar{x}$}
        \put(-137,68){$Q_{\bar{x}'}$}
        \put(-117,43){$\bar{x}'$}
        \put(-165,10){$\M$}
        \put(-25,79){$\tilde{T}_{Q_{\Bar{x}}}\M$}
        \put(-228,20){$\tilde{T}_{Q_{\Bar{x}'}}\M$}
        \put(-100,67){$x$}
	\caption{$Q_{\bar{x}}$ and $Q_{\bar{x}'}$ are the closest points to $\bar{x}$ and $\bar{x}'$ on $\M$, respectively. $\tilde{T}_{Q_{\Bar{x}}}\M$ and $\tilde{T}_{Q_{\Bar{x}'}}\M$ are the tangent planes of $\M$ on $Q_{\Bar{x}}$ and $Q_{\Bar{x}'}$, respectively.} 
		\label{fig:project}
\end{figure}

\begin{lemma} \label{lem:quadbehavior} Let $x$, $Q_{\bar{x}}$, $Q_{\bar{x}'}$, ${\rm L}_{Q_{\bar{x}}}(x)$, ${\rm L}_{Q_{\bar{x}'}}(x)$ be as above. Then it holds that
\begin{equation} \label{quadbehavior}
    \|x-{\rm L}_{Q_{\bar{x}}}(x)\|=\mathcal{O}(\|{\rm L}_{Q_{\bar{x}}}(x)-Q_{\bar{x}}\|^2) \quad\text{ and }\quad \|x-{\rm L}_{Q_{\bar{x}'}}(x)\|=\mathcal{O}(\|{\rm L}_{Q_{\bar{x}'}}(x)-Q_{\bar{x}'}\|^2)
\end{equation}
where the majorant constants depend only on $\M$. 
\end{lemma}

\begin{lemma} \label{lem:xQxlogQx} Let $x$, $Q_{\bar{x}}$, $Q_{\bar{x}'}$, ${\rm L}_{Q_{\bar{x}}}(x)$, ${\rm L}_{Q_{\bar{x}'}}(x)$ be as above. 
Then under Assumption~\ref{assum:nearness}, the following hold,
\begin{equation} \label{xQxlogQxupbd}
    \begin{split}
        \|x-\bar{x}\|^2 &\leq \|\bar{x}-Q_{\bar{x}}\|^2 + (1+C\eps_{\rm n}^{2/3})\|Q_{\bar{x}} - {\rm L}_{Q_{\bar{x}}}(x)\|^2,\\
        \|x-\bar{x}'\|^2 &\leq \|\bar{x}'-Q_{\bar{x}'}\|^2 + (1+C\eps_{\rm n}^{2/3})\|Q_{\bar{x}'}-{\rm L}_{Q_{\bar{x}'}}(x)\|^2,
    \end{split}   
\end{equation}
and moreover,
\begin{equation} \label{xQxlogQxlwbd}
    \begin{split}
        \|x-\bar{x}\|^2 &\geq \|\bar{x}-Q_{\bar{x}}\|^2 + (1-C'\eps_{\rm n}^{2/3})\|Q_{\bar{x}}-{\rm L}_{Q_{\bar{x}}}(x)\|^2,\\
        \|x-\bar{x}'\|^2 &\geq \|\bar{x}'-Q_{\bar{x}'}\|^2 + (1 - C'\eps_{\rm n}^{2/3})\|Q_{\bar{x}'}-{\rm L}_{Q_{\bar{x}'}}(x)\|^2,
    \end{split}   
\end{equation}
where $C, C'>0$ dependent on $\M$.
\end{lemma}

Recalling the random variable 
$$X=\mathbbm{1}_{\{\|x-Q_{\bar{x}}\|< \eps_{\rm n}^{2/3}\} \wedge \{\|x-Q_{\bar{x}'}\|< \eps_{\rm n}^{2/3}\}}$$ from \eqref{eqdef:X}, we define another, related variable
\begin{equation} \label{eqdef:Y}
    Y:=1-X=\mathbbm{1}_{\{\|x-Q_{\bar{x}}\|\geq \eps_{\rm n}^{2/3}\} \vee \{\|x-Q_{\bar{x}'}\|\geq \eps_{\rm n}^{2/3}\}}.
\end{equation}
By expanding upon \eqref{def:edgeweight}, we attain
\begin{multline} \label{expansion}
    \omega(\bar{x},\bar{x}')
    = \E_{x\sim \mu} \Big[ \exp\Big(\frac{-\|x-\bar{x}\|^2-\|x-\bar{x}'\|^2}{2\eps_{\rm w}^2}\Big) \Big]\\
    = \E_{x\sim \mu} \Big[ \Big( X + Y\Big) \exp\Big(\frac{-\|x-\bar{x}\|^2-\|x-\bar{x}'\|^2}{2\eps_{\rm w}^2}\Big) \Big].
\end{multline}
We will first prove the lower bound in \eqref{etaandwNr}. 
By virtue of Lemma \ref{lem:xQxlogQx}, we have
\begin{multline} \label{lwbdstep1} 
    \Big( X + Y\Big) \exp\Big(\frac{-\|x-\bar{x}\|^2-\|x-\bar{x}'\|^2}{2\eps_{\rm w}^2}\Big) \\
    \geq X\,  \exp\Big(\frac{-\|\bar{x}-Q_{\bar{x}}\|^2-\|\bar{x}'-Q_{\bar{x}'}\|^2- (1+C\eps_{\rm n}^{2/3})(\|Q_{\bar{x}}-{\rm L}_{Q_{\bar{x}}}(x)\|^2+\|Q_{\bar{x}'}-{\rm L}_{Q_{\bar{x}'}}(x)\|^2)}{2\eps_{\rm w}^2}\Big).
\end{multline}
Further, by Lemma~\ref{lem:quadbehavior} along with \eqref{QxlogQxdistance}, it holds that
\begin{equation} \label{lwbdstep2}
    \begin{split}
        \|Q_{\bar{x}} {-} {\rm L}_{Q_{\bar{x}}}(x)\|^2 &{\leq} \|Q_{\bar{x}} {-} x\|^2 + 2\|Q_{\bar{x}} {-} x\|\|x {-} {\rm L}_{Q_{\bar{x}}}(x)\| {+} \|x {-} {\rm L}_{Q_{\bar{x}}}(x)\|^2 \leq \|Q_{\bar{x}} {-} x\|^2 + C\eps_{\rm n}^2 \\
        \|Q_{\bar{x}'} {-} {\rm L}_{Q_{\bar{x}'}}(x)\|^2 & {\leq} \|Q_{\bar{x}'} {-} x\|^2 + 2\|Q_{\bar{x}'} {-} x\|\|x {-} {\rm L}_{Q_{\bar{x}'}}(x)\| {+} \|x {-} {\rm L}_{Q_{\bar{x}'}}(x)\|^2 {\leq} \|Q_{\bar{x}'} {-} x\|^2 {+} C\eps_{\rm n}^2,
    \end{split}
\end{equation}
and that
\begin{equation} \label{extraterms}
    \begin{split}
        C\eps_{\rm n}^{2/3} \|Q_{\bar{x}}-{\rm L}_{Q_{\bar{x}}}(x)\|^2 &\leq  C\eps_{\rm n}^2\\
        C\eps_{\rm n}^{2/3} \|Q_{\bar{x}'}-{\rm L}_{Q_{\bar{x}'}}(x)\|^2 &\leq  C\eps_{\rm n}^2.
    \end{split}
\end{equation}
Hence, combining \eqref{lwbdstep1}, \eqref{lwbdstep2}, \eqref{extraterms} gives us the desired lower bound
\begin{equation*} 
    \omega(\bar{x},\bar{x}') \geq \exp\Big(-\frac{C\eps_{\rm n}^2}{\eps_{\rm w}^2}\Big)\xi(\bar{x},\bar{x}') = \exp(-C\eps^2)\xi(\bar{x},\bar{x}').
\end{equation*}
Now, for the upper bound in \eqref{etaandwNr}, we first consider the case when
\begin{equation*} 
    \|x-Q_{\bar{x}}\|\geq \eps_{\rm n}^{2/3} \quad\text{ or }\quad \|x-Q_{\bar{x}'}\|\geq \eps_{\rm n}^{2/3},
\end{equation*}
i.e. when $Y=1$ in \eqref{eqdef:Y}.
Suppose the first inequality holds.
Then since $\eps_{\rm n}^{2/3}\gg\eps_{\rm n}$, $\|x-Q_{\bar{x}}\| - \|\bar{x}-Q_{\bar{x}}\| \geq c\eps_{\rm n}^{2/3}> 0$, and consequently,
\begin{equation*}
    \|\bar{x} - x\|^2 \geq (\|x-Q_{\bar{x}}\| - \|\bar{x}-Q_{\bar{x}}\|)^2 \geq c\eps_{\rm n}^{4/3}.
\end{equation*}
Therefore,
\begin{equation} \label{eq:case2Y}
    \E_{x\sim \mu} \Big[ Y \exp\Big(\frac{-\|x-\bar{x}\|^2-\|x-\bar{x}'\|^2}{2\eps_{\rm w}^2}\Big) \Big]\leq \exp\Big(-\frac{c\eps_{\rm n}^{4/3}}{\eps_{\rm w}^2}\Big) = \exp\Big(-\frac{c\eps^{4/3}}{\eps^{2\tau/3}}\Big).
\end{equation}
Next, suppose
\begin{equation*} 
    \|x-Q_{\bar{x}}\|< \eps_{\rm n}^{2/3} \quad\text{ and }\quad \|x-Q_{\bar{x}'}\|< \eps_{\rm n}^{2/3},
\end{equation*}
i.e. when $X=1$ in \eqref{eqdef:X}.
Then, due to
\begin{align*}
    \|Q_{\bar{x}}-{\rm L}_{Q_{\bar{x}}}(x)\|^2 &\geq \|Q_{\bar{x}}-x\|^2 - 2\|Q_{\bar{x}}-x\|\|x-{\rm L}_{Q_{\bar{x}}}(x)\| + \|x-{\rm L}_{Q_{\bar{x}}}(x)\|^2 \\
    \|Q_{\bar{x}'}-{\rm L}_{Q_{\bar{x}'}}(x)\|^2 &\geq \|Q_{\bar{x}'}-x\|^2 - 2\|Q_{\bar{x}'}-x\|\|x-{\rm L}_{Q_{\bar{x}'}}(x)\| + \|x-{\rm L}_{Q_{\bar{x}'}}(x)\|^2,
\end{align*}
we get, from \eqref{QxlogQxdistance} and Lemma~\ref{lem:quadbehavior},
\begin{equation} \label{eq:case2Xstep2}
    \begin{split}
        \|Q_{\bar{x}}-{\rm L}_{Q_{\bar{x}}}(x)\|^2 &\geq \|Q_{\bar{x}}-x\|^2 - C\eps_{\rm n}^2\\
        \|Q_{\bar{x}'}-{\rm L}_{Q_{\bar{x}'}}(x)\|^2 &\geq \|Q_{\bar{x}'}-x\|^2 - C\eps_{\rm n}^2.
    \end{split}
\end{equation}
Further, by virtue of Lemma \ref{lem:xQxlogQx},
\begin{align*}
    -\|\bar{x}-x\|^2 &\leq -\|\bar{x}-Q_{\bar{x}}\|^2- (1-C\eps_{\rm n}^{2/3})\|Q_{\bar{x}}-{\rm L}_{Q_{\bar{x}}}(x)\|^2 \\
    -\|\bar{x}'-x\|^2 &\leq -\|\bar{x}'-Q_{\bar{x}'}\|^2 -(1-C\eps_{\rm n}^{2/3})\|Q_{\bar{x}'}-{\rm L}_{Q_{\bar{x}'}}(x)\|^2,
\end{align*}
we get from another application of \eqref{QxlogQxdistance}
\begin{equation} \label{case2Xstep1}
    \begin{split}
        -\|\bar{x}-x\|^2 &\leq -\|\bar{x}-Q_{\bar{x}}\|^2 - \|Q_{\bar{x}}-{\rm L}_{Q_{\bar{x}}}(x)\|^2 + C\eps_{\rm n}^2 \\
        -\|\bar{x}'-x\|^2 &\leq -\|\bar{x}'-Q_{\bar{x}'}\|^2 - \|Q_{\bar{x}'}-{\rm L}_{Q_{\bar{x}'}}(x)\|^2 + C\eps_{\rm n}^2.
    \end{split} 
\end{equation}
Altogether, \eqref{eq:case2Xstep2}, \eqref{case2Xstep1} demonstrate
\begin{equation} \label{eq:case2Xstep3}
    \E_{x\sim \mu} \Big[ X \exp\Big(\frac{-\|x-\bar{x}\|^2-\|x-\bar{x}'\|^2}{2\eps_{\rm w}^2}\Big) \Big]
    \leq \exp\Big(\frac{C\eps_{\rm n}^2}{\eps_{\rm w}^2}\Big)\xi(\bar{x},\bar{x}') = \exp(C\eps^2)\xi(\bar{x},\bar{x}'). 
\end{equation}
Combining \eqref{expansion}, \eqref{eq:case2Y}, \eqref{eq:case2Xstep3}, we conclude
\begin{equation*}
    \omega(\bar{x},\bar{x}') \leq \exp(C\eps^2)\xi(\bar{x},\bar{x}') + \exp\Big(-\frac{c\eps^{4/3}}{\eps^{2\tau/3}}\Big),
\end{equation*}
which is the desired upper bound. 
\end{proof}

\subsection{Proof of Lemma \ref{lem:quadbehavior}} \label{sec:quadbehaviorproof}

\begin{proof}
We will only prove the first inequality in \eqref{quadbehavior}; the second inequality is handled similarly. 
Let $r>0$. 
Let $y\in B^d(Q_{\bar{x}},r)\cap \tilde{T}_{Q_{\bar{x}}}\M$, where $\tilde{T}_{Q_{\bar{x}}}\M$ denotes the affine plane tangent to $\M$ at $Q_{\bar{x}}$. 
Recalling further that $({\rm A}_{Q_{\bar{x}}}\circ \mathsf{I}_m)^{-1}(\tilde{T}_{Q_{\bar{x}}}\M)=T_{Q_{\bar{x}}}\M$ (see \eqref{Gx}), we define ${\rm E}_{Q_{\bar{x}}}:= {\rm Exp}_{Q_{\bar{x}}}\circ ({\rm A}_{Q_{\bar{x}}}\circ \mathsf{I}_m)^{-1}$ and write
\begin{equation} \label{Taylorexp}
    {\rm E}_{Q_{\bar{x}}}(y)=Q_{\bar{x}} + I_{Q_{\bar{x}}}(y)+\frac{1}{2}P_{Q_{\bar{x}}}(y)+\dots
\end{equation}
where $I_{Q_{\bar{x}}}, P_{Q_{\bar{x}}}$ are linear and quadratic forms on $\tilde{T}_{Q_{\bar{x}}}\mathcal{M}$, respectively, with values in $\mathbb{R}^{d}$. 
Consider the plane $\mathcal{P}$ determined by the points $x$, $Q_{\bar{x}}$, ${\rm L}_{Q_{\bar{x}}}(x)$, and let $\gamma$ be the geodesic curve connecting $Q_{\bar{x}}$, $x$ that lies in the intersection $\mathcal{P}\cap\mathcal{B}(Q_{\bar{x}},r)$. 
Let $v$ be the unit vector in the positive direction of ${\rm L}_{Q_{\bar{x}}}(x) - Q_{\bar{x}}$. 
We can suppose that $\gamma$ is parametrized by its arc length $s$, such that
\begin{equation*}
    \gamma(0) = Q_{\bar{x}} \quad\text{ and }\quad \frac{d}{ds}\gamma(0)=v.
\end{equation*}
Now, setting $y=vs$ in \eqref{Taylorexp} leads to
\begin{equation} \label{Taylorcurve}
    \gamma(s)={\rm E}_{Q_{\bar{x}}}(vs)=Q_{\bar{x}} + I_{Q_{\bar{x}}}(v)s+\frac{1}{2}P_{Q_{\bar{x}}}(v)s^2+\dots
\end{equation}
where $I_{Q_{\bar{x}}}(v)=v=\frac{d}{ds}\gamma(0)$ and $P_{Q_{\bar{x}}}(v)=\frac{d^2}{ds^2}\gamma(0)$. 
Furthermore, since $\gamma$ is a geodesic curve on $\mathcal{M}$, we have
\begin{equation*}
    \frac{d^2}{ds^2}\gamma(0)\perp \tilde{T}_{\gamma(0)}\mathcal{M}=\tilde{T}_{Q_{\bar{x}}}\mathcal{M},
\end{equation*}
and since $\gamma$ is a plane curve on $\mathcal{P}$, we can write \citep[Section~1-4]{struik2012lectures} 
\begin{equation*}
    \frac{d^2}{ds^2}\gamma(0) = \kappa\mathfrak{n}
\end{equation*}
where $\mathfrak{n}$ is the unit normal vector of $\gamma$ in $\mathcal{P}$ at $s=0$, and $\kappa$ is the curvature of $\gamma$ at $s=0$. By the assumptions made on $\M$ in Section \ref{sec:assumption}, $|\kappa|\leq K$. Hence, if we let $s=s^{*}>0$ be such that ${\rm E}_{Q_{\bar{x}}}(vs^{*})=x$ in \eqref{Taylorcurve}, we can conclude that
\begin{equation} \label{Os2}
    \|x-Q_{\bar{x}}-vs^{*}\|=\mathcal{O} ((s^{*})^2),
\end{equation}
and the majorant constant is bounded by a multiple of $K$. 
Since $Q_{\bar{x}}+vs^{*}={\rm L}_{Q_{\bar{x}}}(x)$, we obtain from \eqref{Os2}, 
\begin{equation*}
    \|x-{\rm L}_{Q_{\bar{x}}}(x)\|=\mathcal{O}(\|{\rm L}_{Q_{\bar{x}}}(x)-Q_{\bar{x}}\|^2),
\end{equation*}    
and the proof is completed. 
\end{proof}

\subsection{Proof of Lemma \ref{lem:xQxlogQx}} \label{sec:xQxlogQxproof}

\begin{proof}
We will only prove the first inequalities in \eqref{xQxlogQxupbd} and in \eqref{xQxlogQxlwbd}. We prepare two facts.
First, it follows from the orthogonality in \eqref{perp} that 
\begin{equation} \label{pythagorean}
    \|\bar{x}-{\rm L}_{Q_{\bar{x}}}(x)\|^2 = \|\bar{x}-Q_{\bar{x}}\|^2+\|Q_{\bar{x}}-{\rm L}_{Q_{\bar{x}}}(x)\|^2,
\end{equation}
and second, from Lemma \ref{lem:quadbehavior},
\begin{align}\label{quadrecall}
    \|x-{\rm L}_{Q_{\bar{x}}}(x)\|\le C\|Q_{\bar{x}}-{\rm L}_{Q_{\bar{x}}}(x)\|^2.
\end{align}
For the first inequality in \eqref{xQxlogQxupbd}, we have,
\begin{equation} \label{triangleineq}
    \begin{split}
        \|x-\bar{x}\|^2 &\leq (\|{\rm L}_{Q_{\bar{x}}} (x)-\bar{x}\| + \|{\rm L}_{Q_{\bar{x}}}(x)-x\|)^2\\
        &=\|{\rm L}_{Q_{\bar{x}}}(x)-\bar{x}\|^2 + \|{\rm L}_{Q_{\bar{x}}}(x)-x\|^2 + 2\|{\rm L}_{Q_{\bar{x}}}(x)-\bar{x}\|\|{\rm L}_{Q_{\bar{x}}} (x)-x\|.
    \end{split}
\end{equation}
Observe that
\begin{equation} \label{middleterm}
    \begin{split}
        \|{\rm L}_{Q_{\bar{x}}}(x)-\bar{x}\|\|{\rm L}_{Q_{\bar{x}}} (x)-x\| &\leq C\|Q_{\bar{x}} - {\rm L}_{Q_{\bar{x}}}(x)\|^2\sqrt{\|\bar{x}-Q_{\bar{x}}\|^2 + \|Q_{\bar{x}}-{\rm L}_{Q_{\bar{x}}}(x)\|^2}\\
        &\leq C\eps_{\rm n}^{2/3} \|Q_{\bar{x}} - {\rm L}_{Q_{\bar{x}}}(x)\|^2,
    \end{split}
\end{equation}
where the first inequality is due to \eqref{pythagorean}, \eqref{quadrecall} and the second to \eqref{QxlogQxdistance}. 
Putting \eqref{middleterm} back in \eqref{triangleineq}, and invoking all \eqref{QxlogQxdistance}, \eqref{pythagorean}, \eqref{quadrecall} again, allows us to obtain
\begin{equation*}
    \begin{split}
        \|x-\bar{x}\|^2 &\leq \|Q_{\bar{x}}-\bar{x}\|^2 + \|{\rm L}_{Q_{\bar{x}}}(x)-Q_{\bar{x}}\|^2 +  \|{\rm L}_{Q_{\bar{x}}}(x)-x\|^2 + C\eps_{\rm n}^{2/3}\|Q_{\bar{x}} - {\rm L}_{Q_{\bar{x}}}(x)\|^2\\
        &\leq \|Q_{\bar{x}}-\bar{x}\|^2 + \|{\rm L}_{Q_{\bar{x}}}(x)-Q_{\bar{x}}\|^2 +  C\eps_{\rm n}^{2/3}\|Q_{\bar{x}} - {\rm L}_{Q_{\bar{x}}}(x)\|^2 + C\eps_{\rm n}^{2/3}\|Q_{\bar{x}} - {\rm L}_{Q_{\bar{x}}}(x)\|^2\\
        &\leq \|Q_{\bar{x}}-\bar{x}\|^2 + (1+C\eps_{\rm n}^{2/3})\|Q_{\bar{x}} - {\rm L}_{Q_{\bar{x}}}(x)\|^2.
    \end{split}
\end{equation*}
Similarly, to handle the first inequality in \eqref{xQxlogQxlwbd}, we note
\begin{align*} 
   \|x-\bar{x}\|^2 &\geq (\|{\rm L}_{Q_{\bar{x}}}(x)-\bar{x}\|-\|{\rm L}_{Q_{\bar{x}}}(x)-x\|)^2\\
   &=\|{\rm L}_{Q_{\bar{x}}}(x)-\bar{x}\|^2 + \|{\rm L}_{Q_{\bar{x}}}(x)-x\|^2 - 2\|{\rm L}_{Q_{\bar{x}}}(x)-\bar{x}\|\|{\rm L}_{Q_{\bar{x}}}(x)-x\|\\
   &\geq \|{\rm L}_{Q_{\bar{x}}}(x)-\bar{x}\|^2 - 2\|{\rm L}_{Q_{\bar{x}}}(x)-\bar{x}\|\|{\rm L}_{Q_{\bar{x}}}(x)-x\|\\
   &\geq \|\bar{x}-Q_{\bar{x}}\|^2+\|Q_{\bar{x}}-{\rm L}_{Q_{\bar{x}}}(x)\|^2 -C\eps_{\rm n}^{2/3} \|Q_{\bar{x}} - {\rm L}_{Q_{\bar{x}}}(x)\|^2 \\
   &= \|\bar{x}-Q_{\bar{x}}\|^2+ (1-C\eps_{\rm n}^{2/3}) \|Q_{\bar{x}} - {\rm L}_{Q_{\bar{x}}}(x)\|^2,
\end{align*}
where the last inequality is an application of \eqref{pythagorean}, \eqref{middleterm}. 
\end{proof}

\section{Proof of Proposition~\ref{prop:qtrunc}} \label{appx:qtrunc}

\begin{proof}
In Appendix~\ref{appx:pathderivatives}, we express the relevant partial derivatives of $q$ as path derivatives and address their continuity.
We prove \eqref{eq:qprimebounded}, \eqref{eq:qdoubleprimebounded} in Appendices~\ref{appx:qprimebounded},~\ref{appx:qdoubleprimebounded}, respectively.
The proof of \eqref{eq:qbounded} and \eqref{eq:qboundedextra} relies on arguments similar to those in Appendix~\ref{appx:qprimebounded} and is presented in Appendix~\ref{appx:qtruncbounded}. 
\end{proof}

\subsection{First and second partial derivatives of $q$ as path derivatives} \label{appx:pathderivatives}

Let $z\in\M$.
Consider the diffeomorphic map
\begin{align*} 
    \overline{{\rm Exp}}_z: B^m(0,\eps_{\rm n})\times B^{d-m}(0,\eps_{\rm n}) 
    &\to \mathcal{N}_{\eps_{\rm n}}(\mathcal{B}(z,\eps_{\rm n})) \\
    (v,k) &\mapsto z' + h
\end{align*}
constructed in Appendix~\ref{appx:dg}, with $r=s=\eps_{\rm n}$. 
Here, $z'={\rm Exp}_z(v)$ and $h = {\rm R}_{z'}\circ {\rm J}_{d-m}(k)$.
Under this map, a straight line $\Gamma: [0,1]\to B^m(0,\eps_{\rm n})$, such that $\Gamma(t) = tu$, with $u\in\mathbb{S}^{m-1}$, is mapped to a geodesic curve $\gamma: [0,1]\to\mathcal{B}(z,\eps_{\rm n})$, such that $\gamma(0)=z$ and
\begin{equation*}
    \frac{d}{dt}\gamma(t)|_{t=0} =\tilde{u}= \langle \nabla {\rm Exp}_z(0), u\rangle \in \tilde{T}_z\M-z
\end{equation*} 
is a unit vector. 
Letting $z'$ denote a generic point on $\gamma$, and following definition \eqref{eqdef:q}, we express $q(z')$ as
\begin{equation*}  
    q(\gamma(t)) 
    = \frac{1}{(\eps_{\rm p}\sqrt{2\pi})^d} \int_{B^{d-m}(0,\eps_{\rm n})} \int_{\M} \exp\Big(-\frac{\|{\rm Exp}_z(tu) + {\rm R}_{\gamma(t)}\circ {\rm J}_{d-m}(k) -x\|^2}{2\eps_{\rm p}^2}\Big)\rho(x)\dd\mathcal{V}(x)\dd k.
\end{equation*}
Hence, 
\begin{multline} \label{eq:qz2}
    \frac{d}{dt} q(\gamma(t))|_{t=0} \\
    = \frac{1}{(\eps_{\rm p}\sqrt{2\pi})^d} \int_{B^{d-m}(0,\eps_{\rm n})} \int_{\M} \frac{d}{dt} \exp\Big(-\frac{\|{\rm Exp}_z(tu) + {\rm R}_{\gamma(t)}\circ {\rm J}_{d-m}(k) -x\|^2}{2\eps_{\rm p}^2}\Big)\Big|_{t=0}\rho(x)\dd\mathcal{V}(x)\dd k 
\end{multline}
provides us with the directional derivative $\partial_{\tilde{u}} q(z)$.
Observe, since ${\rm R}_{\gamma(t)}$ is a rotation, we get 
\begin{equation} \label{eq:rotortho}
    \big\langle \frac{d}{dt} {\rm R}_{\gamma(t)} y, {\rm R}_{\gamma(t)} y \big\rangle = 0,
\end{equation}
for all $y\in\mathbb{R}^d$. 
Then, from \eqref{eq:qz2} and a straightforward calculation,
\begin{equation} \label{eq:splitqz2}
    \partial_{\tilde{u}} q(z) = \frac{d}{dt} q(\gamma(t))|_{t=0} = E+Z,
\end{equation}
where, due to the orthogonality $h\perp\tilde{u}$,
\begin{align*}
    E &:= -\frac{1}{(\eps_{\rm p}\sqrt{2\pi})^d} \int_{h+z \in \tilde{N}_{z}\M\cap \mathcal{N}_{\eps_{\rm n}}(\M)}
    \int_{\M} \frac{\langle z + h - x, \tilde{u} \rangle}{\eps_{\rm p}^2} 
    \exp\Big(-\frac{\|z+h-x\|^2}{2\eps_{\rm p}^2}\Big) \rho(x)\dd\mathcal{V}(x)\dd h \\
    &= -\frac{1}{(\eps_{\rm p}\sqrt{2\pi})^d} \int_{h+z \in \tilde{N}_{z}\M\cap \mathcal{N}_{\eps_{\rm n}}(\M)}
    \int_{\M} \frac{\langle z - x, \tilde{u} \rangle}{\eps_{\rm p}^2} 
    \exp\Big(-\frac{\|z+h-x\|^2}{2\eps_{\rm p}^2}\Big) \rho(x)\dd\mathcal{V}(x)\dd h,
\end{align*}
and due to \eqref{eq:rotortho},
\begin{equation*}
    Z := -\frac{1}{(\eps_{\rm p}\sqrt{2\pi})^d} \int_{h+z \in \tilde{N}_{z}\M\cap \mathcal{N}_{\eps_{\rm n}}(\M)}
    \int_{\M} \frac{\langle z - x, \mathsf{A}h \rangle}{\eps_{\rm p}^2} 
    \exp\Big(-\frac{\|z+ h -x\|^2}{2\eps_{\rm p}^2}\Big) \rho(x)\dd\mathcal{V}(x)\dd h,
\end{equation*}
where $\mathsf{A} := \frac{d}{dt} {\rm R}_{\gamma(t)}|_{t=0}$.

Let $w\in\mathbb{S}^{m-1}$ and $\tilde{w} = \nabla {\rm Exp}_z(0) w \in \tilde{T}_z\M-z$.
Repeating this process for $\partial_{\tilde{u}}q$ in place of $q$, and $w$, $\tilde{w}$ in place of $u$, $\tilde{u}$, we interpret the second partial derivative $\partial_{\tilde{w}}\partial_{\tilde{u}}q$ as the time derivative of $\partial_{\tilde{u}}q(\zeta(t))$ at $t=0$, where $\zeta: [0,1]\to \mathcal{B}(z,\eps_{\rm n})$ is a geodesic curve on $\M$ starting at $z$ with velocity $\tilde{w}$. 
Thus, from \eqref{eq:splitqz2}
\begin{equation} \label{eq:splitqz3}
    \partial_{\tilde{w}}\partial_{\tilde{u}} q(z) = \frac{d}{dt} \partial_{\tilde{u}} q(\zeta(t))|_{t=0} = D + V + F + S
\end{equation}
where\footnote{The letters $E, Z, D, V, F, S$ are chosen correspondingly to the German words for ``first, second, third, fourth, fifth, sixth'' which are ``Erster, Zweiter, Dritter, Vierter, F\"unfter, Sechster'' respectively.}
\begin{equation*}
    \begin{split}
        D := \frac{1}{(\eps_{\rm p}\sqrt{2\pi})^d} \int_{h+z \in \tilde{N}_{z}\M\cap \mathcal{N}_{\eps_{\rm n}}(\M)} 
    \int_{\M} \frac{\langle z - x, \tilde{u} + \mathsf{A}h \rangle \langle z - x, \tilde{w} + \mathsf{A}h \rangle}{\eps_{\rm p}^4}\\
    \times \exp\Big(-\frac{\| z +h-x\|^2}{2\eps_{\rm p}^2}\Big) \rho(x)\dd\mathcal{V}(x)\dd h,
    \end{split}
\end{equation*}
and
\begin{align*}
    V &:= -\frac{1}{(\eps_{\rm p}\sqrt{2\pi})^d} \int_{h+z \in \tilde{N}_{z}\M\cap \mathcal{N}_{\eps_{\rm n}}(\M)} 
    \int_{\M} \frac{\langle \tilde{w}, \tilde{u} \rangle}{\eps_{\rm p}^2} \exp\Big(-\frac{\| z +h-x\|^2}{2\eps_{\rm p}^2}\Big) \rho(x)\dd\mathcal{V}(x)\dd h \\
    F &:= -\frac{1}{(\eps_{\rm p}\sqrt{2\pi})^d} \int_{h+z \in \tilde{N}_{z}\M\cap \mathcal{N}_{\eps_{\rm n}}(\M)} 
    \int_{\M} \frac{\langle \tilde{w}, \mathsf{A}h \rangle}{\eps_{\rm p}^2} \exp\Big(-\frac{\| z +h-x\|^2}{2\eps_{\rm p}^2}\Big) \rho(x)\dd\mathcal{V}(x)\dd h \\
    S &:= \frac{1}{(\eps_{\rm p}\sqrt{2\pi})^d} \int_{h+z \in \tilde{N}_{z}\M\cap \mathcal{N}_{\eps_{\rm n}}(\M)} 
    \int_{\M} \frac{\langle z-x, \mathsf{B}h \rangle}{\eps_{\rm p}^2} \exp\Big(-\frac{\| z +h-x\|^2}{2\eps_{\rm p}^2}\Big) \rho(x)\dd\mathcal{V}(x)\dd h,
\end{align*}
where $\mathsf{B}:=\frac{d^2}{dt^2} {\rm R}_{\gamma(t)}|_{t=0}$.

We conclude by observing from \eqref{eq:splitqz2}, \eqref{eq:splitqz3} that $\partial_{\tilde{u}} q$, $\partial_{\tilde{w}}\partial_{\tilde{u}} q$ are both continuous in terms of $z$.

\subsection{Proof of \eqref{eq:qprimebounded} in Proposition~\ref{prop:qtrunc}} \label{appx:qprimebounded}

\begin{proof}
From \eqref{eq:splitqz2}, let $E=E_1+E_2$, where
\begin{align*}
    E_1 & {:=} {-}\frac{1}{(\eps_{\rm p}\sqrt{2\pi})^d} \int_{h{+}z \in \tilde{N}_{z}\M\cap \mathcal{N}_{\eps_{\rm n}}(\M)} \int_{\mathcal{B}(z,\eps_{\rm n}^{2/3})} \frac{\langle z{-}x,\tilde{u}\rangle}{\eps_{\rm p}^2} \exp\Big({-}\frac{\|z{+}h{-}x\|^2}{2\eps_{\rm p}^2}\Big)\rho(x)\dd\mathcal{V}(x) \dd h, \\
    E_2 & {:=} {-}\frac{1}{(\eps_{\rm p}\sqrt{2\pi})^d} {\int_{h{+}z \in \tilde{N}_{z}\M\cap \mathcal{N}_{\eps_{\rm n}}(\M)} }\int_{\M\setminus\mathcal{B}(z,\eps_{\rm n}^{2/3})} {\frac{\langle z{-}x,\tilde{u}\rangle}{\eps_{\rm p}^2}} \exp\Big({-}\frac{\|z{+}h{-}x\|^2}{2\eps_{\rm p}^2}\Big)\rho(x)\dd\mathcal{V}(x) \dd h.
\end{align*}
Starting with $E_2$, we note from \eqref{eq:distancecompare} that, $d(z,x)\geq \eps_{\rm n}^{2/3}$ implies $\|x-z\|\geq c\eps_{\rm n}^{2/3}$. 
Further, from the fact that $\|h\|\leq\eps_{\rm n}$, we also acquire
\begin{equation} \label{ineqchain}
    c\eps_{\rm n}^{2/3} \leq cd(z,x) \leq \|z-x\| \leq c'(\|z-x\|-\|h\|) \leq c'\|z+h-x\|. 
\end{equation}
Recall the quantity $\mathfrak{D} = \sup\{\|x-x'\|: x,x'\in\M\}$ in \eqref{mandiam}. Since $\M$ is a closed manifold, $\mathfrak{D} <\infty$. 
Then gathering all in \eqref{ineqchain}, we deduce,
\begin{align}
    \nonumber |E_2| &{\leq} \frac{1}{(\eps_{\rm p}\sqrt{2\pi})^d}\int_{h+z \in \tilde{N}_{z}\M\cap \mathcal{N}_{\eps_{\rm n}}(\M)} \int_{\M\setminus\mathcal{B}(z,\eps_{\rm n}^{2/3})} \frac{\| z - x\|}{\eps_{\rm p}^2} \exp\Big(-\frac{\| z + h - x\|^2}{2\eps_{\rm p}^2}\Big)\rho(x)\dd\mathcal{V}(x) \dd h \\
    \label{eq:qE2bound} &{\leq} \frac{\mathfrak{D}}{(\eps_{\rm p}\sqrt{2\pi})^d} \int_{h+z \in \tilde{N}_{z}\M\cap \mathcal{N}_{\eps_{\rm n}}(\M)} \int_{\M\setminus\mathcal{B}(z,\eps_{\rm n}^{2/3})} \frac{1}{\eps_{\rm p}^2} \exp\Big({-}\frac{c}{\eta^{2/d} \eps_{\rm n}^{2/3}}\Big)\rho(x)\dd\mathcal{V}(x) \dd h.
\end{align}
By taking $\eps$ sufficiently small satisfying \eqref{thecondition1}, we acquire from \eqref{eq:qE2bound}, 
\begin{equation} \label{eq:boundingqE2}
    |E_2| \leq \frac{C_{\eta}\eps_{\rm n}^m}{\eta^{2/d} (\eps_{\rm p}\sqrt{2\pi})^d} \int_{h+z \in \tilde{N}_{z}\M\cap \mathcal{N}_{\eps_{\rm n}}(\M)} \dd h \leq C_{\eta}.
\end{equation}
Moving on to $E_1$, recall that we can associate for each $x\in\mathcal{B}(z,\eps_{\rm n}^{2/3})$, a unique ${\rm L}_z(x)-z \in \tilde{T}_z\M-z$ such that Lemma~\ref{lem:quadbehavior} holds.
To simplify the notation, we will write $v_x:={\rm L}_z(x)-z$ from here onward. 
Then
\begin{equation} \label{eq:quadbehavior2}
    \|v_x\| = d(z,x) \quad\text{ and } \quad \|x-(v_x+z)\|=\|x-{\rm L}_z(x)\|\leq C\| {\rm L}_z(x) - z\|^2 = Cd(z,x)^2, 
\end{equation}
and hence, due to the orthogonality $v_x\perp h$,
\begin{equation*} 
    \|z+h-x\|^2 = \|z-x\|^2 + \|h\|^2 + 2\langle z-x,h\rangle 
     = \|z-x\|^2 + \|h\|^2 + 2\langle z+v_x-x,h\rangle.
\end{equation*}
The following straightforward lemma holds.

\begin{lemma} \label{lem:keyTayexp}
For $x\in\mathcal{B}(z,\eps_{\rm n}^{2/3})$, it holds that
\begin{equation} \label{eq:coreexpfactor2}
    \exp\Big(-\frac{\|x-z\|^2}{2\eps_{\rm p}^2} + \frac{d(z,x)^2}{2\eps_{\rm p}^2}\Big) \leq \exp\Big(\frac{C\|x-z\|^4}{\eps_{\rm p}^2}\Big) \leq 1 + \frac{C\|x-z\|^4}{\eps_{\rm p}^2},
\end{equation}
and that
\begin{equation} \label{eq:coreexpfactor1}
    \exp\Big(-\frac{Cd(z,x)^2\|h\|}{\eps_{\rm p}^2}\Big)
    \leq \exp\Big(-\frac{\langle z+v_x-x,h\rangle}{\eps_{\rm p}^2}\Big) 
    \leq \exp\Big(\frac{Cd(z,x)^2\|h\|}{\eps_{\rm p}^2}\Big).
\end{equation}
Furthermore, if $h+z\in \tilde{N}_z\M\cap\mathcal{N}_{\eps_{\rm n}}(\M)$, then
\begin{align} \label{eq:keyTayexp}
    \nonumber \exp\Big(-\frac{\langle z+v_x-x,h\rangle}{\eps_{\rm p}^2}\Big)
    = 1 &- \frac{\langle z+v_x-x,h\rangle}{\eps_{\rm p}^2} 
    + \frac{\langle z+v_x-x,h\rangle^2}{2\eps_{\rm p}^4} - \frac{\langle z+v_x-x,h\rangle^3}{3!\eps_{\rm p}^6} \\
    &+ \frac{\langle z+v_x-x,h\rangle^4}{4!\eps_{\rm p}^8} -
    \frac{\langle z+v_x-x,h\rangle^5}{5!\eps_{\rm p}^{10}} +
    \mathcal{O}_{\eta}(\eps_{\rm n}^2).
\end{align} 
\end{lemma}

\begin{proof}[Proof of Lemma~\ref{lem:keyTayexp}]
Observe that \eqref{eq:coreexpfactor2} follows from a combination of \eqref{eq:distancecompare} and Taylor's expansion, \eqref{eq:coreexpfactor1} from \eqref{eq:quadbehavior2}, and \eqref{eq:keyTayexp} from Taylor's expansion again. 
\end{proof}

Coming back to $E_1$, we split $E_1 = E_{11}+E_{12}$, where
\begin{align*}
    E_{11} &{:=} \frac{1}{(\eps_{\rm p}\sqrt{2\pi})^d} \int_{h+z \in \tilde{N}_{z}\M\cap \mathcal{N}_{\eps_{\rm n}}(\M)} \int_{\mathcal{B}(z,\eps_{\rm n}^{2/3})} \frac{\langle \tilde{u},z{+} v_x {-}x \rangle}{\eps_{\rm p}^2} \exp\Big({-}\frac{\|z{+}h{-}x\|^2}{2\eps_{\rm p}^2}\Big)\rho(x)\dd\mathcal{V}(x)  \dd h, \\
    E_{12} &{:=} \frac{1}{(\eps_{\rm p}\sqrt{2\pi})^d} \int_{h+z \in \tilde{N}_{z}\M\cap \mathcal{N}_{\eps_{\rm n}}(\M)} \int_{\mathcal{B}(z,\eps_{\rm n}^{2/3})} {-}\frac{\langle \tilde{u},v_x\rangle}{\eps_{\rm p}^2} \exp\Big(-\frac{\|z{+}h{-}x\|^2}{2\eps_{\rm p}^2}\Big)\rho(x)\dd\mathcal{V}(x)  \dd h.
\end{align*}
Guided by Lemma~\ref{lem:keyTayexp}, we let 
\begin{multline*}
    E_{11}^* 
    := \frac{1}{(\eps_{\rm p}\sqrt{2\pi})^d} \int_{h+z\in \tilde{N}_{z}\M\cap \mathcal{N}_{\eps_{\rm n}}(\M)} \int_{\mathcal{B}(z,\eps_{\rm n}^{2/3})} \frac{\langle \tilde{u},z+ v_x -x \rangle}{\eps_{\rm p}^2}\\
    \times \exp\Big(-\frac{\|h\|^2}{2\eps_{\rm p}^2}\Big)\exp\Big(-\frac{\|x-z\|^2}{2\eps_{\rm p}^2}\Big) \rho(x)\dd\mathcal{V}(x)  \dd h, 
    \end{multline*}
    and
    \begin{multline*}
    E_{12}^* := \frac{1}{(\eps_{\rm p}\sqrt{2\pi})^d} \int_{h+z\in \tilde{N}_{z}\M\cap \mathcal{N}_{\eps_{\rm n}}(\M)} \int_{\mathcal{B}(z,\eps_{\rm n}^{2/3})} -\frac{\langle \tilde{u},v_x\rangle}{\eps_{\rm p}^2}\\
    \times \exp\Big(-\frac{\|h\|^2}{2\eps_{\rm p}^2}\Big) \exp\Big(-\frac{\|x-z\|^2}{2\eps_{\rm p}^2}\Big) \rho(x)\dd\mathcal{V}(x)  \dd h.
\end{multline*}
Then another lemma applies.

\begin{lemma} \label{lem:Ii1s}
Let $E^*_{1i}$ be defined as above, for $i=1,2$.
Then $|E^*_{1i}| \leq C$.
\end{lemma}

\begin{proof}[Proof of Lemma~\ref{lem:Ii1s}]
It is straightforward from \eqref{eq:distancecompare}, \eqref{eq:quadbehavior2} and a familiar calculation that 
\begin{align*}
    |E_{11}^*| 
    &{\leq} \frac{C}{(\eps_{\rm p}\sqrt{2\pi})^d} \int_{h{+}z \in \tilde{N}_{z}\M\cap \mathcal{N}_{\eps_{\rm n}}(\M)} \int_{\mathcal{B}(z,\eps_{\rm n}^{2/3})} \frac{d(z,x)^2}{\eps_{\rm p}^2} \\
    &\qquad\qquad\qquad\qquad\times \exp\Big({-}\frac{\|h\|^2}{2\eps_{\rm p}^2}\Big) \exp\Big({-}\frac{cd(z,x)^2}{\eps_{\rm p}^2}\Big) \rho(x)\dd\mathcal{V}(x) \dd h\\
    &= \frac{C\eps_{\rm p}^{d-m}}{(\eps_{\rm p}\sqrt{2\pi})^d} \noindent \int_{[-\eta^{-1/d},\eta^{-1/d}]^{d-m}} \exp\Big(-\frac{\|h\|^2}{2}\Big) \dd h \int_{B^m(0,\eps_{\rm n}^{2/3})\subset T_z\M} \frac{\|y\|^2}{\eps_{\rm p}^2}\\
    &\qquad\qquad\qquad\qquad\times \exp\Big(-\frac{c\|y\|^2}{\eps_{\rm p}^2}\Big) \tilde{\rho}(y) J_{z}(y)  \dd y  \\
    &\leq C.
\end{align*} 
We next focus on a version of $E_{12}^*$, which is
\begin{equation*}
    E^{**}_{12} {:=} {-} \frac{\rho(z)}{(\eps_{\rm p}\sqrt{2\pi})^d} \int_{h{+}z\in \tilde{N}_z\M\cap\mathcal{N}_{\eps_{\rm n}}(\M)} \int_{\mathcal{B}(z,\eps_{\rm n}^{2/3})} \frac{\langle \tilde{u},v_x\rangle}{\eps_{\rm p}^2} \exp\Big({-}\frac{\|h\|^2}{2\eps_{\rm p}^2}\Big) \exp\Big({-}\frac{d(z,x)^2}{2\eps_{\rm p}^2}\Big) \dd\mathcal{V}(x)\dd h.
\end{equation*}
Note, for a fixed $h+z\in \tilde{N}_z\M\cap\mathcal{N}_{\eps_{\rm n}}(\M)$, if $x,x'\in\M$ are such that $v_x = -v_{x'}$, then $d(z,x)= d(z,x')$, and conversely. 
Therefore
\begin{equation*}
    -\int_{\mathcal{B}(z,\eps_{\rm n}^{2/3})} \frac{\langle \tilde{u},v_x\rangle}{\eps_{\rm p}^2} \exp\Big(-\frac{d(z,x)^2}{2\eps_{\rm p}^2}\Big)\dd\mathcal{V}(x) = 0,
\end{equation*}
and we conclude that $E_{12}^{**}=0$. 
Thus, by virtue of Lemma~\ref{lem:keyTayexp}, particularly \eqref{eq:coreexpfactor2}, the Lipschitz continuity of $\rho$, and \eqref{eq:quadbehavior2},
\begin{align*}
    |E^*_{12} {-} E^{**}_{12}| & {\leq} \frac{C}{(\eps_{\rm p}\sqrt{2\pi})^d} \int_{h {+} z\in \tilde{N}_z\M\cap\mathcal{N}_{\eps_{\rm n}}(\M)} \exp\Big( {-} \frac{\|h\|^2}{2\eps_{\rm p}^2}\Big) \int_{\mathcal{B}(z,\eps_{\rm n}^{2/3})} \frac{d(z,x)^2}{\eps_{\rm p}^2} \\
    &\qquad\qquad\qquad\qquad\times\exp\Big( {-} \frac{cd(z,x)^2}{\eps_{\rm p}^2}\Big) \dd\mathcal{V}(x) \dd h \\
    &\leq C.
\end{align*}
Henceforth, $|E^*_{12}|\leq C$. 
\end{proof}

In view of Lemma~\ref{lem:keyTayexp}, particularly \eqref{eq:keyTayexp},
\begin{align} \label{specialE}
    \nonumber |E_{11} - E_{11}^*| &\leq \frac{C\eps_{\rm n}}{(\eps_{\rm p}\sqrt{2\pi})^d} \int_{h+z \in \tilde{N}_{z}\M\cap \mathcal{N}_{\eps_{\rm n}}(\M)} \\
     &\qquad\qquad\times\exp\Big(-\frac{\|y\|^2}{2\eps_{\rm p}^2}\Big) 
    \int_{\mathcal{B}(z,\eps_{\rm n}^{2/3})} \frac{d(z,x)^4}{\eps_{\rm p}^4} \exp\Big(-\frac{cd(z,x)^2}{\eps_{\rm p}^2}\Big)  \dd\mathcal{V}(x) \dd h 
    \leq C \eps_{\rm n},
\end{align}
and
\begin{align*}
    |E_{12} - E_{12}^*| &\leq \frac{C}{(\eta^{1/d} \eps_{\rm p}\sqrt{2\pi})^d} \int_{h+z \in \tilde{N}_{z}\M\cap \mathcal{N}_{\eps_{\rm n}}(\M)} \exp\Big(-\frac{\|y\|^2}{2\eps_{\rm p}^2}\Big) 
    \int_{\mathcal{B}(z,\eps_{\rm n}^{2/3})} \frac{d(z,x)^3}{\eps_{\rm p}^3} \\
    &\qquad\qquad\qquad\qquad\times\exp\Big(-\frac{cd(z,x)^2}{\eps_{\rm p}^2}\Big)  \dd\mathcal{V}(x) \dd h 
    \leq C_{\eta}.
\end{align*}
Thus, in view of Lemma~\ref{lem:Ii1s}, $|E_{1i}| \leq C_{\eta}$.
Consequently,
\begin{equation} \label{eq:boundingqE1}
    |E_1| = |E_{11} + E_{12}| \leq C_{\eta}.
\end{equation}
Combining \eqref{eq:boundingqE2}, \eqref{eq:boundingqE1}, we arrive at 
\begin{equation} \label{eq:boundingqE}
    |E| = |E_1+E_2| \leq C_{\eta}.
\end{equation}

Finally, regarding $Z$ in \eqref{eq:splitqz2}, we note that $Z$ can also be split into
\begin{align*}
    Z_1 &{:=} {-}\frac{1}{(\eps_{\rm p}\sqrt{2\pi})^d} \int_{h+z \in \tilde{N}_{z}\M\cap \mathcal{N}_{\eps_{\rm n}}(\M)} \int_{\mathcal{B}(z,\eps_{\rm n}^{2/3})} \frac{\langle z{-}x, \mathsf{A}h \rangle}{\eps_{\rm p}^2} \exp\Big({-}\frac{\|z{+}h{-}x\|^2}{2\eps_{\rm p}^2}\Big)\rho(x)\dd\mathcal{V}(x) \dd h \\
    &= -\frac{1}{(\eps_{\rm p}\sqrt{2\pi})^d} \int_{h+z \in \tilde{N}_{z}\M\cap \mathcal{N}_{\eps_{\rm n}}(\M)} \|\mathsf{A}h\| \int_{\mathcal{B}(z,\eps_{\rm n}^{2/3})} \frac{\langle z-x, \mathsf{A}h / \|\mathsf{A}h\| \rangle}{\eps_{\rm p}^2} \\
    &\qquad\qquad\qquad\qquad\times\exp\Big(-\frac{\|z+h-x\|^2}{2\eps_{\rm p}^2}\Big)\rho(x)\dd\mathcal{V}(x) \dd h,
\end{align*}
and
\begin{equation*}
    Z_2 {:=} {-}\frac{1}{(\eps_{\rm p}\sqrt{2\pi})^d} \int_{h{+}z \in \tilde{N}_{z}\M\cap \mathcal{N}_{\eps_{\rm n}}(\M)} \int_{\M\setminus\mathcal{B}(z,\eps_{\rm n}^{2/3})} \frac{\langle z{-}x, \mathsf{A}h \rangle}{\eps_{\rm p}^2} \exp\Big({-}\frac{\|z{+}h{-}x\|^2}{2\eps_{\rm p}^2}\Big)\rho(x)\dd\mathcal{V}(x) \dd h.
\end{equation*}
Since $|\langle z-x, \mathsf{A}h \rangle|\leq C\|z-x\|\eps_{\rm n}$, following a similar calculation to that in \eqref{eq:qE2bound}, \eqref{eq:boundingqE2}, we get, for $\eps$ sufficiently small
\begin{align} \label{eq:qboundingZ2}
    \nonumber |Z_2| &\leq \frac{C_{\eta}\mathfrak{D}}{(\eps_{\rm p}\sqrt{2\pi})^d} \int_{h+z \in \tilde{N}_{z}\M\cap \mathcal{N}_{\eps_{\rm n}}(\M)} \int_{\M\setminus\mathcal{B}(z,\eps_{\rm n}^{2/3})} \frac{1}{\eps_{\rm p}} \exp\Big(-\frac{c}{\eta^{2/d} \eps_{\rm n}^{2/3}}\Big) \rho(x)\dd\mathcal{V}(x) \dd h \\
    &\leq C_{\eta}.
\end{align}
For $Z_1$, we draw an observation from the expression of $Z_1$ and the treatment of $E_1$, particularly from the proof of Lemma~\ref{lem:Ii1s} that
\begin{equation} \label{eq:qboundingZ1}
    |Z_1| \leq C_{\eta} \eps_{\rm n}.
\end{equation}
Then \eqref{eq:qboundingZ2}, \eqref{eq:qboundingZ1} imply 
\begin{equation} \label{eq:boundingqZ}
    |Z| = |Z_1+Z_2| \leq C_{\eta}.
\end{equation}
Altogether, \eqref{eq:splitqz2}, \eqref{eq:boundingqE}, \eqref{eq:boundingqZ} conclude $|\partial_{\tilde{u}}q(z)|\leq |E|+|Z|\leq C_{\eta}$, as desired. 
\end{proof}

\subsection{Proof of \eqref{eq:qboundedextra} and \eqref{eq:qbounded} in Proposition~\ref{prop:qtrunc}} \label{appx:qtruncbounded}

\begin{proof}
To establish the boundedness of $q$ in \eqref{eqdef:q}, we reuse the arguments presented in Appendix~\ref{appx:qprimebounded}.
Specifically, we observe from the integral expression \eqref{eqdef:q} that $q(z)$ can be approximated by
\begin{equation} \label{eq:integralqsplit1}
    A := \frac{1}{(\eps_{\rm p}\sqrt{2\pi})^d} \int_{h+z \in \tilde{N}_{z}\M\cap \mathcal{N}_{\eps_{\rm n}}(\M)} \int_{\mathcal{B}(z,\eps_{\rm n}^{2/3})} \exp\Big(-\frac{d(z,x)^2}{2\eps_{\rm p}^2}\Big)\exp\Big(-\frac{\|h\|^2}{2\eps_{\rm p}^2}\Big)\rho(x) \dd\mathcal{V}(x) \dd h.
\end{equation}
Indeed, using reasoning akin to that employed for estimating $E_1$ in Appendix~\ref{appx:qprimebounded}, we find that $|q(z)-A|$ to be bounded by a universal constant multiple of $B_1 + B_2$, where
\begin{align*}
    B_1 &{:=} \frac{\eps_{\rm n}}{(\eps_{\rm p}\sqrt{2\pi})^d} \int_{h+z \in \tilde{N}_{z}\M\cap \mathcal{N}_{\eps_{\rm n}}(\M)} \exp\Big({-}\frac{\|y\|^2}{2\eps_{\rm p}^2}\Big) 
    \int_{\mathcal{B}(z,\eps_{\rm n}^{2/3})} \frac{d(z,x)^2}{\eps_{\rm p}^2} \exp\Big({-}\frac{d(z,x)^2}{\eps_{\rm p}^2}\Big)  \dd\mathcal{V}(x) \dd h \\
    B_2 &{:=} \frac{1}{(\eps_{\rm p}\sqrt{2\pi})^d} \int_{h+z \in \tilde{N}_{z}\M\cap \mathcal{N}_{\eps_{\rm n}}(\M)} \int_{\M\setminus \mathcal{B}(z,\eps_{\rm n}^{2/3})} \exp\Big(-\frac{\|z+h-x\|^2}{2\eps_{\rm p}^2}\Big)\rho(x) \dd\mathcal{V}(x) \dd h.
\end{align*}
To derive $B_1$, we use \eqref{eq:coreexpfactor2}, \eqref{eq:keyTayexp}, and a similar approach to the one yielding \eqref{specialE}.
We skip further details.
Then, $B_1\leq \eps_{\rm n}$, and by referencing the calculations used to bound $|E_2|$ in \eqref{eq:qE2bound}, albeit with a simpler argument, we can show 
\begin{equation*} 
    B_2 = \frac{1}{(\eps_{\rm p}\sqrt{2\pi})^d} \int_{h+z \in \tilde{N}_{z}\M\cap \mathcal{N}_{\eps_{\rm n}}(\M)} \int_{\M\setminus \mathcal{B}(z,\eps_{\rm n}^{2/3})} \exp\Big(-\frac{\|z+h-x\|^2}{2\eps_{\rm p}^2}\Big)\rho(x) \dd\mathcal{V}(x) \dd h \leq C_{\eta}\eps_{\rm n}^2.
\end{equation*}
Back to $A$ in \eqref{eq:integralqsplit1}, we write
\begin{equation} \label{eq:Asplit}
    \begin{split}
        &A {=} \frac{1}{(\eps_{\rm p}\sqrt{2\pi})^d} \int_{h{+}z \in \tilde{N}_{z}\M\cap \mathcal{N}_{\eps_{\rm n}}(\M)} \int_{\mathcal{B}(z,\eps_{\rm n}^{2/3})} \exp\Big({-}\frac{d(z,x)^2}{2\eps_{\rm p}^2}\Big)\exp\Big({-}\frac{\|h\|^2}{2\eps_{\rm p}^2}\Big) (\rho(x) {-} \rho(z))\dd\mathcal{V}(x) \dd h \\
    &\quad \quad {+} \frac{\rho(z)}{(\eps_{\rm p}\sqrt{2\pi})^d} \int_{h+z \in \tilde{N}_{z}\M\cap \mathcal{N}_{\eps_{\rm n}}(\M)} \int_{\mathcal{B}(z,\eps_{\rm n}^{2/3})} \exp\Big(-\frac{d(z,x)^2}{2\eps_{\rm p}^2}\Big)\exp\Big(-\frac{\|h\|^2}{2\eps_{\rm p}^2}\Big) \dd\mathcal{V}(x) \dd h,
    \end{split}
\end{equation}
where $J_z$ is the Jacobian matrix.
Using a now-standard argument and the Lipschitz continuity of $\rho$, we find that the absolute value of the first integral in \eqref{eq:Asplit} is bounded above by $C_{\eta}\eps_{\rm n}^{2/3}$, while the second simplifies to
\begin{align*}
    &\frac{\rho(z)}{(\eps_{\rm p}\sqrt{2\pi})^d} 
    \int_{[-\eps_{\rm n},\eps_{\rm n}]^{d-m}} \exp\Big(-\frac{\|h\|^2}{2\eps_{\rm p}^2}\Big) \dd h \int_{B^m(0,\eps_{\rm n}^{2/3})} \exp\Big(-\frac{\|x\|^2}{2\eps_{\rm p}^2}\Big) J_z(x)\dd x \\
    &{=}\frac{\rho(z)}{(2\pi)^{d/2}} 
    \Big(\int_{[-\eta^{-1/d},\eta^{-1/d}]^{d-m}} \exp\Big({-}\frac{\|h\|^2}{2}\Big) \dd h \int_{B^m(0,\eta^{-1/d}\eps_{\rm n}^{-1/3})} \exp\Big({-}\frac{\|x\|^2}{2}\Big) \dd x\Big) \Big(1{+}\mathcal{O}(\eps_{\rm n}^{4/3})\Big) \\
    &{=} \rho(z) + \mathcal{O}_{\eta}(\eps_{\rm n}^{4/3}).
\end{align*}
Combining all the derivations, we conclude that $q(z) = \rho(z) + \mathcal{O}_{\eta}(\eps_{\rm n}^{2/3})$, which is \eqref{eq:qboundedextra}. 
Then \eqref{eq:qbounded} follows as a consequence of \eqref{eq:qboundedextra}. 
\end{proof}

\subsection{Proof of \eqref{eq:qdoubleprimebounded} in Proposition~\ref{prop:qtrunc}} \label{appx:qdoubleprimebounded}

\begin{proof}
Recall the expression \eqref{eq:splitqz3} in Appendix \ref{appx:pathderivatives} that describes $\partial_{\tilde{u}}\partial_{\tilde{w}} q(z)$.
We proceed to establish that all quantities $D,V,F,S$ in \eqref{eq:splitqz3} are bounded.
As usual, we split each quantity into two parts, a local part, and a tail part.
Specifically, we let
\begin{align*}
    &D_1 
    := \frac{1}{(\eps_{\rm p}\sqrt{2\pi})^d} \int_{h+z \in \tilde{N}_{z}\M\cap \mathcal{N}_{\eps_{\rm n}}(\M)} 
    \int_{\mathcal{B}(z,\eps_{\rm n}^{2/3})} \frac{\langle z - x, \tilde{u} + \mathsf{A}h \rangle \langle z - x, \tilde{w} + \mathsf{A}h \rangle}{\eps_{\rm p}^4} \\
    &\qquad\qquad\qquad\qquad\qquad\qquad\qquad\qquad\qquad\qquad \times \exp\Big(-\frac{\| z +h-x\|^2}{2\eps_{\rm p}^2}\Big) \rho(x)\dd\mathcal{V}(x)\dd h,
\end{align*}
and
\begin{align*}
    V_1 &:= -\frac{1}{(\eps_{\rm p}\sqrt{2\pi})^d} \int_{h+z \in \tilde{N}_{z}\M\cap \mathcal{N}_{\eps_{\rm n}}(\M)} 
    \int_{\mathcal{B}(z,\eps_{\rm n}^{2/3})} \frac{\langle \tilde{w}, \tilde{u} \rangle}{\eps_{\rm p}^2} \exp\Big(-\frac{\| z +h-x\|^2}{2\eps_{\rm p}^2}\Big) \rho(x)\dd\mathcal{V}(x)\dd h \\
    F_1 &:= -\frac{1}{(\eps_{\rm p}\sqrt{2\pi})^d} \int_{h+z \in \tilde{N}_{z}\M\cap \mathcal{N}_{\eps_{\rm n}}(\M)} 
    \int_{\mathcal{B}(z,\eps_{\rm n}^{2/3})} \frac{\langle \tilde{w}, \mathsf{A}h \rangle}{\eps_{\rm p}^2} \exp\Big(-\frac{\| z +h-x\|^2}{2\eps_{\rm p}^2}\Big) \rho(x)\dd\mathcal{V}(x)\dd h \\
    S_1 &:= \frac{1}{(\eps_{\rm p}\sqrt{2\pi})^d} \int_{h+z \in \tilde{N}_{z}\M\cap \mathcal{N}_{\eps_{\rm n}}(\M)} 
    \int_{\mathcal{B}(z,\eps_{\rm n}^{2/3})} \frac{\langle z-x, \mathsf{B}h \rangle}{\eps_{\rm p}^2} \exp\Big(-\frac{\| z +h-x\|^2}{2\eps_{\rm p}^2}\Big) \rho(x)\dd\mathcal{V}(x)\dd h,
\end{align*}
and $D_2:=D-D_1$, $V_2:=V-V_1$, $F_2:=F-F_1$, $S_2:=S-S_1$.

\paragraph{Bounding $F$.} Let 
\begin{multline*}
    F_1^* := -\frac{1}{(\eps_{\rm p}\sqrt{2\pi})^d} \bigg(\int_{h+z \in\tilde{N}_{z}\M\cap \mathcal{N}_{\eps_{\rm n}}(\M)} \frac{\langle \tilde{w}, \mathsf{A}h \rangle}{\eps_{\rm p}^2} \exp\Big(-\frac{\|h\|^2}{2\eps_{\rm p}^2}\Big) \dd h\bigg)\\
    \times
    \bigg(\int_{\mathcal{B}(z,\eps_{\rm n}^{2/3})} \exp\Big(-\frac{\|x-z\|^2}{2\eps_{\rm p}^2}\Big) \rho(x)\dd\mathcal{V}(x)\bigg).
\end{multline*}
Due to symmetry, we have $F_1^*=0$.
Then applying \eqref{eq:quadbehavior2} and Lemma~\ref{lem:keyTayexp}, we obtain 
\begin{align*}
    |F_1-F_1^*| &\leq \frac{C_{\eta}}{(\eps_{\rm p}\sqrt{2\pi})^d} \int_{h+z \in\tilde{N}_{z}\M\cap \mathcal{N}_{\eps_{\rm n}}(\M)} \exp\Big(-\frac{\|h\|^2}{2\eps_{\rm p}^2}\Big)\int_{\mathcal{B}(z,\eps_{\rm n}^{2/3})} \frac{d(z,x)^2}{\eps_{\rm p}}\\
    &\qquad\qquad\qquad\qquad \qquad\qquad\qquad\qquad\times \exp\Big(-\frac{cd(z,x)^2}{\eps_{\rm p}^2}\Big) \rho(x)\dd\mathcal{V}(x) \dd h \leq C_{\eta}\eps_{\rm n}^{2/3}.
\end{align*}
Thus, $|F_1|\leq C_{\eta}\eps_{\rm n}^{2/3}$. 
Moreover, by employing a similar argument as in \eqref{eq:qE2bound}, \eqref{eq:boundingqE2}
\begin{equation*}
    |F_2| 
    \leq \frac{C}{(\eps_{\rm p}\sqrt{2\pi})^d} \int_{h+z \in\tilde{N}_{z}\M\cap \mathcal{N}_{\eps_{\rm n}}(\M)} \frac{\|h\|}{\eta^{2/d}\eps_{\rm p}} 
    \int_{\M\setminus\mathcal{B}(z,\eps_{\rm n}^{2/3})} \frac{1}{\eps_{\rm p}} \exp\Big(-\frac{c}{\eps_{\rm n}^{2/3}}\Big) \rho(x)\dd\mathcal{V}(x)\dd h \leq C_{\eta}.
\end{equation*}
Gathering all this, we conclude $|F|\leq C_{\eta}$.

\paragraph{Bounding $S$.} We note that the expression of $S$ is analogous to that of $Z$ in \eqref{eq:splitqz2}. 
Therefore, by applying a similar argument to deduce \eqref{eq:boundingqZ} previously, we obtain $|S| \leq C_{\eta}$.

\paragraph{Bounding $D$ and $V$.} We write $D_1=D_{11} + D_{12} + D_{13} + D_{14}$,
where 
\begin{align*}
    D_{11} &:=\frac{1}{(\eps_{\rm p}\sqrt{2\pi})^d} \int_{h+z \in \tilde{N}_{z}\M\cap \mathcal{N}_{\eps_{\rm n}}(\M)} \int_{\mathcal{B}(z,\eps_{\rm n}^{2/3})} \frac{\langle z-x,\tilde{u}\rangle}{\eps_{\rm p}^2}\frac{\langle z-x,\tilde{w}\rangle}{\eps_{\rm p}^2} \\
    &\qquad\qquad\qquad\qquad\qquad\qquad\qquad\qquad\times\exp\Big(-\frac{\|z+h-x\|^2}{2\eps_{\rm p}^2}\Big) \rho(x)\dd\mathcal{V}(x)\dd h,\\
    D_{12} &:=\frac{1}{(\eps_{\rm p}\sqrt{2\pi})^d} \int_{h+z \in \tilde{N}_{z}\M\cap \mathcal{N}_{\eps_{\rm n}}(\M)} \int_{\mathcal{B}(z,\eps_{\rm n}^{2/3})} \frac{\langle z-x,\tilde{u}\rangle}{\eps_{\rm p}^2}\frac{\langle z-x, \mathsf{A}h \rangle}{\eps_{\rm p}^2} \\
    &\qquad\qquad\qquad\qquad\qquad\qquad\qquad\qquad\times\exp\Big(-\frac{\|z+h-x\|^2}{2\eps_{\rm p}^2}\Big) \rho(x)\dd\mathcal{V}(x)\dd h ,\\
    D_{13} &:=\frac{1}{(\eps_{\rm p}\sqrt{2\pi})^d} \int_{h+z \in \tilde{N}_{z}\M\cap \mathcal{N}_{\eps_{\rm n}}(\M)} \int_{\mathcal{B}(z,\eps_{\rm n}^{2/3})} \frac{\langle z-x, \mathsf{A}h \rangle}{\eps_{\rm p}^2}\frac{\langle z-x,\tilde{w}\rangle}{\eps_{\rm p}^2} \\
    &\qquad\qquad\qquad\qquad\qquad\qquad\qquad\qquad\times\exp\Big(-\frac{\|z+h-x\|^2}{2\eps_{\rm p}^2}\Big) \rho(x)\dd\mathcal{V}(x)\dd h, \\
    D_{14} &:=\frac{1}{(\eps_{\rm p}\sqrt{2\pi})^d} \int_{h+z \in \tilde{N}_{z}\M\cap \mathcal{N}_{\eps_{\rm n}}(\M)} \int_{\mathcal{B}(z,\eps_{\rm n}^{2/3})} \frac{\langle z-x, \mathsf{A}h \rangle}{\eps_{\rm p}^2} \frac{\langle z-x, \mathsf{A}h \rangle}{\eps_{\rm p}^2} \\
    &\qquad\qquad\qquad\qquad\qquad\qquad\qquad\qquad\times\exp\Big(-\frac{\|z+h-x\|^2}{2\eps_{\rm p}^2}\Big) \rho(x)\dd\mathcal{V}(x)\dd h.
\end{align*}
We will now bound $D_{12}$, $D_{13}$, $D_{14}$, $D_{11}+V_1$, sequentially.
Starting with $D_{12}$, we consider
\begin{align*}
    D_{12}^* 
    &:=\frac{1}{(\eps_{\rm p}\sqrt{2\pi})^d} \int_{h+z \in \tilde{N}_{z}\M\cap \mathcal{N}_{\eps_{\rm n}}(\M)} \int_{\mathcal{B}(z,\eps_{\rm n}^{2/3})} \frac{\langle z-x,\tilde{u}\rangle}{\eps_{\rm p}^2}\frac{\langle z-x, \mathsf{A}h \rangle}{\eps_{\rm p}^2} \\
    &\qquad\qquad\qquad\qquad\qquad\qquad\qquad\qquad\times\exp\Big(-\frac{\|h\|^2}{2\eps_{\rm p}^2}\Big) \exp\Big(-\frac{\|z-x\|^2}{2\eps_{\rm p}^2}\Big) \rho(x)\dd\mathcal{V}(x)\dd h.
\end{align*}
Then by integrating in terms of $h$ first, we get $D_{12}^* = 0$. 
Moreover, by Lemma~\ref{lem:keyTayexp} 
\begin{align*}
    |D_{12}-D_{12}^*| &\leq \frac{C_{\eta}}{(\eps_{\rm p}\sqrt{2\pi})^d} \int_{h+z \in\tilde{N}_{z}\M\cap \mathcal{N}_{\eps_{\rm n}}(\M)} \int_{\mathcal{B}(z,\eps_{\rm n}^{2/3})} \frac{d(z,x)^4}{\eps_{\rm p}^4} \\
    &\qquad\qquad\qquad\qquad\qquad\times\exp\Big(-\frac{\|h\|^2}{2\eps_{\rm p}^2}\Big) \exp\Big(-\frac{cd(z,x)^2}{\eps_{\rm p}^2}\Big) \rho(x)\dd\mathcal{V}(x) \dd h \leq C_{\eta},
\end{align*}
and we get $|D_{12}|\leq C_{\eta}$. 
Due to the analogous expressions, it follows that $|D_{13}| \leq C_{\eta}$ as well.
As for $D_{14}$, we see simply that
\begin{align*}
    |D_{14}| 
    &{\leq} \frac{C}{(\eps_{\rm p}\sqrt{2\pi})^d} \int_{h{+}z \in \tilde{N}_{z}\M\cap \mathcal{N}_{\eps_{\rm n}}(\M)} \int_{\mathcal{B}(z,\eps_{\rm n}^{2/3})} \frac{\|z{-}x\|^2 \|h\|^2}{\eps_{\rm p}^4} \exp\Big({-}\frac{\|z{+}h{-}x\|^2}{2\eps_{\rm p}^2}\Big) \rho(x)\dd\mathcal{V}(x)\dd h \\
    &\leq C_{\eta}.
\end{align*}

To facilitate the upcoming calculations for $D_{11}+V_1$, we use a translation by $-z$ to temporarily set $z=0$, and use the rotation ${\rm A}_0^{-1}$, as discussed in Appendix~\ref{appx:dg}, to align $\tilde{T}_0\M$ with $\R^m\times\{0\}\subset\mathbb{R}^d$. 
Then $D_{11}+V_1$ becomes
\begin{equation*}
    \begin{split}
        &\frac{1}{(\eps_{\rm p}\sqrt{2\pi})^d} \int_{h\in \tilde{N}_0\M\cap \mathcal{N}_{\eps_{\rm n}}(\M)} \exp\Big(-\frac{\|h\|^2}{2\eps_{\rm p}^2}\Big) 
    \int_{\mathcal{B}(0,\eps_{\rm n}^{2/3})} 
    \Big(\frac{\langle x,\tilde{u}\rangle}{\eps_{\rm p}^2}\frac{\langle x,\tilde{w}\rangle}{\eps_{\rm p}^2} - \frac{\langle \tilde{u},\tilde{w}\rangle}{\eps_{\rm p}^2}\Big)\\
    &\qquad\qquad\qquad\qquad\qquad\qquad\qquad\qquad\qquad\qquad\times    \exp\Big(-\frac{\|x\|^2}{2\eps_{\rm p}^2}\Big) \mathfrak{f}(x,h) \rho(x) \dd\mathcal{V}(x) \dd h,
    \end{split}
\end{equation*}
where, from Lemma~\ref{lem:keyTayexp},
\begin{equation} \label{eq:Taytay}
    \mathfrak{f}(x,h) := \exp\Big(-\frac{\langle v_x-x,h\rangle}{\eps_{\rm p}^2}\Big)
    = 1 + \sum_{i=1}^5 \mathfrak{f}_i(x,h) + \mathcal{O}_{\eta}(\eps_{\rm n}^2)
    := 1 + \sum_{i=1}^5 \frac{(-1)^i\langle v_x-x,h\rangle^i}{i!\eps_{\rm p}^{2i}} + \mathcal{O}_{\eta}(\eps_{\rm n}^2).
\end{equation} 
Note that, if $i$ is odd, then $\mathfrak{f}_i(x,h)$ is odd in terms of $h$, and if $i$ is even, then 
\begin{equation} \label{eq:boundedeven}
    |\mathfrak{f}_i(x,h)| \leq \frac{C\|x\|^{2i}\|h\|^i}{\eps_{\rm p}^{2i}},
\end{equation}
due to \eqref{eq:quadbehavior2}.
Let
\begin{multline*}
    I 
    :=  \frac{1}{(\eps_{\rm p}\sqrt{2\pi})^d} \int_{h\in \tilde{N}_0\M\cap \mathcal{N}_{\eps_{\rm n}}(\M)} \exp\Big(-\frac{\|h\|^2}{2\eps_{\rm p}^2}\Big) 
    \int_{\mathcal{B}(0,\eps_{\rm n}^{2/3})} 
    \Big(\frac{\langle x,\tilde{u}\rangle}{\eps_{\rm p}^2}\frac{\langle x,\tilde{w}\rangle}{\eps_{\rm p}^2} - \frac{\langle \tilde{u},\tilde{w}\rangle}{\eps_{\rm p}^2}\Big)\\
\times \exp\Big(-\frac{\|x\|^2}{2\eps_{\rm p}^2}\Big) \rho(x) \dd\mathcal{V}(x) \dd h.
\end{multline*}
Then using the given observation, and particularly \eqref{eq:Taytay}, \eqref{eq:boundedeven}, we can dominate $|D_{11}+V_1 - I|$ with a $C_{\eta}$ multiple of $I_1+I_2+I_3$, where 
\begin{multline*}
    I_i 
    :=  \frac{1}{(\eps_{\rm p}\sqrt{2\pi})^d} \int_{h\in \tilde{N}_0\M\cap \mathcal{N}_{\eps_{\rm n}}(\M)} \exp\Big(-\frac{\|h\|^2}{2\eps_{\rm p}^2}\Big) 
    \int_{\mathcal{B}(0,\eps_{\rm n}^{2/3})} 
    \frac{\|x\|^{4i}}{\eps_{\rm p}^{4i}}\Big(\frac{\|x\|^2}{\eps_{\rm p}^2} + 1\Big)\\
    \times \exp\Big(-\frac{\|x\|^2}{2\eps_{\rm p}^2}\Big) \rho(x) \dd\mathcal{V}(x) \dd h,
\end{multline*}
for $i=1,2$, and 
\begin{equation*}
    I_3 
    {:=}  \frac{1}{(\eps_{\rm p}\sqrt{2\pi})^d} \int_{h\in \tilde{N}_0\M\cap \mathcal{N}_{\eps_{\rm n}}(\M)} \exp\Big({-} \frac{\|h\|^2}{2\eps_{\rm p}^2}\Big) 
    \int_{\mathcal{B}(0,\eps_{\rm n}^{2/3})} 
    \Big(\frac{\|x\|^2}{\eps_{\rm p}^2} {+} 1\Big)
    \exp\Big({-}\frac{\|x\|^2}{2\eps_{\rm p}^2}\Big) \rho(x) \dd\mathcal{V}(x) \dd h.
\end{equation*}
It is straightforward to see that $I_i \leq C_{\eta}$.
To calculate $I$, we note that $v_x\in \R^m\times\{0\}$ is ${\rm  Log}_0(x)\in\mathbb{R}^m$ canonically embedded in $\mathbb{R}^d$. 
Moreover, since $\tilde{u}, \tilde{w}\in\R^m\times\{0\}$, they are canonically embedded images of some $\mathsf{u}, \mathsf{w}\in\mathbb{R}^m$. 
Therefore, using Lemma~\ref{lem:keyTayexp}, particularly \eqref{eq:coreexpfactor2}, we can bound $|I|$ by a sum of the integral
\begin{multline*}
    \frac{1}{(\eps_{\rm p}\sqrt{2\pi})^d} \int_{h\in \tilde{N}_0\M\cap \mathcal{N}_{\eps_{\rm n}}(\M)} \exp\Big(-\frac{\|h\|^2}{2\eps_{\rm p}^2}\Big) 
    \int_{\mathcal{B}(0,\eps_{\rm n}^{2/3})} 
    \frac{\|x\|^2}{\eps_{\rm p}^2}\Big(\frac{\|x\|^4}{\eps_{\rm p}^4} + \frac{\|x\|^2}{\eps_{\rm p}^2}\Big)\\
    \times \exp\Big(-\frac{\|x\|^2}{2\eps_{\rm p}^2}\Big) \rho(x) \dd\mathcal{V}(x) \dd h \leq C_{\eta},
\end{multline*}
and another integral, expressed in local coordinates as
\begin{multline*}
    \frac{1}{(\eps_{\rm p}\sqrt{2\pi})^d} \int_{h\in \tilde{N}_0\M\cap \mathcal{N}_{\eps_{\rm n}}(\M)} \exp\Big(-\frac{\|h\|^2}{2\eps_{\rm p}^2}\Big) 
    \int_{B^m(0,\eps_{\rm n}^{2/3})} 
    \Big(\frac{\langle y,\mathsf{u}\rangle}{\eps_{\rm p}^2}\frac{\langle y,\mathsf{w}\rangle}{\eps_{\rm p}^2} - \frac{\langle \mathsf{u},\mathsf{w}\rangle}{\eps_{\rm p}^2}\Big)\\
    \times \exp\Big(-\frac{\|y\|^2}{2\eps_{\rm p}^2}\Big) \tilde{\rho}(y) J_0(y)  \dd y \dd h.
\end{multline*}
Let $I^*$ be the inner integral above; that is,
\begin{equation*} 
    I^* := 
    \int_{B^m(0,\eps_{\rm n}^{2/3})} 
    \Big(\frac{\langle y,\mathsf{u}\rangle}{\eps_{\rm p}^2}\frac{\langle y,\mathsf{w}\rangle}{\eps_{\rm p}^2} - \frac{\langle \mathsf{u},\mathsf{w}\rangle}{\eps_{\rm p}^2}\Big)
    \exp\Big(-\frac{\|y\|^2}{2\eps_{\rm p}^2}\Big) \tilde{\rho}(y) J_0(y)  \dd y.
\end{equation*}
For $y\in B^m(0,\eps_{\rm n}^{2/3})$, we rely on Taylor's expansion to get
\begin{equation} \label{eq:1}
    \tilde{\rho}(y) = \tilde{\rho}(0) + \langle \nabla\tilde{\rho}(0), y\rangle + \frac{1}{2}\langle D^2 \tilde{\rho}(0) y, y\rangle + \mathcal{O}_{\eta}(\eps_{\rm p}^2),
\end{equation}
and on \eqref{eq:Rauch} to get
\begin{equation} \label{eq:2}
    |J_0(y) - 1| \leq C\|y\|^2.
\end{equation}
Applying \eqref{eq:1}, \eqref{eq:2}, and a standard argument that uses linear symmetry for cancellation, we obtain
\begin{equation} \label{eq:Istarreduction}
    I^* = \eps_{\rm p}^m 
    \int_{B^m(0,\eta^{-1/d}\eps_{\rm n}^{-1/3})} 
    \Big(\frac{\langle y,\mathsf{u}\rangle \langle y,\mathsf{w}\rangle}{\eps_{\rm p}^2} - \frac{\langle\mathsf{u},\mathsf{w}\rangle}{\eps_{\rm p}^2} \Big)
    \exp\Big(-\frac{\|y\|^2}{2}\Big)\tilde{\rho}(0)  \dd y + \mathcal{O}_{\eta}(\eps_{\rm p}^m).
\end{equation}
We now introduce a basic estimate for the $m$-dimensional Gaussian distribution on Euclidean spaces. 

\begin{lemma}\label{lem:Gaussianrv}
Let $\sigma>0$.
Let $X\in\R^m$ be a centered Gaussian distributed random variable such that 
\begin{equation*}
    f_X(x)=\frac{1}{(\sigma \sqrt{2\pi})^m}\exp\Big(-\frac{\|x\|^2}{2\sigma^2}\Big).
\end{equation*}
Then for any $v,w\in\R^m$ we have
\begin{equation*} 
    \int_{\R^m} \Big(\frac{\langle u,x\rangle}{\sigma^2}\frac{\langle w,x\rangle}{\sigma^2}-\frac{\langle u,w\rangle}{\sigma^2}\Big)f_X(x)  \dd x =0.
\end{equation*}
\end{lemma}

\begin{proof}[Proof of Lemma \ref{lem:Gaussianrv}]
We first show the result for the case $m=1$.
\begin{equation*}
    \int_{\R^m} \Big(\frac{\langle u,x\rangle}{\sigma^2}\frac{\langle w,x\rangle}{\sigma^2}-\frac{\langle u,w\rangle}{\sigma^2}\Big)f_X(x)\dd x=\int_{\R^m} \Big(\frac{x^2 uw}{\sigma^4}-\frac{uw}{\sigma^2}\Big)f_X(x)\dd x =\frac{uw}{\sigma^2}-\frac{uw}{\sigma^2}=0,
\end{equation*}
as it becomes a computation of the second moment of Gaussian distribution.
For the case $m>1$, since we can decompose the high-dimensional Gaussian distribution $X$ into $(X^1,\dots,X^m)$ where $X^i$ are i.i.d. one-dimensional centered Gaussian distributed random variable, and the problem reduces to one-dimensional problem. 
The proof is now complete.
\end{proof}

Applying Lemma~\ref{lem:Gaussianrv} (with $\sigma=1$) to \eqref{eq:Istarreduction} gives us
\begin{align*}
    I^* &= -\eps_{\rm p}^m 
    \int_{\mathbb{R}^m-B^m(0,\eta^{-1/d}\eps_{\rm n}^{-1/3})} 
    \Big(\frac{\langle y,\mathsf{u}\rangle \langle y,\mathsf{w}\rangle}{\eps_{\rm p}^2} - \frac{\langle\mathsf{u},\mathsf{w}\rangle}{\eps_{\rm p}^2} \Big)
    \exp\Big(-\frac{\|y\|^2}{2}\Big)\tilde{\rho}(0)  \dd y + \mathcal{O}_{\eta}(\eps_{\rm p}^m),
\end{align*}
which, due to the rapid Gaussian tail decay, implies that $|I^*|\leq C_{\eta}\eps_{\rm p}^m$.
Thus, we conclude that $|I|\leq C_{\eta}$. 
In addition, by recalling $|D_{11} + V_1 - I|\leq |I_1|+|I_2|+|I_3|\leq C_{\eta}$, we conclude $|D_{11}+V_1|\leq C_{\eta}$. 
Finally, since
\begin{align*}
    D_2 &= D-D_1 \\
    &= \frac{1}{(\eps_{\rm p}\sqrt{2\pi})^d} \int_{h+z \in \tilde{N}_{z}\M\cap \mathcal{N}_{\eps_{\rm n}}(\M)} 
    \int_{\M\setminus\mathcal{B}(z,\eps_{\rm n}^{2/3})} \frac{\langle z - x, \tilde{u} + \mathsf{A}h \rangle \langle z - x, \tilde{w} + \mathsf{A}h \rangle}{\eps_{\rm p}^4} \\
    &\qquad\qquad\qquad\qquad\qquad\qquad\qquad\qquad\qquad\qquad\qquad\times\exp\Big(-\frac{\| z +h-x\|^2}{2\eps_{\rm p}^2}\Big) \rho(x)\dd\mathcal{V}(x)\dd h,
\end{align*}
and 
\begin{multline*}
    V_2 = V-V_1 \\
    = -\frac{1}{(\eps_{\rm p}\sqrt{2\pi})^d} \int_{h+z \in \tilde{N}_{z}\M\cap \mathcal{N}_{\eps_{\rm n}}(\M)} 
    \int_{\M\setminus\mathcal{B}(z,\eps_{\rm n}^{2/3})} \frac{\langle \tilde{w}, \tilde{u} \rangle}{\eps_{\rm p}^2} \exp\Big(-\frac{\| z +h-x\|^2}{2\eps_{\rm p}^2}\Big) \rho(x)\dd\mathcal{V}(x)\dd h,
\end{multline*}
employing the rapid Gaussian tail decay, as in \eqref{eq:qE2bound}, \eqref{eq:boundingqE2}, yields, $|D_2|, |V_2| \leq C_{\eta}$.
Putting altogether, we get $|D+V| \leq |D_1+V_1|+|D_2|+|V_2|\leq C_{\eta}$, and moreover, 
\begin{equation*}
    |\partial_{\tilde{w}}\partial_{\tilde{u}} q(z)| \leq |D+V+F+S| \leq |D+V| + |F|+ |S| \leq C_{\eta},
\end{equation*}
as desired. 
\end{proof}

\section{Proof of Lemma~\ref{lem:varvarstar}} \label{appx:varvarstar}

\begin{proof}
Define
\begin{align*}
    \varsigma^1(x,y) &:= \int_{\{z\in\mathbb{R}^d: d(x,z) < \eps_{\rm n}^{2/3} \wedge \|y-z\|< \eps_{\rm n}^{2/3}\} \cap\M} \exp\Big(-\frac{\|x-z\|^2}{2\eps_{\rm w}^2}\Big) \exp\Big(-\frac{\|y-z\|^2}{2\eps_{\rm w}^2}\Big) \rho(z)\dd\mathcal{V}(z) \\
    \varsigma^2(x,y) &:= \int_{\{z\in\mathbb{R}^d: d(x,z) < \eps_{\rm n}^{2/3} \wedge d(y,z)< \eps_{\rm n}^{2/3}\} \cap\M} \exp\Big(-\frac{\|x-z\|^2}{2\eps_{\rm w}^2}\Big) \exp\Big(-\frac{\|y-z\|^2}{2\eps_{\rm w}^2}\Big) \rho(z)\dd\mathcal{V}(z).
\end{align*}
Then, we observe
\begin{align*}
    |\varsigma(x,y)-\varsigma^1(x,y)| 
    &\leq \int_{(B^d(x,\eps_{\rm n}^{2/3})\setminus \mathcal{B}(x,\eps_{\rm n}^{2/3}))\cap\mathcal{M}} \exp\Big(-\frac{\|x-z\|^2}{2\eps_{\rm w}^2}\Big) \rho(z)\dd\mathcal{V}(z) \\
    &\leq \int_{\mathcal{B}(x,2\eps_{\rm n}^{2/3})\setminus \mathcal{B}(x,\eps_{\rm n}^{2/3})} \exp\Big(-\frac{cd(x,z)^2}{\eps_{\rm w}^2}\Big) \rho(z)\dd\mathcal{V}(z),
\end{align*}
as a consequence of \eqref{eq:distancecompare}. 
Therefore,
\begin{align} \label{eq:var01}
    \nonumber |\varsigma(x,y)-\varsigma^1(x,y)| &\leq \int_{(B^m(0,2\eps_{\rm n}^{2/3})\setminus B^m(0,\eps_{\rm n}^{2/3}))} \exp\Big(-\frac{c\|u\|^2}{\eps_{\rm w}^2}\Big) \tilde{\rho}(u) J_x(u) \dd u \\
    \nonumber &\leq C\eps_{\rm w}^m \int_{\mathbb{R}^m\setminus B^m(0,\eps_{\rm w}^{-1/3}\eps^{2/3})} \exp(-c\|u\|^2) \dd u\\
    &\leq C\eps_{\rm w}^m\eps^2,
\end{align}
once $\eps$ is sufficiently small. 
Applying the same approach for $|\varsigma^1(x,y)-\varsigma^2(x,y)|$, we also have
\begin{equation} \label{eq:var12}
    |\varsigma^1(x,y) - \varsigma^2(x,y)| \leq C\eps_{\rm w}^m\eps^2. 
\end{equation}
Together, \eqref{eq:var01}, \eqref{eq:var12} yield
\begin{equation} \label{eq:var02}
    \varsigma(x,y) = \varsigma^2(x,y) + \mathcal{O}(\eps_{\rm w}^m\eps^2).
\end{equation}
Continuing, we define
\begin{align*}
    \varsigma^3(x,y) := \int_{\{z\in\mathbb{R}^d: d(x,z) < \eps_{\rm n}^{2/3} \wedge d(y,z)< \eps_{\rm n}^{2/3}\} \cap\M} \exp\Big(-\frac{d(x,z)^2}{2\eps_{\rm w}^2}\Big) \exp\Big(-\frac{\|y-z\|^2}{2\eps_{\rm w}^2}\Big) \rho(z)\dd\mathcal{V}(z) \\
    \varsigma^4(x,y) := \int_{\{z\in\mathbb{R}^d: d(x,z) < \eps_{\rm n}^{2/3} \wedge d(y,z)< \eps_{\rm n}^{2/3}\} \cap\M} \exp\Big(-\frac{d(x,z)^2}{2\eps_{\rm w}^2}\Big) \exp\Big(-\frac{d(y,z)^2}{2\eps_{\rm w}^2}\Big) \rho(z)\dd\mathcal{V}(z).
\end{align*}
Note that $\varsigma^*(x,y) = \varsigma^4(x,y)$.
Then it can be readily checked from \eqref{eq:distancecompare} that $\varsigma^3(x,y) \leq \varsigma^2(x,y)$.
Moreover, for $\|x-z\|<\eps_{\rm n}^{2/3}$,
\begin{equation} \label{eq:twoGaussians}
    \exp\Big(-\frac{\|x-z\|^2}{2\eps_{\rm w}^2}\Big) 
    \leq \exp\Big(-\frac{d(x,z)^2}{2\eps_{\rm w}^2}\Big) \exp\Big(\frac{C\|x-z\|^4}{\eps_{\rm w}^2}\Big) \leq \exp\Big(-\frac{d(x,z)^2}{2\eps_{\rm w}^2}\Big)\Big(1+ C\eps_{\rm w}^{2/3}\eps^{8/3}\Big),
\end{equation}
where Taylor expansion has been applied to the last inequality, akin to the derivation of \eqref{eq:coreexpfactor2}.
Hence, from \eqref{eq:twoGaussians}, together with \eqref{omegamagext}, 
\begin{equation} \label{eq:var23}
    \varsigma^2(x,y) 
    \leq \varsigma^3(x,y) + C\eps_{\rm w}^m(\eps_{\rm w}^{2/3}\eps^{8/3}).
\end{equation}
By a similar argument, we can show that 
\begin{equation} \label{eq:var34}
    \varsigma^3(x,y) = \varsigma^4(x,y) + \mathcal{O}(\eps_{\rm w}^m(\eps_{\rm w}^{2/3}\eps^{8/3})).
\end{equation}
Combining \eqref{eq:var02}, \eqref{eq:var23}, \eqref{eq:var34}, we get
\begin{equation*}
    \varsigma(x,y) = \varsigma^4(x,y) + \mathcal{O}(\eps_{\rm w}^m\eps^2) = \varsigma^*(x,y) + \mathcal{O}(\eps_{\rm w}^m\eps^2),
\end{equation*}
which is what we want. 
\end{proof}

\section{Proof of Lemma~\ref{lem:necessityTay}} \label{appx:necessityTay}

\begin{proof}
We first demonstrate \eqref{eq:alphaclaim}. 
Express $\varsigma^*_x$, via a change of variables using the exponential map ${\rm Exp}_x$ \eqref{expmap}, as
\begin{equation} \label{eq:varsigmaEu}
    \tilde{\varsigma}^*_{0}(v) 
    = \int_{\tilde{\mathcal{U}}_{0,v}} \exp\Big(-\frac{\|u\|^2}{2\eps_{\rm w}^2}\Big) \exp\Big(-\frac{\|u-v\|^2}{2\eps_{\rm w}^2}\Big) \tilde{\rho}(u) J_x(u) \dd u,
\end{equation}
where $\tilde{\mathcal{U}}_{0,v}:=\{u: \|u\| <\eps_{\rm n}^{2/3} \wedge \|u-v\|<\eps_{\rm n}^{2/3}\}\cap\M$, and ${\rm Exp}_x(v) = y$. 
For instance, $\tilde{\mathcal{U}}_{0,0}=B^m(0,\eps_{\rm n}^{2/3})$.
Then by an application of Taylor's expansion and a change in variables,
\begin{align}
    \nonumber \tilde{\varsigma}^*_0(0) &= 
    \int_{B^m(0,\eps_{\rm n}^{2/3})} \exp\Big(-\frac{\|u\|^2}{\eps_{\rm w}^2}\Big) \tilde{\rho}(u) J_x(u) \dd u \\
    \label{eq:varsigma0} &= \eps_{\rm w}^m \int_{B^m(0,\eps_{\rm w}^{-1/3}\eps^{2/3})} \exp (-\|u\|^2) (\tilde{\rho}(0) + \mathcal{O}(\eps_{\rm n}^{2/3}))(1+\mathcal{O}(\eps_{\rm n}^{2/3})) \dd u.
\end{align}
It is readily apparent that 
\begin{align}
    \nonumber &\int_{B^m(0,\eps_{\rm w}^{-1/3}\eps^{2/3})} \exp (-\|u\|^2) \tilde{\rho}(0)  \dd u \\
    \nonumber&\qquad\qquad\qquad= \tilde{\rho}(0)\int_{\mathbb{R}^m} \exp (-\|u\|^2)  \dd u - \tilde{\rho}(0) \int_{\mathbb{R}^m\setminus B^m(0,\eps_{\rm w}^{-1/3}\eps^{2/3})} \exp (-\|u\|^2)  \dd u\\
    \label{eq:varsigmaalpha} &\qquad\qquad\qquad= \alpha \rho(x) + \mathcal{O}(\eps) 
\end{align}
due to the rapid Gaussian decay, once $\eps$ is sufficiently small. 
The claim \eqref{eq:alphaclaim} now follows from a combination of \eqref{eq:varsigma0}, \eqref{eq:varsigmaalpha}.

We demonstrate \eqref{eq:derivativeclaim} next.
For $v\in B^m(0,\eps_{\rm n}^{2/3})$, we write $v=\sum_{i=1}^d v^ie_{i}$, where $v^i\in\mathbb{R}$, and $e_{i}$ is the standard $i$th basis vector.
Let $\partial_{i}$ denote the partial derivative along the $i$th axis.  
By the Reynold's transport equation \citep[Appendix~C, Theorem~6]{evans2010partial}, the partial derivative $\partial_i\tilde{\varsigma}^*_0$ of $\tilde{\varsigma}^*_0$ \eqref{eq:varsigmaEu} at $v$ is
\begin{equation} \label{eq:firstderivativesaslims}
    \partial_{i}\tilde{\varsigma}^*_0(v) = \frac{d}{ds} \Big(\int_{\tilde{\mathcal{U}}_{0,v+se_i}} \exp\Big(-\frac{\|u\|^2}{2\eps_{\rm w}^2}\Big) \exp\Big(-\frac{\|u-v-se_i\|^2}{2\eps_{\rm w}^2}\Big) \tilde{\rho}(u) J_x(u) \dd u\Big)|_{s=0} = L_1 + L_2,
\end{equation}
where
\begin{equation} \label{eq:theLs}
    \begin{split}
        L_1 &:= \int_{\tilde{\mathcal{U}}_{0,v}} \frac{u^i-v^i}{\eps_{\rm w}^2}\exp\Big(-\frac{\|u-v\|^2}{2\eps_{\rm w}^2}\Big) \exp\Big(-\frac{\|u\|^2}{2\eps_{\rm w}^2}\Big) \tilde{\rho}(u) J_x(u) \dd u,\\
        L_2 &:= \int_{\partial \tilde{\mathcal{U}}_{0,v}} \exp\Big(-\frac{\|u\|^2}{2\eps_{\rm w}^2}\Big) \exp\Big(-\frac{\|u-v\|^2}{2\eps_{\rm w}^2}\Big) \tilde{\rho}(u) J_x(u) \langle e_i, \vec{n}\rangle \dd S.
    \end{split}  
\end{equation}
Here, $\vec{n}$ denotes the outward pointing unit normal, while $\partial \tilde{\mathcal{U}}_{0,v}$ denote the surface of $\tilde{\mathcal{U}}_{0,v}$. 
Moreover, it can be seen from \eqref{eq:theLs} that both $L_1$, $L_2$ are continuous in terms of $v$.
Thus, we conclude that $\nabla\tilde{\varsigma}^*_0(0)$ exists.
To show that
\begin{equation} \label{eq:vanishfirstderivative}
    \nabla \tilde{\varsigma}^*_0(0) = 0,
\end{equation}
we make an observation from \eqref{eq:varsigmaEu} that $\tilde{\varsigma}^*_0(v)\leq\tilde{\varsigma}^*_0(0)$, for any $v\in B^m(0,\eps_{\rm n}^{2/3})$.
Hence in particular $0$ is a global maximum of $\tilde{\varsigma}^*_0$ along the $i$th axis. 
Therefore, $\partial_{i}\tilde{\varsigma}^*_0(0)=0$, and \eqref{eq:vanishfirstderivative} must hold.

Continuing, we prove that
\begin{equation} \label{eq:boundedsecondderivative}
    |D^2\tilde{\varsigma}^*_0(0)| = \mathcal{O}(\eps_{\rm w}^m).
\end{equation}
From \eqref{eq:firstderivativesaslims}, \eqref{eq:theLs}, and the Gauss-Green theorem \citep[Appendix~C, Theorem~1]{evans2010partial}, 
\begin{multline} \label{eq:firstderivrecap}
    \partial_i\tilde{\varsigma}^*_0(v) 
    = \int_{\tilde{\mathcal{U}}_{0,v}} \frac{u^i-v^i}{\eps_{\rm w}^2}\exp\Big(-\frac{\|u-v\|^2}{2\eps_{\rm w}^2}\Big) \exp\Big(-\frac{\|u\|^2}{2\eps_{\rm w}^2}\Big) \tilde{\rho}(u) J_x(u) \dd u \\
    + \int_{\tilde{\mathcal{U}}_{0,v}} \partial_i\Big(\exp\Big(-\frac{\|u\|^2}{2\eps_{\rm w}^2}\Big) \exp\Big(-\frac{\|u-v\|^2}{2\eps_{\rm w}^2}\Big)\tilde{\rho}(u) J_x(u)\Big) \dd u.
\end{multline}
Then, with a small simplification, repeating the previous process with $v+te_j$ replacing $v$ in \eqref{eq:firstderivrecap}, we obtain
\begin{equation} \label{eq:2ndprimeform}
    \partial_j\partial_i\tilde{\varsigma}^*_0(v) 
    = \frac{d}{dt} \bigg(\int_{\tilde{\mathcal{U}}_{0,v+te_j}} \exp\Big(-\frac{\|u-v-te_j\|^2}{2\eps_{\rm w}^2}\Big) \partial_i\Big(\exp\Big(-\frac{\|u\|^2}{2\eps_{\rm w}^2}\Big) \tilde{\rho}(u) J_x(u)\Big) \dd u \bigg)|_{t=0}.
\end{equation}
The existence of $\partial_j\partial_i\tilde{\varsigma}^*_0(v)$ now follows from the Reynolds transport equation, and its continuous dependence on $v$ is straightforward. 
It remains to calculate individual $\partial_j\partial_i\tilde{\varsigma}^*_0(0)$, which is, from \eqref{eq:2ndprimeform},
\begin{equation} \label{eq:2ndprimeformsplit}
    \partial_j\partial_i\tilde{\varsigma}^*_0(0) = K_1+K_2+K_3
\end{equation}
where
\begin{align*}
    K_1 & {:=} \int_{B^m(0,\eps_{\rm n}^{2/3})} \frac{u^j}{\eps_{\rm w}^2} \frac{u^i}{\eps_{\rm w}^2} \exp\Big({-}\frac{\|u\|^2}{\eps_{\rm w}^2}\Big) \tilde{\rho}(u) J_x(u) \dd u {-} \int_{B^m(0,\eps_{\rm n}^{2/3})} \frac{\delta_{ij}}{\eps_{\rm w}^2} \exp\Big({-}\frac{\|u\|^2}{\eps_{\rm w}^2}\Big) \tilde{\rho}(u) J_x(u) \dd u \\
    K_2 &{:=} -\int_{B^m(0,\eps_{\rm n}^{2/3})} \frac{u^i}{\eps_{\rm w}^2} \exp\Big(-\frac{\|u\|^2}{\eps_{\rm w}^2}\Big) \partial_j\big(\tilde{\rho}(u) J_x(u)\big) \dd u 
\end{align*}
and
\begin{equation*}
    K_3 := \frac{d}{dt} \bigg(\int_{\tilde{\mathcal{U}}_{0,v+te_j}} \exp\Big(-\frac{\|u-v-te_j\|^2}{2\eps_{\rm w}^2}\Big) \exp\Big(-\frac{\|u\|^2}{2\eps_{\rm w}^2}\Big) \partial_i\big(\tilde{\rho}(u) J_x(u)\big) \dd u \bigg)|_{t=0}.
\end{equation*}
Using a nearly identical argument as in Appendix~\ref{appx:qdoubleprimebounded} to address $|D_{11}+V_1|$, particularly Lemma~\ref{lem:Gaussianrv}, we get
\begin{align*} 
    |K_1| & {=} \bigg|\int_{B^m(0,\eps_{\rm n}^{2/3})} \frac{u^i}{\eps_{\rm w}^2}\frac{u^j}{\eps_{\rm w}^2} \exp\Big({-}\frac{\|u\|^2}{\eps_{\rm w}^2}\Big) \tilde{\rho}(u) J_x(u) \dd u {-} \int_{B^m(0,\eps_{\rm n}^{2/3})} \frac{\delta_{ij}}{\eps_{\rm w}^2}\exp\Big({-}\frac{\|u\|^2}{\eps_{\rm w}^2}\Big) \tilde{\rho}(u) J_x(u) \dd u\bigg| \\
    &= \mathcal{O}(\eps_{\rm w}^m).
\end{align*}
Another similarly straightforward argument, combined with Taylor's expansion, also leads to
\begin{equation*} 
    |K_2| = \bigg|\int_{B^m(0,\eps_{\rm n}^{2/3})} \frac{u^i}{\eps_{\rm w}^2} \exp\Big(-\frac{\|u\|^2}{\eps_{\rm w}^2}\Big) \partial_j\big(\tilde{\rho}(u) J_x(u)\big) \dd u \bigg| = \mathcal{O}(\eps_{\rm w}^m).
\end{equation*}
Lastly, we observe that the expression for $K_3$ mirrors that in \eqref{eq:firstderivativesaslims}. 
Therefore, by the same reasoning used to derive the boundedness of the first derivative, we can express $K_3$ as a sum of two terms analogous to $L_1$, $L_2$, and, through Taylor's expansion, deduce that $|K_3|=\mathcal{O}(\eps_{\rm w}^m)$.
Then putting all the obtained facts that $|K_i|=\mathcal{O}(\eps_{\rm w}^m)$ in \eqref{eq:2ndprimeformsplit}, we arrive at \eqref{eq:boundedsecondderivative}, as wanted. 
\end{proof}

\end{document}